\documentclass[twoside]{article}

\usepackage[accepted]{aistats2021-1}

\usepackage{xcolor}
\usepackage{amsmath}
\usepackage{amsfonts}
\usepackage{amssymb}
\usepackage{amsbsy}
\usepackage{amsthm}
\usepackage{appendix}
\usepackage{diagbox}
\usepackage{multirow}
\usepackage[ruled,vlined]{algorithm2e}

\newcommand{\argmin}[1]{\underset{#1}{\operatorname{arg}\,\operatorname{min}}\;}
\newcommand{\R}{\mathbb{R}}
\newcommand{\opt}{\mathrm{opt}}
\newcommand{\card}[1]{\mathrm{card}\left( #1\right)}
\newcommand{\vc}[1]{\mathrm{vc}\left( #1 \right)}
\newcommand{\diam}[1]{\mathrm{diam}\left(#1\right)}
\newcommand{\CR}{\mathrm{CR}}
\newcommand{\PI}{\mathrm{PI}}

\usepackage{bbm}
\usepackage{float}
\usepackage{tabularx}
\usepackage{comment}
\usepackage{caption}
\usepackage{multicol}
\usepackage{graphicx}
\usepackage[round]{natbib}

\usepackage{multibib}
\newcites{APX}{References}

\newtheorem{theorem}{Theorem}[section]
\newtheorem{definition}[theorem]{Definition}

\newtheorem{lemma}[theorem]{Lemma}

\newtheorem{assumption}{Assumption}
\numberwithin{equation}{section}

\begin{document}

\twocolumn[

\aistatstitle{Learning Prediction Intervals for Regression: Generalization and Calibration}

\aistatsauthor{Haoxian Chen \And Ziyi Huang \And Henry Lam \And Huajie Qian \And Haofeng Zhang }

\aistatsaddress{IEOR\\ Columbia University \And  EE\\ Columbia University \And IEOR\\ Columbia University \And  Alibaba Group \And IEOR\\ Columbia University }]

\begin{abstract}
We study the generation of prediction intervals in regression for uncertainty quantification. This task can be formalized as an empirical constrained optimization problem that minimizes the average interval width while maintaining the coverage accuracy across data. We strengthen the existing literature by studying two aspects of this empirical optimization. First is a general learning theory to characterize the optimality-feasibility tradeoff that encompasses Lipschitz continuity and VC-subgraph classes, which are exemplified in regression trees and neural networks. Second is a calibration machinery and the corresponding statistical theory to optimally select the regularization parameter that manages this tradeoff, which bypasses the overfitting issues in previous approaches in coverage attainment. We empirically demonstrate the strengths of our interval generation and calibration algorithms in terms of testing performances compared to existing benchmarks.

\end{abstract}

\section{Introduction}
While most literature in machine learning focuses on point prediction, uncertainty quantification plays, arguably, an equally important role in reliability assessment and risk-based decision-making. In regression, a natural approach to quantify uncertainty is the prediction interval (PI), namely an upper and lower limit for a given feature value $X$ that covers the corresponding outcome $Y$ with high probability. The interval center represents the expected outcome, whereas the width represents the uncertainty. A test point with a high expected outcome, but wide PI width, means that the outcome can still be low with a significant chance, thus signifies a downside risk that should not be overlooked.

In this paper, we study the construction of PIs that satisfy an overall target prediction level across data, known as the expected coverage rate \citep{rosenfeld2018discriminative}. Compared to widely used conditional (on $X$) coverage rate, this notion is advantageously more tangible to measure and easier to achieve. This means that a much wider class of models can be trained to build PIs, as less conditions are needed to impose on the true relation and the model class to obtain satisfactory guarantees. In general, constructing a good PI in this framework requires balancing a tradeoff between the expected interval width and coverage maintenance, which can be formalized as an empirical constrained optimization. This viewpoint has been used in \citet{khosravi2010lower} and \citet{pearce2018high} that focus on neural networks, \citet{rosenfeld2018discriminative} that studies a dual formulation, and \citet{galvan2017multi} that uses multi-objective evolutionary optimization. It also relates to the learning of minimum volume sets \citep{polonik1997minimum,scott2006learning} in which a similar tradeoff between volume and probability content appears. Building on these works, our goal in this paper is to study two key inter-related statistical aspects of this empirical constrained optimization that enhances previous results both in theory and in practice:



\textbf{Feasibility-Optimality Tradeoff for Interval Models.} We develop a learning theory for the PIs constructed from empirical constrained optimization that statistically achieves both feasibility (coverage) and optimality (interval width). Methodologically, we build a general ``sensitivity measure" that controls this tradeoff, which in turn requires developing deviation bounds for simultaneous empirical processes. Our theory in particular covers the Lipschitz continuous model class (in parameter) and finite Vapnik–Chervonenkis (VC)-subgraph class,  exemplified by a wide class of neural networks and regression trees. Such type of joint coverage-width learning guarantees appears the first in the literature. It expands the coverage-only results and the considered model classes in \citet{rosenfeld2018discriminative}. It also generalizes \citet{scott2006learning} as both our constraints and objectives possess extra sophistication related to the shape requirement of the set as an interval, and also we characterize feasibility and optimality attainments explicitly instead of the implicit metric in \citet{scott2006learning}.

\textbf{Calibration Method and Performance Guarantees.} We propose a general-purpose, ready-to-implement calibration methodology to guarantee overall PI coverage with prefixed confidence. This approach is guided by a novel utilization of the high-dimensional Berry-Esseen theorem \citep{chernozhukov2017central}. It is designed to combat the overfitting issue of interval models and perform accurately on the \emph{test} set. We demonstrate empirically how our approach either outperforms other methods in terms of achieving correct coverages or, for those methods with comparable coverages, we attain shorter interval widths. Moreover, our approach applies, with little adjustment, to accurately construct multiple PIs at different prediction levels simultaneously. This adds extra flexibility for decision-makers to construct PIs without needing to pre-select the prediction level in advance.

\section{Related Work}
We first review two most closely related methods, and then move on to other works.


\textbf{Conformal Learning (CL).} 
First proposed in \citet{vovk2005algorithmic}, conformal learning (CL) is a class of methods that leverage data exchangeability to constructs PIs with finite-sample and distribution-free coverage guarantees. The original CL requires retraining for each possible test point and is therefore computationally prohibitive in general. Split/inductive CL \citep{papadopoulos2008inductive,lei2015conformal,lei2018distribution} improves the computational efficiency based on a holdout validation that avoids retraining, but at the cost of higher variability and wider intervals due to less efficient data use. Lying in between are variants based on more efficient cross-validation schemes, including leave-one-out (or the Jackknife; \citet{barber2019predictive,alaa2020discriminative,steinberger2016leave}), K-fold \citep{vovk2015cross} and ensemble methods \citep{gupta2019nested,kim2020predictive}. Recently, quantile regression are combined with CL \citep{kivaranovic2020adaptive,romano2019conformalized} to take into account the heterogeneity of uncertainties across feature values. Despite its generality, the coverage guarantees from CL are only marginal with respect to the training data (except split CL \citep{vovk2012conditional}), whereas our proposed calibration method provides a stronger high confidence guarantee. Moreover, our approach explicitly optimizes the interval width, therefore typically generates shorter PIs than CL.

\textbf{Quantile Regression (QR).}
Quantile regression (QR) estimates the conditional quantiles of $Y$ that can be used to construct PIs. Classical QR methods require strong assumptions (e.g., linearity or other parametric forms; Chapter 4 in \citet{koenker2001quantile}). Approaches that relax these assumptions include quantile regression forests (QRF) \citep{meinshausen2006quantile} and kernel support vector machine (SVM) \citep{steinwart2011estimating}. However, little is known about their finite-sample coverage performance because of estimation errors in the quantiles. Recently, \citet{kivaranovic2020adaptive} proposes calibrating the weight parameter in the pinball loss on a holdout data set to enhance PI coverage. This calibration scheme, however, does not address overfitting on the holdout set, thus could fall short of providing correct coverages on test data as our experiments show.

\textbf{Other Approaches.} PI construction has been substantially studied in classical statistics. To understand this construction, the error of a (point) prediction can typically be decomposed into two components: model uncertainty, which comes from the variability in the training data or method, and outcome uncertainty, which comes from the noise of $Y$ conditional on $X$. The classical literature often assumes well-defined and simple forms on the relation between $X$ and $Y$ (e.g., linear model, Gaussian; \citealt{seber2012linear}). In this case, model uncertainty reduces to parameter estimation errors. Outcome uncertainty, on the other hand, is intrinsic in the population distribution but not the training, i.e. it arises even if the model is perfectly trained. An array of methods account for both sources of variability, which utilize approaches ranging from asymptotic normality \citep{seber2012linear}, deconvolution \citep{schmoyer1992asymptotically}, and resampling schemes such as the bootstrap \citep{stine1985bootstrap} and jackknife \citep{steinberger2016leave}.

To overcome the strong assumptions in classical statistical models, several model-free approaches have been developed. Nonparametric regression, such as spline or kernel-based methods  \citep{doksum2000spline,olive2007prediction}, removes rigid model assumptions but at the expense of strong dimension dependence. Gaussian processes or kriging-based methods \citep{sacks1989design}, popular in the areas of metamodeling and computer emulation, regard outcomes as a response surface and perform Gaussian posterior updates. In particular, stochastic kriging (SK) model \citep{doi:10.1287/opre.1090.0754} constructs PIs that account for both model and output variabilities by using a prior correlation structure that entails both. However, SK does not deliver a frequentist coverage guarantee nor convergence rate. More recently, uncertainty quantification of neural networks regarding their model and output variabilities are studied, via methods such as the delta method and the bootstrap  \citep{papadopoulos2001confidence,khosravi2011comprehensive}. Nonetheless, like in classical statistical models, these approaches can only capture variability due to data and training noises, but not the bias against the true relation.

Lastly, our PI construction follows the \textit{high-quality} criterion in works including \citet{khosravi2010lower, khosravi2011comprehensive, galvan2017multi, pearce2018high,rosenfeld2018discriminative,zhang2019random,zhu2019hdi}, which propose various loss functions to capture the width-coverage tradeoff. They are also related to the highest density intervals in statistics \citep{box2011bayesian}. Our investigations in this paper provide theoretical guarantees in using this criterion.

\section{PI Learning as Empirical Constrained Optimization}\label{sec:problem}
We consider the general regression setting where $X \in \mathcal{X}\subset \R^d$ is the feature vector and $Y\in \mathcal{Y}\subset \R$ is the outcome. Given an i.i.d. data set $\mathcal{D}:=\{(X_i,Y_i)\}_{i=1,\ldots,n}$ each drawn from an unknown joint distribution $\pi$, our goal is to find a PI defined as:
\begin{definition}
An interval $[L(x),U(x)]$, where both $L,U:\mathcal{X}\to \R$, is called a prediction interval (PI) with (overall) coverage rate $1-\alpha$ ($0\le \alpha\le 1$) if $\mathbb{P}_{\pi}(Y\in[L(X),U(X)])\geq 1-\alpha$, where $\mathbb{P}_{\pi}$ denotes the probability with respect to the joint distribution $\pi$.
\end{definition}







We aim to find two functions $L$ and $U$, both in a hypothesis class $\mathcal{H}$. 
To learn the optimal high-quality PI that attains a given coverage rate $1-\alpha$, we consider the following constrained optimization problem:
\begin{equation} \label{OP1}
\begin{aligned}
\min_{L,U\in \mathcal{H}\;\text{and}\;L\leq U}\ &\mathbb{E}_{\pi_X}[U(X)-L(X)]\\
\text{subject to } &\mathbb{P}_{\pi}(Y\in[L(X),U(X)])\geq 1-\alpha
\end{aligned}   
\end{equation}
where $\mathbb{E}_{\pi_X}$ denotes the expectation with respect to the marginal distribution of $X$. 
Given the data $\mathcal{D}$, we approximate \eqref{OP1} with the following empirical constrained optimization problem
\begin{equation}\label{OP2}
\begin{aligned}
\widehat{\opt} (t):
\min_{L,U\in \mathcal{H}\;\text{and}\;L\leq U}\ &\mathbb{E}_{\hat\pi_X}[U(X)-L(X)] \\
\text{subject to }&\mathbb{P}_{\hat\pi}(Y\in[L(X),U(X)])\geq 1-\alpha+t
\end{aligned}
\end{equation}
parameterized by $t\in [0,\alpha]$, where $\mathbb{E}_{\hat\pi_X}$, $\mathbb{P}_{\hat\pi}$ are expectation and probability with respect to the empirical distribution constructed from the data $\mathcal{D}$, e.g., $\mathbb{E}_{\hat\pi_X}[U(X)-L(X)]=\frac{1}{n}\sum_{i=1}^n(U(X_i)-L(X_i))$. The adjustment $t$, which can be viewed as a penalty term for the empirical coverage rate, is used to boost the generalized coverage performance for the optimal interval solved from $\widehat{\opt} (t)$: If no adjustment is made in the constraint ($t=0$), then, because of noise, the true coverage can be lower than $1-\alpha$ with significant probability even if the empirical coverage is above $1-\alpha$. A positive $t$ decreases the probability of such an event.
Choosing $t$ too large, however, would eliminate more intervals from the feasible set, leading to a deterioration in the obtained expected width (objective). One of our main investigations is to balance the coverage and width performance by judiciously choosing $t$.

We point out that, while our focus is on training $L$ and $U$ directly, our approach can also be applied naturally when we are given in advance a well-trained point predictor $f:\mathcal{X}\to\R$ (obtained by any means independent of $\mathcal{D}$). In this case we may seek two non-negative functions $\delta_i : \mathcal X \to [0,\infty)$ such that $[L(x),U(x)]=[f(x)-\delta_1(x), f(x)+\delta_2(x)]$.
Our subsequent development applies by constructing lower and upper bounds for the ``translational" data set $\tilde{\mathcal{D}}:=\{(X_i,Y_i-f(X_i))\}_{i=1,\ldots,n}$.

\section{Joint Coverage-Width Learning Guarantees}\label{sec:learning}
We establish finite-sample generalization error bounds for two major classes: 1) finite VC dimensions, and 2) Lipschitz continuous in parameters, by building on a unified ``sensitivity bound" on the oracle optimization. Corresponding results on consistency are provided in Appendix \ref{sec:consistency}. To begin, 
let $R^*_t(\mathcal{H})$ be the optimal value of
\begin{equation}\label{OP3}
\begin{aligned}
\opt (t):
\min_{L,U\in\mathcal{H}\;\text{and}\;L\leq U}\ &\mathbb{E}_{\pi_X}[U(X)-L(X)]\\
\text{subject to } &\mathbb{P}_{\pi}(Y\in[L(X),U(X)])\geq 1-\alpha+t
\end{aligned}
\end{equation}
which is \eqref{OP1} but with a higher target prediction level, and correspondingly $R^*(\mathcal{H})$ be the optimal value of \eqref{OP1}. 
We make the following assumptions on the hypothesis class $\mathcal{H}$, and the conditional distribution of $Y$ given $X$:
\begin{assumption}\label{class: non-negative translation}
For every function $h \in \mathcal{H}$, we have $h + c\in \mathcal{H}$ for every constant $c\in\R$.
\end{assumption}
\begin{assumption}\label{conditional density: positiveness}
The conditional distribution of $Y$ given $X=x$ has a density $p(\cdot\vert x)$ for every $x\in\mathcal{X}$. Moreover, for every $x,y$ such that $\mathbb{P}_{\pi}(Y\leq y\vert X=x)\in (0,1)$, we have $p(y\vert x)>0$ and that $p(\cdot\vert x)$ is continuous at $y$.
\end{assumption}


The simple closedness property in Assumption \ref{class: non-negative translation} turns out to allow sufficiently tight and tractable analysis for many, potentially complicated, function classes under the same roadmap. Many common classes (e.g., linear, piece-wise constant such as tree-based models, and neural networks with linear activation in the output unit) satisfy Assumption \ref{class: non-negative translation}. It is also straightforward to enforce a class to satisfy Assumption \ref{class: non-negative translation} by simply adding one extra parameter. Assumption \ref{conditional density: positiveness} is a mild non-degeneracy condition on the conditional distribution (e.g., Assumption \ref{conditional density: positiveness} holds when $p(\cdot\vert x)$ is Gaussian or uniform over a certain interval for each $x$). 

We have the following upper bound for $R^*_t(\mathcal{H})-R^*(\mathcal{H})$ for $0<t<\alpha$:
\begin{theorem}[Sensitivity bound with respect to target prediction level]\label{sensitivity bound}
Suppose Assumptions \ref{class: non-negative translation}-\ref{conditional density: positiveness} hold. For every $x\in \mathcal{X}$ and $\beta\in(0,1)$, let
  $  \Gamma(x,\beta):=\inf\big\{p(y_1\vert x) + p(y_2\vert x):\ \mathbb{P}_{\pi}(y_1\leq Y \leq y_2\vert X=x)\leq 1-\beta\big\}$
and $\gamma_{\beta}:= \sup\{z>0:\mathbb{P}_{\pi_X}(\Gamma(X,\beta)< z)\leq \beta\}$. Then, for every $\alpha\in (0,1)$ and $t\in (0,\alpha)$, we have
$    R^*_t(\mathcal{H})-R^*(\mathcal{H})\leq 6t/((\alpha - t)\gamma_{(\alpha-t)/3})$.
\end{theorem}

We briefly explain how Theorem \ref{sensitivity bound} helps develop generalization guarantees. Let $\mathcal{H}^2_t,\hat{\mathcal{H}}^2_t\subset \mathcal{H}^2:=\mathcal{H}\times \mathcal{H}$ be the feasible set of $\opt (t)$ and $\widehat{\opt} (t)$ respectively. If a uniform error bound $\Delta t$ over the class $\mathcal{H}^2$ can be established for the empirical coverage constraint in \eqref{OP2}, then $ \mathcal{H}^2_{t+\Delta t}\subset\hat{\mathcal{H}}^2_t$, therefore the shortest interval we can potentially learn from $\widehat{\opt} (t)$ can only be wider than the oracle optimum from \eqref{OP1} by at most $R^*_{t+\Delta t}(\mathcal{H})-R^*(\mathcal{H})$. The density lower bounds in Theorem \ref{sensitivity bound} ensure a sufficient increase of the coverage probability
as interval width increases, which inversely constrains the growth of interval width as the required coverage increases.

To derive finite-sample convergence rate, we first assume the availability of certain deviation bounds for the empirical coverage rate and interval width that appear in $\widehat{\opt}(t)$:
\begin{theorem}[Rate of convergence]\label{convergence rate}
Suppose Assumptions \ref{class: non-negative translation}-\ref{conditional density: positiveness} hold, and the following deviation bounds hold for every $\epsilon,t>0$: $\mathbb{P}\big(\sup_{h\in\mathcal{H}}\lvert\mathbb{E}_{\hat\pi_X}[h(X)]-\mathbb{E}_{\pi_X}[h(X)]\rvert > \epsilon\big)\leq \phi_1(n,\epsilon,\mathcal{H})$ and 
$\mathbb{P}\big(\sup_{L,U\in\mathcal{H}\;\text{and}\;L\leq U}\lvert\mathbb{P}_{\hat\pi}(Y\in [L(X),U(X)])-\mathbb{P}_{\pi}(Y\in [L(X),U(X)])\rvert > t\big)\leq \phi_2(n,t,\mathcal{H})
$. Then for every $t\in (0,\frac{\alpha}{2})$ and $\epsilon>0$, with probability at least $1-\phi_1(n,\epsilon,\mathcal{H})-\phi_2(n,t,\mathcal{H})$, we have, for every optimal solution $(\hat L_{t}^*,\hat U_{t}^*)$ of $\widehat{\opt} (t)$, that $\mathbb{P}_{\pi}(Y\in [\hat L_{t}^*(X),\hat U_{t}^*(X)])\geq 1-\alpha$
and $ \mathbb{E}_{\pi_X}[\hat U_{t}^*(X)-\hat L_{t}^*(X)]\leq \mathcal{R}^*(\mathcal{H})+\frac{12t}{(\alpha-2t)\gamma_{(\alpha-2t)/3}}+4\epsilon$.
\end{theorem}
Theorem \ref{convergence rate} translates the deviation bounds of two empirical processes into the probability of jointly achieving optimality and feasibility. With this, our focus is to derive deviation bounds for important hypothesis classes. Such bounds are well-understood in the literature, e.g., Chapter 2.14 in \citet{van1996weak} and Chapter 3 in \citet{vapnik2013nature}. However, in our case, we need to go beyond the standard theory to control two empirical processes simultaneously by choosing a single function class $\mathcal{H}$. Next we present two fairly general choices of $\mathcal{H}$ for which exponential deviation bounds can be obtained for both processes.

To present the results, we introduce some terminologies. Let $\vc{\mathcal{S}}$ be the VC dimension of a class $\mathcal{S}$ of sets. The VC dimension $\vc{\mathcal{G}}$ of a class $\mathcal{G}$ of functions from $\mathcal{X}$ to $\R$ is defined to be the VC dimension of the set of subgraphs $\mathcal{S}_{\mathcal{G}}:=\{\{(x,z)\in\mathcal{X}\times \R:z<g(x) \}:g\in\mathcal{G}\}$. $\mathcal{G}$ is called VC-subgraph if $\vc{\mathcal{G}}<\infty$. For a vector, $\Vert \cdot \Vert_p$ represents its $l_p$-norm for $p\geq 1$. For a function $\xi:\mathcal{X}\to\R$, we denote by $\Vert \xi \Vert_{\psi_2}:=\inf\{c>0:\mathbb{E}_{\pi_X}[\exp(\xi^2(X)/c^2)]\leq 2\}$ its sub-Gaussian norm under the distribution $\pi_X$, and denote by $\Vert \xi \Vert_p:=\big(\mathbb{E}_{\pi_X}[\lvert \xi(X)\rvert^p]\big)^{1/p}$ its $L_p$ norm.
\begin{theorem}[Joint coverage-width guarantee for VC-subgraph class]\label{rate:vc-subgraph}
If the hypothesis class $\mathcal{H}$ is such that the augmented class $\mathcal{H}_+:=\{h+c:h\in \mathcal{H},c\in\R\}$ is VC-subgraph, and $H(x):=\sup_{h\in\mathcal{H}}\lvert h(x)-\mathbb{E}_{\pi_X}[h(X)]\rvert$ satisfies $\Vert H \Vert_{\psi_2}<\infty$, then the deviation bounds in Theorem \ref{convergence rate} satisfy
{\small
\begin{equation}
\begin{aligned}
&\phi_1(n,\epsilon,\mathcal{H})\leq 2\exp\Big( -\frac{n\epsilon^2}{C\Vert H \Vert_{\psi_2}^2\vc{\mathcal{H}_+}} \Big), \\ &\phi_2(n,t,\mathcal{H})\leq 
    \begin{cases}
    4^{n+1}\exp(-t^2n)\;&\text{if}\;n< \frac{C\vc{\mathcal{H}}}{2}\\
    4\Big(\frac{2en}{C\vc{\mathcal{H}}}\Big)^{C\vc{\mathcal{H}}}\exp(-t^2n)\;&\text{if}\;n\geq \frac{C\vc{\mathcal{H}}}{2}
    \end{cases}
\end{aligned}    
\end{equation}}
where $C$ is a universal constant, and $e$ is the base of the natural logarithm. If Assumptions \ref{class: non-negative translation}-\ref{conditional density: positiveness} also hold (in which case $\mathcal{H}=\mathcal{H}_+$), then for every $\eta \in (0,1)$, when $n\geq \frac{C\vc{\mathcal{H}}}{2}$ and we set $t:=\sqrt{\frac{1}{n}\log\frac{8}{\eta}+\frac{C\vc{\mathcal{H}}}{n}\log\frac{2en}{C\vc{\mathcal{H}}}}\leq \frac{\alpha}{4}$, with probability at least $1-\eta$, we have, for every optimal solution $(\hat L_{t}^*,\hat U_{t}^*)$ of $\widehat{\opt} (t)$, that $\mathbb{P}_{\pi}(Y\in [\hat L_{t}^*(X),\hat U_{t}^*(X)])\geq 1-\alpha$ and
{\small
\begin{equation*}
\begin{aligned}
 &\mathbb{E}_{\pi_X}[\hat U_{t}^*(X)-\hat L_{t}^*(X)]\leq \mathcal{R}^*(\mathcal{H})\\
 &+\frac{24}{\alpha \gamma_{\frac{\alpha}{6}}}\sqrt{\frac{1}{n}\log\frac{8}{\eta}+\frac{C\vc{\mathcal{H}}}{n}\log\frac{2en}{C\vc{\mathcal{H}}}}\\
 &+\sqrt{\frac{\vc{\mathcal{H}_+}}{n}\cdot 16C\Vert H \Vert_{\psi_2}^2\log\frac{4}{\eta}}.
\end{aligned} 
\end{equation*}}
\end{theorem}
Theorem \ref{rate:vc-subgraph} reveals that, after ignoring logarithmic factors, the sample size $n$ needed to learn a good PI with guaranteed coverage from $\widehat{\opt}(t)$ is of order $\Omega(\vc{\mathcal{H}})$, if $t$ of order $O\big(\sqrt{\vc{\mathcal{H}}/n}\big)$ is adopted. The corresponding optimality gap (in width) is $O\big(\sqrt{\vc{\mathcal{H}_+}/n}\big)$. Appendix \ref{sec:vc-subgraph} provides further discussion regarding the use of $\mathcal H_+$ versus $\mathcal H$ in the bound.


Similar sample complexities for VC-major $\mathcal{H}$ have been proposed in \citet{rosenfeld2018discriminative} for a formulation that can be viewed as the dual of \eqref{OP1}, where the coverage rate is maximized under a mean width constraint. Comparing VC-major and VC-subgraph classes, they both cover many hypothesis classes commonly used in practice, e.g., linear functions, regression trees, and neural networks that are to be discussed momentarily. Nonetheless, in general neither VC-major nor VC-subgraph implies the other, and therefore our VC-subgraph results are in parallel to those in \citet{rosenfeld2018discriminative}. More importantly, their results provide finite-sample errors only for the coverage rate, but not the interval width, whereas we characterize performances \textit{jointly} in coverage and width, thanks to our novel sensitivity measure from Theorem \ref{sensitivity bound}. Moreover, our results appear to bypass a technical issue of  \citet{rosenfeld2018discriminative}. A key result there is that for a VC-major class $\mathcal{H}$, the induced set of between-graphs $\{\{(x,z):L(x)\leq z\leq U(x)\}:L,U\in\mathcal{H}\text{ and }L\leq U\}$ is a VC class. When the class $\mathcal{H}$ is uniformly bounded from below, say $0$, then this is equivalent to $\mathcal{H}$ being VC-subgraph. However, a counter-example for this conclusion is constructed in Theorem 2.1, statement f, in \cite{dudley1987universal}. Our approach overcomes such technical ambiguities.

Our next considered hypothesis class is on Lipschitz continuity with respect to the class parameter:
\begin{theorem}[Joint coverage-width guarantee for the Lipschitz class in parameter]\label{rate:Lipschitz}
Suppose $\mathcal{H}=\{h(\cdot,\theta):\theta\in\Theta\}$ where the parameter space $\Theta$ is a bounded set in $\R^l$. If the functions are Lipschitz continuous in the parameter, i.e.,
$   \lvert h(x,\theta_1) - h(x,\theta_2)\rvert \leq \mathcal{L}(x)\Vert \theta_1-\theta_2 \Vert_2,\;\text{for all}\;\theta_1,\theta_2\in\Theta,x\in\mathcal{X}$
for some $\mathcal{L}:\mathcal{X}\to \R$ such that $\Vert \mathcal{L} \Vert_{2}<\infty$, and $H(x):=\sup_{\theta\in\Theta}\lvert h(x,\theta) - \mathbb{E}_{\pi_X}[h(X,\theta)]\rvert$ satisfies $\Vert H \Vert_{\psi_2}<\infty$, then the deviation bounds in Theorem \ref{convergence rate} satisfy

{\small
$$
\begin{aligned}
 &\phi_1(n,\epsilon,\mathcal{H})\leq 2\exp\Big( -\frac{\epsilon^2n}{C\max\{\log\frac{\diam{\Theta}\Vert \mathcal{L}\Vert_2}{\Vert H\Vert_2},1\}\Vert H\Vert^2_{\psi_2}l} \Big),\\
&\phi_2(n,t,\mathcal{H})\leq \Big(C_{\mathcal{H}}^{2C\log\log C_{\mathcal{H}}} \cdot\max\{\frac{t^2n}{2Cl},1\}\big) \Big)^{2l}\exp(-2t^2n)
\end{aligned}$$}
where $\diam{\Theta}:=\sup_{\theta_1,\theta_2\in \Theta}\Vert \theta_1-\theta_2 \Vert_2$ is the diameter of $\Theta$, $D_{Y\vert X}:=\sup_{x,y}p(y\vert x)$ is the supremum of the conditional density, $C_{\mathcal{H}}:=12\diam{\Theta}D_{Y\vert X}\Vert \mathcal{L}\Vert_1$, and $C$ is a universal constant. If Assumptions \ref{class: non-negative translation}-\ref{conditional density: positiveness} also hold, then for every $\eta \in (0,1)$, when we set $t:=\sqrt{\frac{1}{n}\log\frac{2}{\eta}+\frac{l}{n}\cdot 4C\log (C_{\mathcal{H}})\log\log (C_{\mathcal{H}})}\leq \frac{\alpha}{4}$, with probability at least $1-\eta$, we have, for every optimal solution $(\hat L_t^*,\hat U_t^*)$ of $\widehat{\opt} (t)$, that $\mathbb{P}_{\pi}(Y\in [\hat L_t^*(X),\hat U_t^*(X)])\geq 1-\alpha$ and
{\small
\begin{equation}
\begin{aligned}
&\mathbb{E}_{\pi_X}[\hat U_t^*(X)-\hat L_t^*(X)]\leq \mathcal{R}^*(\mathcal{H})\\
&+\frac{24}{\alpha \gamma_{\frac{\alpha}{6}}}\sqrt{\frac{1}{n}\log\frac{2}{\eta}+\frac{l}{n}\cdot 4C\log (C_{\mathcal{H}}) \log\log (C_{\mathcal{H}}})\\
&+\sqrt{\frac{l}{n}\cdot 16C\max\big\{\log\frac{\diam{\Theta}\Vert \mathcal{L}\Vert_2}{\Vert H\Vert_2},1\big\}\Vert H \Vert_{\psi_2}^2\log\frac{4}{\eta}}. 
\end{aligned}    
\end{equation}}
\end{theorem}
The Lipschitz class is beyond the scope of \citet{rosenfeld2018discriminative}. Theorem \ref{rate:Lipschitz} states that the required sample size for this class to achieve a certain learning accuracy is of order $\Omega(l\log(\Vert \mathcal{L}\Vert_1))$, which depends on the dimension of the parameter space and the Lipschitz coefficient. Correspondingly, the optimality gap on the interval width is $O(\sqrt{l\log(\Vert \mathcal{L}\Vert_2)/n})$. Here the logarithmic factor associated with the Lipschitz coefficient is significant because $\Vert \mathcal{L}\Vert_2$ can be exponential in the number of layers for deep neural networks (see more details below). Note that our Lipschitzness is in the class parameter $\theta$ rather than the input $x$. The latter has recently been used to regularize neural networks to improve generalization gaps (e.g., \cite{bartlett2017spectrally,yoshida2017spectral,gouk2018regularisation}) and robustness against adversarial attacks (e.g., \cite{cisse2017parseval,hein2017formal,tsuzuku2018lipschitz}), but can potentially lead to a loss in expressiveness of the network \citep{huster2018limitations,anil2019sorting} because of the size restriction on network weights. Our result does not restrict the sensitivity of the network in the inputs, and in turn its expressiveness.

To further distinguish the technical novelty of Theorem \ref{rate:vc-subgraph} and Theorem \ref{rate:Lipschitz}, note that established results on uniform convergence bounds (UCBs) assuming VC or Lipschitzness on the hypothesis class $\mathcal{H}$ can only handle the objective (width)
on $\mathcal{H}$, but not the constraint (coverage) on a different hypothesis class of indicator functions on the joint $(x,y)$ space. Our analysis explicitly controls this constraint complexity to achieve joint optimality-feasibility guarantees. The proof of Theorem \ref{rate:vc-subgraph}  involves bounding the VC-dimension of the indicator class through intersections of VC set classes, while that of Theorem \ref{rate:Lipschitz} requires directly bounding the bracketing number and evaluating the entropy integral. Moreover, even for the objective, it appears that the explicit UCBs with potentially unbounded VC or Lipschitz classes considered here are new to the best of our knowledge.

We showcase the application of Theorems \ref{rate:vc-subgraph} and \ref{rate:Lipschitz} in two important classes: Regression trees and neural networks. Appendix \ref{sec:linear} presents an additional class of linear hypothesis.



\textbf{Regression Tree. }
Suppose we build two binary trees to construct $L$, $U$ respectively. The tree has at most $S+1$ terminal nodes, and each non-terminal node is split according to $x_{i^*}\leq q$ or $x_{i^*}> q$ for some $i^*\in \{1,\ldots,d\}$ and $q\in \R$. In other words, at most $S$ splits are allowed. A regression tree constructed this way is therefore a piece-wise constant function: $h(x)=\sum_{s=1}^{S+1}c_sI_{x\in R_s}$, where each $c_s\ge 0$, $\cup_{s=1}^{S+1}R_s=\mathcal{X}$, and each $R_s$ is a hyper-rectangle in $\R^d$ that takes the form
\begin{equation*}
    R_s=\left\{x\in\mathcal{X}:
    \begin{array}{l}
       x_{i_{k,1}}\leq q_{i_{k,1}}\;\text{for}\;k=1,\ldots,S_1\\
       x_{i_{k,2}}> q_{i_{k,2}}\;\text{for}\;k=1,\ldots,S_2  \\
       \text{each }q_{i_{k,1}},q_{i_{k,2}}\in [-\infty,+\infty]\\
        0\leq S_1+S_2\leq S
    \end{array}\right\}.
\end{equation*}
Let hypothesis class $\mathcal{H}$ be the collection of all such regression trees. We use Theorem \ref{rate:vc-subgraph}. Note that the augmented class $\mathcal{H}_+=\mathcal{H}$. Since the regression tree takes constant values on each rectangle, its subgraph in the space $\mathcal{X}\times \R$ is a union of hyper-rectangles, i.e., 
   $ \{(x,z)\in\mathcal{X}\times \R:h(x)>z\}=\cup_{s=1}^{S+1}\big(R_s\cup (-\infty,c_s)\big)$.
Note that each $R_s\cup (-\infty,c_s)$ is an intersection of at most $S+1$ axis-parallel cuts in $\R^{d+1}$, and the set of all axis-parallel cuts is shown to have a VC dimension $O(\log(d))$ \citep{gey2018vapnik}. $\vc{\mathcal{H}}$ can therefore be obtained by applying VC bounds for unions and intersections of VC classes of sets \citep{van2009note}:
\begin{theorem}[Regression tree]\label{vc bound:regression tree}
The class $\mathcal{H}$ of regression trees described as above is the same class as its augmentation $\mathcal{H}_+$, and is VC-subgraph with $\vc{\mathcal{H}} = \vc{\mathcal{H}_+}\leq CS^2(\log(S))^2\log(d)$ for some universal constant $C$. If the trees are constructed in such a way that $\max_x h(x) - \min_x h(x)\leq M<\infty$, then Theorem \ref{rate:vc-subgraph} can be applied with $\Vert H\Vert_{\psi_2}\leq C'M$ for another universal constant $C'$.
\end{theorem}

We note that, although upper bounds of VC dimension have been available for classification trees with binary features, and bounds for classification trees with continuous features appeared very recently \citep{leboeuf2020decision}, here we consider continuous-valued regression trees with continuous features whose VC dimensions have not been addressed by previous works. Our proof of Theorem \ref{vc bound:regression tree} is based on recently developed VC results for axis-parallel cuts \citep{gey2018vapnik}.

\textbf{Neural Network. }
Consider the class $\mathcal{H}$ of feedforward neural networks with a fixed architecture and fixed activation functions, indexed by the weights and biases. Suppose the network has $S-1$ hidden layers, one input layer, and two output units. Denote by $W$ the total number of parameters for weights and biases, and by $U$ the total number of computation units (neurons). Let $O_s\in \R^{n_s},s=0,\ldots,S$ be the output of the $s$-th layer. Then $O_s=\phi_s(W_sO_{s-1}+b_s)$ where $W_s\in \R^{n_s\times n_{s-1}}$ is a matrix of weights, $b_s \in \R^{n_s}$ is a vector of biases, and $\phi_s=(\phi_{s,1},\ldots,\phi_{s,n_s})$ is a vector of activation functions. $n_s$ denotes the number of neurons in the $s$-th layer. In particular $O_0 = x$ is the input vector and $O_S=(L(x),U(x))$ is the final output vector. We aim to characterize the class of one output unit $L(x)$ since the class of $U(x)$ is the same.


We utilize Theorem \ref{rate:Lipschitz}. This approach can advantageously handle activation functions beyond sigmoid and piece-wise polynomial (An alternate approach, which we present in Appendix \ref{sec:NN}, uses Theorem \ref{rate:vc-subgraph} and Pollard's pseudo-dimension \citep{pollard2012convergence} that applies to sigmoid and piece-wise polynomial). Assume that each activation function $\phi_{s,k}$ is globally $M$-Lipschitz so that for some constant $M_0$ the growth condition $\lvert \phi_{s,k}(z) \rvert\leq M_0+M\lvert z\rvert$ holds for all $z\in \R$, and that each weight or bias parameter is restricted to the bounded interval $[-B,B]$ for some $B>0$. To apply Theorem \ref{rate:Lipschitz}, we show the following Lipschitz property by a backpropagation-like calculation:
\begin{theorem}[Neural network]\label{Lipschitz property for neural network}
The neural network class
$\mathcal{H}=\{h(\cdot,\theta):\theta \in [-B,B]^{W}\}$ defined above satisfies the Lipschitzness condition with
  $  \mathcal{L}(x)=C\sqrt{S}(BM\sqrt{W})^S(\Vert x\Vert_2+M_0\sqrt{U}+BM\sqrt{W})$
where $C$ is a universal constant. Therefore Theorem \ref{rate:Lipschitz} can be applied with $l=W$, $\diam{\Theta} = 2B\sqrt{W}$, and $\Vert \mathcal{L}\Vert_2 \leq C\sqrt{S}(BM\sqrt{W})^S(\Vert\Vert X\Vert_2\Vert_2+M_0\sqrt{U}+BM\sqrt{W})$.
\end{theorem}
We make two remarks. First, the sample size required in Theorem \ref{rate:Lipschitz} to achieve a certain learning accuracy is of order $l\log(C_{\mathcal{H}})\log\log(C_{\mathcal{H}})$. Applying the Lipschitz constants from Theorem \ref{Lipschitz property for neural network} to evaluate the $C_{\mathcal{H}}$ reveals a required sample size of order $WS$, up to logarithmic factors, for neural networks. Second, the size restriction $B$ on the weights and biases enters into the error bounds in a logarithmic manner. Therefore $B$ is allowed to be (exponentially) large and exerts little impact on the training of the network. 


\section{Data-Driven Coverage Calibration}\label{sec:calibration}
In this section, we propose a general-purpose PI calibration method to balance coverage and width performances in practice. On a high level, our proposal selects the margin $t$ in \eqref{OP2} in a data-driven manner to guarantee a coverage maintenance.

More precisely, we recall that standard practice in validation requires 1) training models multiple times each with different hyperparameters, and then 2) evaluating the trained models on a validation set to select the optimal one. Our proposal aims to select the optimal PI model in 2). In the following, we thus assume multiple ``candidate'' models are already available.


Algorithm \ref{calibration:normalized} shows our procedure, which simultaneously outputs $K$ PIs, each at a given prediction level $1-\alpha_k$ ($k=1,...,K$). It starts from a candidate set of PI models, called $\{\PI_j(x)=[L_j(x),U_j(x)]:j=1,\ldots,m\}$. These models can be obtained from setting $m$ different values at a ``tradeoff" parameter (e.g., the dual multiplier in a Lagrangian formulation of the empirical constrained optimization; see Appendix \ref{sec:calibrate NN} for a neural net example), but can also be a more general collection of PI models. We then use a validation data set $\mathcal{D}_v:=\{(X'_i,Y'_i):i=1,\ldots,n_v\}$, independent of the PI training, to check the feasibility of each candidate PI using the criterion $\hat{\CR}(\PI_j):=(1/n_v)\sum_{i=1}^{n_v}I_{Y'_i\in \PI_j(X'_i)}\geq 1-\alpha_k+\epsilon_j$ for some selected margins $\epsilon_j$. 

\begin{algorithm}[h]
\caption{Normalized PI Calibration}
\label{calibration:normalized}
\DontPrintSemicolon
\SetKwInOut{Input}{Input}\SetKwInOut{Output}{Output}

\textbf{Input:} Candidate PIs $\{\PI_j=[L_j,U_j]:j=1,\ldots,m\}$, target coverage rates $\{1-\alpha_k\in(0,1):k=1,\ldots,K\}$, calibration data $\mathcal{D}_v=\{(X'_i,Y'_i):i=1,\ldots,n_v\}$, and confidence level $1-\beta\in(0,1)$.
\BlankLine
\textbf{Procedure:}\;

\textbf{1.}~For each $\PI_j$, $j=1,\ldots,m$, compute its empirical coverage rate on $\mathcal{D}_v$, $\hat{\CR}(\PI_j):=\frac{1}{n_v}\sum_{i=1}^{n_v}I_{Y'_i\in \PI_j(X'_i)}$. Compute the sample covariance matrix $\hat \Sigma\in \R^{m\times m}$ with $\hat \Sigma_{j_1,j_2} = \frac{1}{n_v}\sum_{i=1}^{n_v}\big(I_{Y'_i\in \PI_{j_1}(X'_i)}-\hat{\CR}(\PI_{j_1})\big)\big(I_{Y'_i\in \PI_{j_2}(X'_i)}-\hat{\CR}(\PI_{j_2})\big)$.\;

\textbf{2.}~Let $\hat{\sigma}^2_j=\hat\Sigma_{j,j}$, and compute $q_{1-\beta}$, the $(1-\beta)$-quantile of $\max_{1\leq j\leq m}\{Z_j/\hat{\sigma}_j:\hat{\sigma}_j>0\}$ where $(Z_1,\ldots,Z_m)$ is a multivariate Gaussian with mean zero and covariance $\hat \Sigma$.\;

\textbf{3.}~For each coverage rate $1-\alpha_k$, $k=1,\ldots,K$, compute
\begin{equation*}
\begin{aligned}
&j^*_{1-\alpha_k} = \argmin{1\leq j\leq m}\Big\{\frac{1}{n+n_v}\Big(\sum_{i=1}^n\lvert \PI_j(X_i)\rvert+\sum_{i=1}^{n_v}\lvert \PI_j(X'_i)\rvert\Big)\\
&:\hat{\CR}(\PI_j)\geq 1-\alpha_k + \frac{ q_{1-\beta}\hat{\sigma}_j}{\sqrt{n_v}}\Big\}    
\end{aligned}
\end{equation*}
where $\lvert \PI_j(\cdot) \rvert:=U_j(\cdot)-L_j(\cdot)$ is the width, and $\{X_i\}_{i=1}^n$ is the training data set.\;

\BlankLine
\Output{$\PI_{j^*_{1-\alpha_k}}$ for $k=1,\ldots,K$.}
\end{algorithm}

The key of our procedure is to tune $\epsilon_j$ based on a uniform central limit theorem (CLT) that captures the overall errors incurred in the empirical coverage rates. Denoting by $\CR(\PI_j):=\mathbb{P}_{\pi}(Y\in \PI_j(X))$ the true coverage rate of $\PI_j$, this CLT implies that, setting $q_{1-\beta}=(1-\beta)$-quantile of $\max_j Z_j/\hat{\sigma}_j$ for some properly chosen Gaussian vector $(Z_j)_{j=1,\ldots,m}$, we have
$\CR(\PI_j)\geq \hat{\CR}(\PI_j)-q_{1-\beta}\hat\sigma_j/\sqrt{n_v}$
for all $j=1,\ldots,m$ uniformly with probability $\approx 1-\beta$.    
The uniformity over $j$ ensures the solution in Step 3, which optimizes the interval width, indeed attains feasibility (target coverage) with $1-\beta$ confidence. In this ``meta-optimization", we pool the training and validation sets together in the objective to improve the width performance. We have the following finite-sample guarantee:
\begin{theorem}\label{feasibility:normalized validator_simple}
Let $1-\underline{\alpha}:=\max_{j=1,\ldots,m}\CR(\PI_j)$, $1-\alpha_{\min}:=1-\min_{k=1,\ldots,K}\alpha_k$, and $\tilde{\alpha}:=\min\{\alpha_{\min},1-\max_{k=1,\ldots,K}\alpha_k\}$. For every collection of interval models $\{\PI_j:1\leq j\leq m\}$, every $n_v$, and $\beta\in(0,\frac{1}{2})$, the PIs output by Algorithm \ref{calibration:normalized} satisfy
\begin{eqnarray}
\notag &&\mathbb P_{\mathcal D_v}(\CR(\PI_{j^*_{1-\alpha_k}})\geq 1-\alpha_k\; \text{for all}\; k=1,\ldots,K)\\
\notag &\geq& 1-\beta-C_1\bigg(\Big(\frac{\log^7(mn_v)}{n_v\tilde{\alpha}}\Big)^{\frac{1}{6}}+\\
&&\exp\big(-C_2n_v\min\big\{\epsilon,\frac{\epsilon^2}{\underline{\alpha}(1-\underline{\alpha})}\big\}\big)\bigg)\label{finite sample error:normalized_simple}
\end{eqnarray}
with $\epsilon=\max\big\{\alpha_{\min}-\underline{\alpha}-C_1\big((\underline{\alpha}(1-\underline{\alpha})/n_v+\log(n_v\alpha_{\min})/n_v^2)\log( m/\beta)\big)^{1/2},0\big\}$, where $\mathbb P_{\mathcal D_v}$ denotes the probability with respect to the calibration data, and $C_1,C_2$ are universal constants.
\end{theorem}

The most important implication of Theorem \ref{feasibility:normalized validator_simple} is that the finite-sample deterioration in the confidence level using our procedure depends only \emph{logarithmically} on $m$ and is \emph{independent} of $K$. These enable both the use of many candidate models and the output of many PIs at different prediction levels. The latter implies that our algorithm can advantageously generate all PIs for arbitrarily many target levels simultaneously with a single validation exercise, thus is computationally cheap and comprises a strength.  We provide further interpretation on the error terms of \eqref{finite sample error:normalized_simple} in Appendix \ref{sec:calibration discussion}. Besides coverage attainment guarantee in Theorem \ref{feasibility:normalized validator_simple}, our calibration procedure also possesses guaranteed performance regarding the optimality of the width, provided that only the calibration data are used to assess the width in Step 3 of Algorithm \ref{calibration:normalized}. More details can be found in Appendix \ref{sec:calibration discussion}.

In addition, we provide and compare an alternate calibration scheme, viewed as an ``unnormalized" (as opposed to ``normalized") version of Algorithm \ref{calibration:normalized} when handling the standard error $\hat\sigma_j$, in Appendix \ref{sec:alternate}.


We discuss Algorithm \ref{calibration:normalized} in relation to a naive, but natural approach that simply selects the model with the shortest average interval width, among candidate models with empirical coverage rates on the validation dataset reaching the target levels (i.e., $t=0$ in \eqref{OP2}). This latter approach is also known as the PAV validation scheme in \citet{kivaranovic2020adaptive} (which we call NNVA in Section \ref{sec:experiments}). Our algorithms improve PAV in two aspects. First, our proposal guarantees with high probability that the target level will be achieved by the test coverage thanks to a corrective margin, while PAV does not offer such a guarantee and tends to fall short. Second, our proposal enjoys a higher statistical power than PAV in that unlike PAV whose analysis is based on concentration bounds, Algorithm \ref{calibration:normalized} is analyzed via an asymptotically tight joint CLT. These guarantees build on recent high-dimensional Berry-Esseen bounds \citep{chernozhukov2017central} (which notably does not require functional complexity measures but only the geometry of ``hit sets''). Moreover, we will observe in the experiments in Section \ref{sec:experiments} that our proposal empirically performs better than PAV.

\section{Experiments}\label{sec:experiments}

\begin{table*}[h]
\small
  \caption{Single PI at the 95\% prediction level. The best results are in \textbf{bold}.}
  \label{single-level}
  \centering
  \renewcommand{\arraystretch}{1}
  \setlength{\tabcolsep}{2pt}
  \begin{tabular}{|c|cc|cc|cc|cc|cc|cc|cc|cc|}
    \hline
&\multicolumn{2}{c|}{Synthetic1}& \multicolumn{2}{c|}{Synthetic2}& \multicolumn{2}{c|}{Synthetic3} &\multicolumn{2}{c|}{Boston} &\multicolumn{2}{c|}{Concrete}&\multicolumn{2}{c|}{Energy}&\multicolumn{2}{c|}{Wine} & \multicolumn{2}{c|}{Yacht}\\
Methods &EP & IW& EP & IW& EP & IW &EP & IW &EP & IW & EP & IW & EP & IW & EP & IW\\
\hline

QRF & 1.00 & 3.122 & 1.00 & 3.667 & 1.00  & 3.609 & 0.46 & 2.085 & 0.72 & 1.995 & 0.78 & 0.664 & 0.00 & 2.671 & 0.22 & 1.110\\
 
CV+ & 1.00 & 0.302 & 1.00 & 2.039 & 1.00 & 3.523 & 0.96 & \textbf{1.538} & 0.88 & 1.413 & 0.98 & \textbf{0.500} & 1.00 & 4.359 & 0.92 & 0.339\\

SCQR & 0.90 & 2.264 & 0.78 & 2.863 & 0.98 & 2.804 & 0.66 & 2.309 & 0.62 & 2.289 & 0.68 & 0.837 & 0.30 & 2.810 & 0.86 & 1.681\\

SVMQR & 0.00 & 0.256 & 0.00 & 2.351 & 0.00 & 1.855 & 0.06 & 1.340 & 0.10 & 1.376 & 0.18 & 0.411 & 0.48 & 2.975 & 0.30 & 0.483\\

SCL & 0.98 & 0.313 & 0.70 & 2.211 & 1.00 & 3.754 & 0.88 & 2.053 & 0.88 & 1.630 & 0.74 & 0.769 & 0.80 & 4.437 & 0.92 & 0.906\\

NNVA & 0.00 & 0.221 & 0.74 & 1.630 & 0.04 & 1.991 & 0.58 & 1.837 & 0.28 & 1.607 & 0.44 & 0.234 & 0.00 & 1.817 & 0.80 & 0.154\\

Ours-NNGN & 1.00 & \textbf{0.296} & 1.00 & \textbf{1.921} & 1.00 &  \textbf{2.176} & 0.90 & 2.477 & 0.86 & 2.375 & 0.62 & 0.398 & 0.74 & 2.155 & 0.92 & \textbf{0.217}\\

Ours-NNGU & 1.00 & \textbf{0.296} & 1.00 & 2.557 & 1.00 & 3.155 & 0.96 & 2.692 & 0.96 & \textbf{2.643} & 1.00 & 0.561 & 0.98 & \textbf{2.648} & 1.00 & 0.299\\
\hline
 \end{tabular}
\end{table*}

\begin{table*}[h]
\small
  \caption{Simultaneous PIs at $19$ target prediction levels.  The best results are in \textbf{bold}.}
  \label{multiple-level}
  \centering
  \renewcommand{\arraystretch}{1}
  \setlength{\tabcolsep}{2pt}
  \begin{tabular}{|c|cc|cc|cc|cc|cc|cc|cc|cc|}
    \hline
&\multicolumn{2}{c|}{Synthetic1}& \multicolumn{2}{c|}{Synthetic2}& \multicolumn{2}{c|}{Synthetic3} &\multicolumn{2}{c|}{Boston} &\multicolumn{2}{c|}{Concrete}&\multicolumn{2}{c|}{Energy}&\multicolumn{2}{c|}{Wine} & \multicolumn{2}{c|}{Yacht}\\
Methods &MEP & MIW& MEP & MIW& MEP & MIW &MEP & MIW &MEP & MIW & MEP & MIW & MEP & MIW & MEP & MIW\\ 
\hline

QRF  & 1.00 & 1.911 & 0.92 & 1.657 & 1.00 & 1.805 & 0.46 & 1.099 & 0.72 & 1.108 & 0.78 & 0.328 & 0.00 & 1.460 & 0.22 & 0.703\\

CV+ &  1.00 & 0.176 & 1.00 & 1.078 & 1.00 & 1.935 & 0.66 & 0.780 & 0.60 & 0.736 & 1.00 & 0.245 & 1.00 & 2.289 & 0.86 & 0.114\\

SCQR & 0.74 & 1.356 & 0.74 & 2.218 & 0.56 & 1.808 & 0.14 & 1.622 & 0.16 & 1.575 & 0.24 & 0.760 & 0.00 & 2.493 & 0.20 & 1.554\\

SVMQR & 0.00 & 0.115 & 0.00 & 1.311 & 0.00 & 1.215 & 0.00 & 0.582 & 0.00 & 0.619 & 0.00 & 0.225 & 0.04 & 1.575 & 0.00 & 0.086\\

SCL & 0.24 & 0.175 & 1.00 & 1.004 & 0.96 & 1.940 & 0.78 & 1.039 & 0.66 & 0.896 & 0.74 & 0.373 & 0.70 & 2.470 & 0.80 & 0.325\\

NNVA & 0.00 & 0.160 & 0.28 & 0.969 & 0.04 & 1.606 & 0.26 & 1.086 & 0.16 & 0.956 & 0.00 & 0.151 & 0.00 & 1.147 & 0.04 & 0.079\\

Ours-NNGN & 1.00 & 0.170 & 1.00 & \textbf{1.050} & 1.00 & \textbf{1.727} & 0.76 & 1.244 & 0.72 & 1.092 & 0.90 & \textbf{0.177} & 0.72 & 1.267 & 0.96 & \textbf{0.120}\\

Ours-NNGU & 1.00 & \textbf{0.168} & 1.00 & 1.083 & 1.00 & 1.734 & \textbf{0.82} & 1.287 & \textbf{0.78} & 1.123 & 0.60 & 0.183 & 0.98 & \textbf{1.276} & 0.72 & 0.145\\
\hline
 \end{tabular}
\end{table*}


\textbf{Datasets.} We evaluate our approaches on both synthetic datasets and real-world benchmark datasets through comparisons to the state-of-the-arts. The real-world datasets (“Boston”, “Concrete”, “Wine”, “Energy” and “Yacht”) have been widely used in previous studies \citep{hernandez2015probabilistic,gal2016dropout,lakshminarayanan2017simple} for regression tasks. The generative distributions for the three synthetic datasets are: 
\begin{equation*}
    \begin{array}{lll}
\text{(1)}:
f(x) = \frac{c^T x}{2} + 10\sin(\frac{c^T x}{8}) + \frac{||x||_{2}}{10}\epsilon, x\sim N(0,I_{10}),
\\
\text{(2)}:
f(x) = \frac{1}{8}(c^T x)^{2}\sin(c^T x) + \frac{||x||_{2}}{10}\epsilon, x\sim N(0,I_{7}),
\\
\text{(3)}:
f(x) = \frac{1}{2} c^{T}x \cdot \cos(c^{T}x)^{2} + \frac{||x||_{2}}{10}\epsilon, x\sim N(0,I_{9}).
    \end{array}
\end{equation*}  
where $\epsilon\sim N(0,1)$, and $c$ is a constant in $[-2,2]^{10}$, $[-2,2]^7$, $[-2,2]^9$ respectively.

\textbf{Experimental Setup.} We conduct experiments under two scenarios: the \textbf{single PI case}, where one PI at a single prediction level $1-\alpha$ are constructed, and the \textbf{simultaneous PI case}, where $K$ PIs at $K$ different prediction levels $1-\alpha_1,\cdots,1-\alpha_K$ are constructed. Each trial is repeated for $N$ times to estimate the confidence of coverage attainment. We adopt neural networks as our PI models with the following loss function: $l(x,y,L,U):=(U(x)-L(x))^2 + \lambda (\max\{L(x)-y,0\}+\max\{y-U(x),0\})^2$, where $\lambda>0$ is a penalty for miscoverage and $(L(x),U(x))$ is the output vector containing the lower and upper bounds. By adjusting $\lambda$, PI models with different coverage levels can be trained, which are then fed into our calibration algorithms to obtain the final PI. We implement two calibration strategies: the normalized Gaussian PI calibration in Algorithm \ref{calibration:normalized} (NNGN), and the unnormalized version in Algorithm \ref{calibration:unnormalized} in Appendix \ref{sec:alternate} (NNGU). In addition, we also test the calibration scheme (NNVA) that directly compares the empirical coverage rates on the validation dataset to the target levels, without the Gaussian margin. Note that this is the PAV validation scheme in \citet{kivaranovic2020adaptive}.
 %



\textbf{Baselines.} We compare our NNGN, NNGU with the following state-of-art approaches: quantile regression forests (QRF) \citep{meinshausen2006quantile}, 
CV+ prediction interval (CV+) \citep{barber2019predictive}, split conformalized quantile regression (SCQR) \citep{romano2019conformalized}, quantile regression via SVM (SVMQR) \citep{steinwart2017liquidsvm}, and split conformal learning (SCL) \citep{lei2018distribution}. 
Implementation details can be found in Appendix \ref{sec:numerics additional}.




\textbf{Evaluation Metrics.} For the single PI analysis, our models are evaluated on both exceedance probability ($EP$) and interval width ($IW$). $EP$ captures the success in achieving the target confidence level, while $IW$ indicates the average interval width. For the simultaneous PI analysis, we use multiple exceedance probability ($MEP$) and multiple interval width ($MIW$). $MEP$ measures the proportion of trials where all PIs reach the target prediction levels simultaneously (i.e., family-wise correctness). Formally:
\begin{equation*}
    \begin{array}{cc}
EP:& \frac{1}{N}\sum_{i = 1}^{N}\mathbbm{1}\{CR_{i} \geq PL\}\\
IW:& \frac{1}{Nn}\sum_{i =1}^{N}\sum_{j =1}^{n}(U_{i,j}-L_{i,j})\\
MEP:&\frac{1}{N}\sum_{i = 1}^{N}\mathbbm{1}\{\bigcap^K_{k=1}\{CR_{i,k}\geq PL_k\}\}\\
MIW:& \frac{1}{KNn}\sum_{k = 1}^{K}\sum_{i = 1}^{N}\sum_{j
= 1}^{n}(U_{i,j,k}-L_{i,j,k})
    \end{array}
\end{equation*}
where $n$ is the size of testing data, $CR_i$ ($CR_{i,k}$) is the estimated coverage rate from the $i$-th repetition and $PL$ ($PL_k$) is the target prediction level $1-\alpha$ ($1-\alpha_k$). Throughout our experiments, the confidence level $1-\beta$ is set to $0.9$ in the calibration algorithms. For both cases, the best result is achieved by the model with the smallest $IW/MIW$ value among those with $EP/MEP \ge 0.90$. If no model achieves $EP/MEP \ge 0.90$, then the one with highest $EP/MEP$ is the best.

\textbf{Single PI Analysis.} Table \ref{single-level} reports the values of $EP$ and $IW$ for PI generation at 95\% prediction level on 3 synthetic datasets and 5 real-world datasets. It is shown that QRF, CV+ and our NNGU and NNGN are the only four methods that achieve the required confidence level in all synthetic cases, and among the four our NNGN consistently generates PIs of shortest width. Moreover, NNGU attains the target confidence level on all real datasets as well. NNGN seems to fall below the target confidence for some real datasets, but is better than all other baseline methods except CV+ and SCL. Among the methods with high $EP$, the interval widths of our NNGN and NNGU are the smallest in 6 out of 8 datasets. In contrast, the $EP$ values by NNVA are below the target confidence in all cases. Numerically, the averaged $EP$ of NNGN/NNGU is 1.4/1.7 times higher than the one of NNVA. This shows in particular that NNVA may fail to ensure a correct coverage rate with high confidence on test data, which necessitates our calibration approaches.

\textbf{Simultaneous PIs Analysis.} 
Table \ref{multiple-level} reports the values of $MEP$ and $MIW$ for simultaneous PIs at 19 target prediction levels: 50\%, 52.5\%, 55\%, 57.5\% $\ldots$, 95\%. Our calibration approaches NNGN and NNGU are always the best in terms of achieving the tightest interval under high $MEP$, or otherwise achieves the highest $MEP$ among all. Among the 6 datasets where the target confidence level 0.9 can be attained, NNGN/NNGU yields the smallest width in 4/2 of them. In the remaining 2 datasets, NNGU achieves the highest $MEP$. NNGU attains the target $MEP$ or the highest $MEP$ in 6 out of 8 datasets. Compared to the case of single-level target, the $MEP$ performance gaps are more significant in multi-level PI constructions. This is because the coverage rates in baseline algorithms get increasingly overfitted as more simultaneous target levels are compared against. Thanks to the use of the uniform safety margin, our NNGN and NNGU schemes are free of overfitting even in this case. These results demonstrate that our methods can accurately construct multiple PIs at different prediction levels simultaneously. Finally, compared to NNGU, NNGN tends to generate shorter PIs, and we recommend NNGN as the preferred choice.

We provide more experimental results in Appendix \ref{sec:numerics additional}. 

\section{Conclusion} In this paper, we study the generation of PIs for regression that satisfy an expected coverage rate. This problem can be cast into an empirical constrained optimization framework that minimizes the expected interval width subject to a coverage satisfaction constraint. We develop a general learning theory to characterize the optimality-feasibility tradeoff in this optimization, in particular joint guarantees on both a short expected interval width and an attainment of the target prediction level. We also propose a readily implementable calibration procedure, constructed based on a high-dimensional Berry-Esseen Theorem, to select the best PI model among trained candidates, which offers a practical approach to build simultaneous PIs at multiple target prediction levels with statistical validity. We demonstrate the empirical strengths of our proposed approach by applying it to neural-network-based PI models with our proposed calibration procedure, and comparing them with other baselines across synthetic and real-data examples.

\section*{Acknowledgments}
We gratefully acknowledge support from the National Science Foundation under grants CAREER CMMI-1834710 and IIS-1849280.

{
\bibliographystyle{abbrvnat}
\bibliography{reference}

\begin{thebibliography}{19}
\providecommand{\natexlab}[1]{#1}
\providecommand{\url}[1]{\texttt{#1}}
\expandafter\ifx\csname urlstyle\endcsname\relax
  \providecommand{\doi}[1]{doi: #1}\else
  \providecommand{\doi}{doi: \begingroup \urlstyle{rm}\Url}\fi

\bibitem[Anthony and Bartlett(2009)]{anthony2009neural}
M.~Anthony and P.~L. Bartlett.
\newblock \emph{Neural network learning: Theoretical foundations}.
\newblock cambridge university press, 2009.

\bibitem[Bartlett et~al.(1999)Bartlett, Maiorov, and Meir]{bartlett1999almost}
P.~L. Bartlett, V.~Maiorov, and R.~Meir.
\newblock Almost linear vc dimension bounds for piecewise polynomial networks.
\newblock In \emph{Advances in Neural Information Processing Systems}, pages
  190--196, 1999.

\bibitem[Bartlett et~al.(2019)Bartlett, Harvey, Liaw, and
  Mehrabian]{bartlett2019nearly}
P.~L. Bartlett, N.~Harvey, C.~Liaw, and A.~Mehrabian.
\newblock Nearly-tight vc-dimension and pseudodimension bounds for piecewise
  linear neural networks.
\newblock \emph{Journal of Machine Learning Research}, 20\penalty0
  (63):\penalty0 1--17, 2019.

\bibitem[Boucheron et~al.(2013)Boucheron, Lugosi, and
  Massart]{boucheron2013concentration}
S.~Boucheron, G.~Lugosi, and P.~Massart.
\newblock \emph{Concentration inequalities: A nonasymptotic theory of
  independence}.
\newblock Oxford university press, 2013.

\bibitem[Chernozhukov et~al.(2017)Chernozhukov, Chetverikov, Kato,
  et~al.]{chernozhukov2017central}
V.~Chernozhukov, D.~Chetverikov, K.~Kato, et~al.
\newblock Central limit theorems and bootstrap in high dimensions.
\newblock \emph{The Annals of Probability}, 45\penalty0 (4):\penalty0
  2309--2352, 2017.

\bibitem[Gey(2018)]{gey2018vapnik}
S.~Gey.
\newblock Vapnik--chervonenkis dimension of axis-parallel cuts.
\newblock \emph{Communications in Statistics-Theory and Methods}, 47\penalty0
  (9):\penalty0 2291--2296, 2018.

\bibitem[Goodfellow et~al.(2016)Goodfellow, Bengio, and
  Courville]{goodfellow2016deep}
I.~Goodfellow, Y.~Bengio, and A.~Courville.
\newblock \emph{Deep learning}.
\newblock MIT press, 2016.

\bibitem[Kosorok(2007)]{kosorok2007introduction}
M.~R. Kosorok.
\newblock \emph{Introduction to empirical processes and semiparametric
  inference}.
\newblock Springer Science \& Business Media, 2007.

\bibitem[Lam and Qian(2019)]{lam2019combating}
H.~Lam and H.~Qian.
\newblock Combating conservativeness in data-driven optimization under
  uncertainty: A solution path approach.
\newblock \emph{arXiv preprint arXiv:1909.06477}, 2019.

\bibitem[LeCun et~al.(2015)LeCun, Bengio, and Hinton]{lecun2015deep}
Y.~LeCun, Y.~Bengio, and G.~Hinton.
\newblock Deep learning.
\newblock \emph{nature}, 521\penalty0 (7553):\penalty0 436--444, 2015.

\bibitem[Maurer and Pontil(2009)]{maurer2009empirical}
A.~Maurer and M.~Pontil.
\newblock Empirical bernstein bounds and sample variance penalization.
\newblock \emph{arXiv preprint arXiv:0907.3740}, 2009.

\bibitem[Pearce et~al.(2018)Pearce, Zaki, Brintrup, and Neely]{pearce2018high}
T.~Pearce, M.~Zaki, A.~Brintrup, and A.~Neely.
\newblock High-quality prediction intervals for deep learning: A
  distribution-free, ensembled approach.
\newblock In \emph{International Conference on Machine Learning, PMLR: Volume
  80}, 2018.

\bibitem[Pollard(2012)]{pollard2012convergence}
D.~Pollard.
\newblock \emph{Convergence of stochastic processes}.
\newblock Springer Science and Business Media, 2012.

\bibitem[Talagrand(1994)]{talagrand1994sharper}
M.~Talagrand.
\newblock Sharper bounds for gaussian and empirical processes.
\newblock \emph{The Annals of Probability}, pages 28--76, 1994.

\bibitem[Van Der~Vaart and Wellner(2009)]{van2009note}
A.~Van Der~Vaart and J.~A. Wellner.
\newblock A note on bounds for vc dimensions.
\newblock \emph{Institute of Mathematical Statistics Collections}, 5:\penalty0
  103, 2009.

\bibitem[Van~der Vaart and Wellner(1996)]{van1996weak}
A.~W. Van~der Vaart and J.~A. Wellner.
\newblock \emph{Weak Convergence and Empirical Processes with Applications to
  Statistics}.
\newblock Springer, 1996.

\bibitem[Vapnik(2013)]{vapnik2013nature}
V.~Vapnik.
\newblock \emph{The nature of statistical learning theory}.
\newblock Springer Science and Business Media, 2013.

\bibitem[Vershynin(2018)]{vershynin2018high}
R.~Vershynin.
\newblock \emph{High-dimensional probability: An introduction with applications
  in data science}, volume~47.
\newblock Cambridge university press, 2018.

\bibitem[Zhang(2002)]{zhang2002covering}
T.~Zhang.
\newblock Covering number bounds of certain regularized linear function
  classes.
\newblock \emph{Journal of Machine Learning Research}, 2\penalty0
  (Mar):\penalty0 527--550, 2002.

\end{thebibliography}


\begin{thebibliography}{54}
\providecommand{\natexlab}[1]{#1}
\providecommand{\url}[1]{\texttt{#1}}
\expandafter\ifx\csname urlstyle\endcsname\relax
  \providecommand{\doi}[1]{doi: #1}\else
  \providecommand{\doi}{doi: \begingroup \urlstyle{rm}\Url}\fi

\bibitem[Alaa and van~der Schaar(2020)]{alaa2020discriminative}
A.~M. Alaa and M.~van~der Schaar.
\newblock Discriminative jackknife: Quantifying uncertainty in deep learning
  via higher-order influence functions.
\newblock \emph{arXiv preprint arXiv:2007.13481}, 2020.

\bibitem[Anil et~al.(2019)Anil, Lucas, and Grosse]{anil2019sorting}
C.~Anil, J.~Lucas, and R.~Grosse.
\newblock Sorting out lipschitz function approximation.
\newblock In \emph{International Conference on Machine Learning}, pages
  291--301, 2019.

\bibitem[Ankenman et~al.(2010)Ankenman, Nelson, and
  Staum]{doi:10.1287/opre.1090.0754}
B.~Ankenman, B.~L. Nelson, and J.~Staum.
\newblock Stochastic kriging for simulation metamodeling.
\newblock \emph{Operations Research}, 58\penalty0 (2):\penalty0 371--382, 2010.

\bibitem[Barber et~al.(2019)Barber, Candes, Ramdas, and
  Tibshirani]{barber2019predictive}
R.~F. Barber, E.~J. Candes, A.~Ramdas, and R.~J. Tibshirani.
\newblock Predictive inference with the jackknife+.
\newblock \emph{arXiv preprint arXiv:1905.02928}, 2019.

\bibitem[Bartlett et~al.(2017)Bartlett, Foster, and
  Telgarsky]{bartlett2017spectrally}
P.~L. Bartlett, D.~J. Foster, and M.~J. Telgarsky.
\newblock Spectrally-normalized margin bounds for neural networks.
\newblock In \emph{Advances in Neural Information Processing Systems}, pages
  6240--6249, 2017.

\bibitem[Box and Tiao(2011)]{box2011bayesian}
G.~E. Box and G.~C. Tiao.
\newblock \emph{Bayesian inference in statistical analysis}, volume~40.
\newblock John Wiley \& Sons, 2011.

\bibitem[Chernozhukov et~al.(2017)Chernozhukov, Chetverikov, Kato,
  et~al.]{chernozhukov2017central}
V.~Chernozhukov, D.~Chetverikov, K.~Kato, et~al.
\newblock Central limit theorems and bootstrap in high dimensions.
\newblock \emph{The Annals of Probability}, 45\penalty0 (4):\penalty0
  2309--2352, 2017.

\bibitem[Cisse et~al.(2017)Cisse, Bojanowski, Grave, Dauphin, and
  Usunier]{cisse2017parseval}
M.~Cisse, P.~Bojanowski, E.~Grave, Y.~Dauphin, and N.~Usunier.
\newblock Parseval networks: Improving robustness to adversarial examples.
\newblock In \emph{Proceedings of the 34th International Conference on Machine
  Learning-Volume 70}, pages 854--863. JMLR. org, 2017.

\bibitem[Doksum and Koo(2000)]{doksum2000spline}
K.~Doksum and J.-Y. Koo.
\newblock On spline estimators and prediction intervals in nonparametric
  regression.
\newblock \emph{Computational Statistics \& Data Analysis}, 35\penalty0
  (1):\penalty0 67--82, 2000.

\bibitem[Dudley(1987)]{dudley1987universal}
R.~Dudley.
\newblock Universal donsker classes and metric entropy.
\newblock \emph{The Annals of Probability}, pages 1306--1326, 1987.

\bibitem[Gal and Ghahramani(2016)]{gal2016dropout}
Y.~Gal and Z.~Ghahramani.
\newblock Dropout as a bayesian approximation: Representing model uncertainty
  in deep learning.
\newblock In \emph{International Conference on Machine Learning}, pages
  1050--1059, 2016.

\bibitem[Galv{\'a}n et~al.(2017)Galv{\'a}n, Valls, Cervantes, and
  Aler]{galvan2017multi}
I.~M. Galv{\'a}n, J.~M. Valls, A.~Cervantes, and R.~Aler.
\newblock Multi-objective evolutionary optimization of prediction intervals for
  solar energy forecasting with neural networks.
\newblock \emph{Information Sciences}, 418:\penalty0 363--382, 2017.

\bibitem[Gey(2018)]{gey2018vapnik}
S.~Gey.
\newblock Vapnik--chervonenkis dimension of axis-parallel cuts.
\newblock \emph{Communications in Statistics-Theory and Methods}, 47\penalty0
  (9):\penalty0 2291--2296, 2018.

\bibitem[Gouk et~al.(2018)Gouk, Frank, Pfahringer, and
  Cree]{gouk2018regularisation}
H.~Gouk, E.~Frank, B.~Pfahringer, and M.~Cree.
\newblock Regularisation of neural networks by enforcing lipschitz continuity.
\newblock \emph{arXiv preprint arXiv:1804.04368}, 2018.

\bibitem[Gupta et~al.(2019)Gupta, Kuchibhotla, and Ramdas]{gupta2019nested}
C.~Gupta, A.~K. Kuchibhotla, and A.~K. Ramdas.
\newblock Nested conformal prediction and quantile out-of-bag ensemble methods.
\newblock \emph{arXiv preprint arXiv:1910.10562}, 2019.

\bibitem[Hein and Andriushchenko(2017)]{hein2017formal}
M.~Hein and M.~Andriushchenko.
\newblock Formal guarantees on the robustness of a classifier against
  adversarial manipulation.
\newblock In \emph{Advances in Neural Information Processing Systems}, pages
  2266--2276, 2017.

\bibitem[Hern{\'a}ndez-Lobato and Adams(2015)]{hernandez2015probabilistic}
J.~M. Hern{\'a}ndez-Lobato and R.~Adams.
\newblock Probabilistic backpropagation for scalable learning of bayesian
  neural networks.
\newblock In \emph{International Conference on Machine Learning}, pages
  1861--1869, 2015.

\bibitem[Huster et~al.(2018)Huster, Chiang, and Chadha]{huster2018limitations}
T.~Huster, C.-Y.~J. Chiang, and R.~Chadha.
\newblock Limitations of the lipschitz constant as a defense against
  adversarial examples.
\newblock In \emph{Joint European Conference on Machine Learning and Knowledge
  Discovery in Databases}, pages 16--29. Springer, 2018.

\bibitem[Khosravi et~al.(2010)Khosravi, Nahavandi, Creighton, and
  Atiya]{khosravi2010lower}
A.~Khosravi, S.~Nahavandi, D.~Creighton, and A.~F. Atiya.
\newblock Lower upper bound estimation method for construction of neural
  network-based prediction intervals.
\newblock \emph{IEEE Transactions on Neural Networks}, 22\penalty0
  (3):\penalty0 337--346, 2010.

\bibitem[Khosravi et~al.(2011)Khosravi, Nahavandi, Creighton, and
  Atiya]{khosravi2011comprehensive}
A.~Khosravi, S.~Nahavandi, D.~Creighton, and A.~F. Atiya.
\newblock Comprehensive review of neural network-based prediction intervals and
  new advances.
\newblock \emph{IEEE Transactions on Neural Networks}, 22\penalty0
  (9):\penalty0 1341--1356, 2011.

\bibitem[Kim et~al.(2020)Kim, Xu, and Barber]{kim2020predictive}
B.~Kim, C.~Xu, and R.~F. Barber.
\newblock Predictive inference is free with the jackknife+-after-bootstrap.
\newblock \emph{arXiv preprint arXiv:2002.09025}, 2020.

\bibitem[Kivaranovic et~al.(2020)Kivaranovic, Johnson, and
  Leeb]{kivaranovic2020adaptive}
D.~Kivaranovic, K.~D. Johnson, and H.~Leeb.
\newblock Adaptive, distribution-free prediction intervals for deep networks.
\newblock In \emph{International Conference on Artificial Intelligence and
  Statistics}, pages 4346--4356, 2020.

\bibitem[Koenker and Hallock(2001)]{koenker2001quantile}
R.~Koenker and K.~F. Hallock.
\newblock Quantile regression.
\newblock \emph{Journal of Economic Perspectives}, 15\penalty0 (4):\penalty0
  143--156, 2001.

\bibitem[Lakshminarayanan et~al.(2017)Lakshminarayanan, Pritzel, and
  Blundell]{lakshminarayanan2017simple}
B.~Lakshminarayanan, A.~Pritzel, and C.~Blundell.
\newblock Simple and scalable predictive uncertainty estimation using deep
  ensembles.
\newblock In \emph{Advances in Neural Information Processing Systems}, pages
  6402--6413, 2017.

\bibitem[Leboeuf et~al.(2020)Leboeuf, LeBlanc, and
  Marchand]{leboeuf2020decision}
J.-S. Leboeuf, F.~LeBlanc, and M.~Marchand.
\newblock Decision trees as partitioning machines to characterize their
  generalization properties.
\newblock \emph{arXiv preprint arXiv:2010.07374}, 2020.

\bibitem[Lei et~al.(2015)Lei, Rinaldo, and Wasserman]{lei2015conformal}
J.~Lei, A.~Rinaldo, and L.~Wasserman.
\newblock A conformal prediction approach to explore functional data.
\newblock \emph{Annals of Mathematics and Artificial Intelligence}, 74\penalty0
  (1-2):\penalty0 29--43, 2015.

\bibitem[Lei et~al.(2018)Lei, G’Sell, Rinaldo, Tibshirani, and
  Wasserman]{lei2018distribution}
J.~Lei, M.~G’Sell, A.~Rinaldo, R.~J. Tibshirani, and L.~Wasserman.
\newblock Distribution-free predictive inference for regression.
\newblock \emph{Journal of the American Statistical Association}, 113\penalty0
  (523):\penalty0 1094--1111, 2018.

\bibitem[Meinshausen(2006)]{meinshausen2006quantile}
N.~Meinshausen.
\newblock Quantile regression forests.
\newblock \emph{Journal of Machine Learning Research}, 7\penalty0
  (Jun):\penalty0 983--999, 2006.

\bibitem[Olive(2007)]{olive2007prediction}
D.~J. Olive.
\newblock Prediction intervals for regression models.
\newblock \emph{Computational Statistics and Data Analysis}, 51\penalty0
  (6):\penalty0 3115--3122, 2007.

\bibitem[Papadopoulos et~al.(2001)Papadopoulos, Edwards, and
  Murray]{papadopoulos2001confidence}
G.~Papadopoulos, P.~J. Edwards, and A.~F. Murray.
\newblock Confidence estimation methods for neural networks: A practical
  comparison.
\newblock \emph{IEEE Transactions on Neural Networks}, 12\penalty0
  (6):\penalty0 1278--1287, 2001.

\bibitem[Papadopoulos(2008)]{papadopoulos2008inductive}
H.~Papadopoulos.
\newblock Inductive conformal prediction: Theory and application to neural
  networks.
\newblock In \emph{Tools in artificial intelligence}. Citeseer, 2008.

\bibitem[Pearce et~al.(2018)Pearce, Zaki, Brintrup, and Neely]{pearce2018high}
T.~Pearce, M.~Zaki, A.~Brintrup, and A.~Neely.
\newblock High-quality prediction intervals for deep learning: A
  distribution-free, ensembled approach.
\newblock In \emph{International Conference on Machine Learning, PMLR: Volume
  80}, 2018.

\bibitem[Pollard(2012)]{pollard2012convergence}
D.~Pollard.
\newblock \emph{Convergence of stochastic processes}.
\newblock Springer Science and Business Media, 2012.

\bibitem[Polonik(1997)]{polonik1997minimum}
W.~Polonik.
\newblock Minimum volume sets and generalized quantile processes.
\newblock \emph{Stochastic Processes and Their Applications}, 69\penalty0
  (1):\penalty0 1--24, 1997.

\bibitem[Romano et~al.(2019)Romano, Patterson, and
  Candes]{romano2019conformalized}
Y.~Romano, E.~Patterson, and E.~Candes.
\newblock Conformalized quantile regression.
\newblock In \emph{Advances in Neural Information Processing Systems}, pages
  3543--3553, 2019.

\bibitem[Rosenfeld et~al.(2018)Rosenfeld, Mansour, and
  Yom-Tov]{rosenfeld2018discriminative}
N.~Rosenfeld, Y.~Mansour, and E.~Yom-Tov.
\newblock Discriminative learning of prediction intervals.
\newblock In \emph{International Conference on Artificial Intelligence and
  Statistics}, pages 347--355, 2018.

\bibitem[Sacks et~al.(1989)Sacks, Welch, Mitchell, and Wynn]{sacks1989design}
J.~Sacks, W.~J. Welch, T.~J. Mitchell, and H.~P. Wynn.
\newblock Design and analysis of computer experiments.
\newblock \emph{Statistical Science}, pages 409--423, 1989.

\bibitem[Schmoyer(1992)]{schmoyer1992asymptotically}
R.~L. Schmoyer.
\newblock Asymptotically valid prediction intervals for linear models.
\newblock \emph{Technometrics}, 34\penalty0 (4):\penalty0 399--408, 1992.

\bibitem[Scott and Nowak(2006)]{scott2006learning}
C.~D. Scott and R.~D. Nowak.
\newblock Learning minimum volume sets.
\newblock \emph{Journal of Machine Learning Research}, 7\penalty0
  (Apr):\penalty0 665--704, 2006.

\bibitem[Seber and Lee(2012)]{seber2012linear}
G.~A. Seber and A.~J. Lee.
\newblock \emph{Linear regression analysis}, volume 329.
\newblock John Wiley \& Sons, 2012.

\bibitem[Steinberger and Leeb(2016)]{steinberger2016leave}
L.~Steinberger and H.~Leeb.
\newblock Leave-one-out prediction intervals in linear regression models with
  many variables.
\newblock \emph{arXiv preprint arXiv:1602.05801}, 2016.

\bibitem[Steinwart and Thomann(2017)]{steinwart2017liquidsvm}
I.~Steinwart and P.~Thomann.
\newblock liquidsvm: A fast and versatile svm package, 2017.

\bibitem[Steinwart et~al.(2011)Steinwart, Christmann,
  et~al.]{steinwart2011estimating}
I.~Steinwart, A.~Christmann, et~al.
\newblock Estimating conditional quantiles with the help of the pinball loss.
\newblock \emph{Bernoulli}, 17\penalty0 (1):\penalty0 211--225, 2011.

\bibitem[Stine(1985)]{stine1985bootstrap}
R.~A. Stine.
\newblock Bootstrap prediction intervals for regression.
\newblock \emph{Journal of the American Statistical Association}, 80\penalty0
  (392):\penalty0 1026--1031, 1985.

\bibitem[Tsuzuku et~al.(2018)Tsuzuku, Sato, and Sugiyama]{tsuzuku2018lipschitz}
Y.~Tsuzuku, I.~Sato, and M.~Sugiyama.
\newblock Lipschitz-margin training: Scalable certification of perturbation
  invariance for deep neural networks.
\newblock In \emph{Advances in Neural Information Processing Systems}, pages
  6541--6550, 2018.

\bibitem[Van Der~Vaart and Wellner(2009)]{van2009note}
A.~Van Der~Vaart and J.~A. Wellner.
\newblock A note on bounds for vc dimensions.
\newblock \emph{Institute of Mathematical Statistics Collections}, 5:\penalty0
  103, 2009.

\bibitem[Van~der Vaart and Wellner(1996)]{van1996weak}
A.~W. Van~der Vaart and J.~A. Wellner.
\newblock \emph{Weak Convergence and Empirical Processes with Applications to
  Statistics}.
\newblock Springer, 1996.

\bibitem[Vapnik(2013)]{vapnik2013nature}
V.~Vapnik.
\newblock \emph{The nature of statistical learning theory}.
\newblock Springer Science and Business Media, 2013.

\bibitem[Vovk(2012)]{vovk2012conditional}
V.~Vovk.
\newblock Conditional validity of inductive conformal predictors.
\newblock In \emph{Asian Conference on Machine Learning}, pages 475--490, 2012.

\bibitem[Vovk(2015)]{vovk2015cross}
V.~Vovk.
\newblock Cross-conformal predictors.
\newblock \emph{Annals of Mathematics and Artificial Intelligence}, 74\penalty0
  (1-2):\penalty0 9--28, 2015.

\bibitem[Vovk et~al.(2005)Vovk, Gammerman, and Shafer]{vovk2005algorithmic}
V.~Vovk, A.~Gammerman, and G.~Shafer.
\newblock \emph{Algorithmic learning in a random world}.
\newblock Springer Science \& Business Media, 2005.

\bibitem[Yoshida and Miyato(2017)]{yoshida2017spectral}
Y.~Yoshida and T.~Miyato.
\newblock Spectral norm regularization for improving the generalizability of
  deep learning.
\newblock \emph{arXiv preprint arXiv:1705.10941}, 2017.

\bibitem[Zhang et~al.(2019)Zhang, Zimmerman, Nettleton, and
  Nordman]{zhang2019random}
H.~Zhang, J.~Zimmerman, D.~Nettleton, and D.~J. Nordman.
\newblock Random forest prediction intervals.
\newblock \emph{The American Statistician}, pages 1--15, 2019.

\bibitem[Zhu et~al.(2019)Zhu, Lu, and Chen]{zhu2019hdi}
L.~Zhu, J.~Lu, and Y.~Chen.
\newblock {HDI}-forest: highest density interval regression forest.
\newblock In \emph{Proceedings of the 28th International Joint Conference on
  Artificial Intelligence}, pages 4468--4474. AAAI Press, 2019.

\end{thebibliography}
}

\newpage

\onecolumn
\aistatstitle{Learning Prediction Intervals for Regression: Generalization and Calibration: Supplementary Materials}

\appendix
We provide further results and discussions in this supplemental material. Appendix \ref{sec:consistency} presents results on the consistency of the obtained PI from the empirical constrained optimization. Appendix \ref{sec:vc-subgraph} provides additional discussion on Theorem \ref{rate:vc-subgraph}. Appendix \ref{sec:linear} shows the joint coverage-width guarantee for the linear hypothesis class. Appendix \ref{sec:NN} shows an alternate analysis for neural networks using Pollard's pseudo-dimension and our derived results for the VC-subgraph class. Appendix \ref{sec:calibration discussion} further discusses the finite-sample guarantees for our coverage calibration procedure. Appendix \ref{sec:alternate} presents and explains an alternate calibration procedure. Appendix \ref{sec:calibrate NN} discusses a Lagrangian formulation to train neural networks that construct PIs. Appendix \ref{sec:ep} reviews some background in empirical processes. Appendix \ref{sec:numerics additional} illustrates experimental details and additional experimental results. Finally, Appendix \ref{sec:proofs} shows all technical proofs. 

\section{Results on Basic Consistency}\label{sec:consistency}
This section presents our results regarding asymptotic consistency in using $\widehat{\opt} (t)$ to approximate the PI rendered by \eqref{OP1}. Assuming the weak uniform law of large numbers for both the empirical interval width and coverage rate, we first show the following general result:
\begin{theorem}[A general consistency result]\label{basic consistency}
Denote by $(\hat{L}^*_t,\hat{U}^*_t)$ an optimal solution of $\widehat{\opt} (t)$. Suppose Assumptions \ref{class: non-negative translation}-\ref{conditional density: positiveness} hold. If the hypothesis class $\mathcal{H}$ is weak $\pi_X$-Glivenko-Cantelli (GC), and the induced set class $\{\{(x,y)\in \mathcal{X}\times \R:L(x)\leq y\leq U(x)\}:L,U\in\mathcal{H},L\leq U\}$ is weak $\pi$-GC in the product space (see Section \ref{sec:ep} for related definitions), then $\widehat{\opt} (t)$ is consistent with respect to \eqref{OP1} in the sense that there exists a sequence $t_n\to 0$ such that, with probability tending to one,
$ \mathbb{P}_{\pi}(Y\in [\hat{L}_{t_n}^*(X),\hat{U}_{t_n}^*(X)])\geq 1-\alpha$,
and that
 $\mathbb{E}_{\pi_X}[\hat{U}_{t_n}^*(X)-\hat{L}_{t_n}^*(X)]\to\mathcal{R}^*(\mathcal{H})$ in probability.
\end{theorem}
Theorem \ref{basic consistency} states that, if the weak uniform law of large numbers holds for both the hypothesis class and the induced set of ``between''-graphs, the interval learned from $\widehat{\opt}(t)$ has the desired coverage rate and a vanishing optimality gap in width by properly selecting the margin $t$. This general result requires a simultaneous control of the function class $\mathcal{H}$ and its induced set class. Our next result shows that, under a mild boundedness condition on the conditional density function, GC property of the function class $\mathcal{H}$ can be propagated to the induced set class, therefore it suffices to control the class $\mathcal{H}$ only.
\begin{theorem}[Consistency for strong $\pi_X$-GC hypothesis]\label{consistency:strong GC H}
Assume the conditional density function from Assumption \ref{conditional density: positiveness} is bounded, i.e., $\sup_{x,y}p(y\vert x)<\infty$, and that $\mathbb{E}_{\pi}[\lvert Y\rvert]<\infty$, then strong $\pi_X$-GC of the hypothesis class $\mathcal{H}$ implies strong $\pi$-GC of the induced set class defined in Theorem \ref{basic consistency}. Therefore, if Assumptions \ref{class: non-negative translation}-\ref{conditional density: positiveness} are further assumed, the conclusion of Theorem \ref{basic consistency} holds for strong $\pi_X$-GC $\mathcal{H}$.
\end{theorem}

\section{Further Discussion of Theorem \ref{rate:vc-subgraph}}\label{sec:vc-subgraph}
We provide further discussion on Theorem \ref{rate:vc-subgraph} regarding $\mathcal H_+$ versus $\mathcal H$ in the bound. Note that, since $\mathcal{H}\subset \mathcal{H}_+$, the augmented class $\mathcal{H}_+$ being VC-subgraph is a stronger condition than $\mathcal{H}$ being VC-subgraph. Nonetheless, we comment that this is a technical assumption used to accommodate potentially unbounded outcomes or PIs (e.g., Assumption \ref{class: non-negative translation} implies unboundedness of functions in $\mathcal{H}$). When $Y$ is uniformly bounded, say within $[0,1]$, it suffices to consider bounded $L,U$ only in PI construction. In that case, $\mathcal{H}$ being VC-subgraph already suffices to ensure similar finite-sample bounds.

\section{Linear Hypothesis Class}\label{sec:linear}
We present a joint coverage-width guarantee for PIs constructed from a linear hypothesis class. Consider the linear hypothesis $\mathcal{H}=\{a^Tx+b:\Vert a \Vert_1\leq B,b\in\R\}$ for some $B>0$. The $l_1$-norm of the coefficient is set bounded to control model complexity.


We demonstrate how Theorem \ref{rate:vc-subgraph} is applied to this class. First note that the augmented class $\mathcal{H}_+=\mathcal{H}$ is the same class of the linear function class $\{a^Tx+b:\Vert a \Vert_1\leq B,b\in\R\}$, which is VC-subgraph of dimension at most $d+2$ (see, e.g., Theorem 2.6.7 in \citeAPX{van1996weak}), and hence $\vc{\mathcal{H}}= \vc{\mathcal{H}_+}\leq d+2$. Therefore
\begin{equation}
\phi_1(n,\epsilon,\mathcal{H})\leq 2\exp\Big( -\frac{n\epsilon^2}{C\Vert H \Vert_{\psi_2}^2\vc{\mathcal{H}_+}} \Big)\label{objective deviation:vc}
\end{equation}
and
$$\phi_2(n,t,\mathcal{H})\leq 
    \begin{cases}
    4^{n+1}\exp(-t^2n)\;&\text{if}\;n< \frac{\vc{\mathcal{H}}}{2}\\
    4\Big(\frac{2en}{\vc{\mathcal{H}}}\Big)^{\vc{\mathcal{H}}}\exp(-t^2n)\;&\text{if}\;n\geq \frac{\vc{\mathcal{H}}}{2}
    \end{cases}$$in Theorem \ref{rate:vc-subgraph} hold with both $\vc{\mathcal{H}}$ and $\vc{\mathcal{H}_+}$ replaced by $d+2$. To derive the $\Vert H \Vert_{\psi_2}$ in \eqref{objective deviation:vc}, we calculate $H(x)= \sup_{\Vert a\Vert_1\leq B}\lvert a^T(x-\mathbb{E}[X]) \rvert =B\Vert x-\mathbb{E}[X]\Vert_{\infty}$, leading to $\Vert H\Vert_{\psi_2}=B\Vert\Vert X-\mathbb{E}[X]\Vert_{\infty}\Vert_{\psi_2}$.


The above analysis is a direct application of general VC theory to the linear function class, and the bound $\phi_1$ exhibits a polynomial dependence on the dimension $d$. A finer analysis that exploits the linear structure can potentially deliver bounds with much lighter dimension dependence, e.g., \citeAPX{zhang2002covering} provides specialized covering number bounds for linear function classes with norm-constrained coefficients which ultimately translate into tighter deviation bounds. The theory in \citeAPX{zhang2002covering} however requires that the variable $X$ has a bounded support, whereas here we are able to show a logarithmic dependence for unbounded $X$ through a more elementary treatment. Specifically, the maximal deviation can be expressed as $\sup_{h \in \mathcal{H}}\lvert \mathbb{E}_{\hat\pi_X}[h(X)] - \mathbb{E}_{\pi_X}[h(X)] \rvert\leq B\Vert \frac{1}{n}\sum_{i=1}^nX_i-\mathbb{E}_{\pi_X}[X]\Vert_{\infty}$, and applying the sub-Gaussian concentration inequality to the supremum norm gives rise to the following:
\begin{theorem}[Linear hypothesis class]\label{deviation bounds for linear hypothesis}
For the linear class $\mathcal{H}$ defined as above we have
\begin{equation*}
    \phi_1(n,\epsilon,\mathcal{H})\leq 2\exp\Big( -\frac{\epsilon^2n}{CB^2\Vert \Vert X -\mathbb{E}_{\pi_X}[X]\Vert_{\infty} \Vert_{\psi_2}^2\log d} \Big)
\end{equation*}
where $C$ is a universal constant.
\end{theorem}

\section{Alternate Analysis of Neural Networks using VC Dimension}\label{sec:NN}
In Section \ref{sec:learning} we have established the joint coverage-width guarantee for PIs constructed from neural networks using our Lipschitz class results (Theorem \ref{rate:Lipschitz}). Here we provide an alternate approach to analyze neural networks via our VC class results (Theorem \ref{rate:vc-subgraph}). We consider the VC dimension of a real-valued neural network as a VC-subgraph class, which is also known as Pollard's pseudo-dimension \citeAPX{pollard2012convergence}. Bounds for pseudo-dimension are relatively well-established for neural networks with sigmoid or piece-wise polynomial activation. For example, when all activation functions are sigmoid, the class is VC-subgraph with $\vc{\mathcal{H}}= O(W^2U^2)$ (Theorem 14.2 in \citeAPX{anthony2009neural}). Alternatively, if all the activation functions are piece-wise polynomials with a bounded number of pieces and of bounded degrees, e.g., rectified linear unit (ReLU, see \citeAPX{lecun2015deep,goodfellow2016deep}) or linear activation, then we have $\vc{\mathcal{H}}=O(WU)$ (Theorem 8 in \citeAPX{bartlett2019nearly}) and simultaneously that $\vc{\mathcal{H}}=O(WS^2+WS\log(W))$ (Theorem 8.8 in \citeAPX{anthony2009neural}, Theorem 6 in \citeAPX{bartlett2019nearly}, and Theorem 1 in \citeAPX{bartlett1999almost}). Similar bounds are also available (e.g., Theorem 8.14 in \citeAPX{anthony2009neural}) when the network involves both sigmoid and piece-wise polynomial activation functions. On the other hand, the augmented class $\mathcal{H}_+$ is a subclass of an augmented neural network where the output unit of $\mathcal{H}$ serves as the last hidden layer (with a single neuron) followed by a new output unit with linear activation, i.e., the class $\{ah+b:h\in\mathcal{H},a\in\R,b\in\R\}$, and thus its VC-subgraph property can be propagated to this augmented class.

\section{More Discussions of Finite-Sample Guarantees for the Coverage Calibration Procedure}\label{sec:calibration discussion}
We provide further interpretations on the margin $q_{1-\beta}\hat{\sigma}_j/\sqrt{n_v}$ in Algorithm \ref{calibration:normalized}, and the error terms of \eqref{finite sample error:normalized_simple} in Theorem \ref{feasibility:normalized validator_simple}. The margin $q_{1-\beta}\hat{\sigma}_j/\sqrt{n_v}$ in Algorithm \ref{calibration:normalized} is reasoned from the CLT that $\sqrt{n_v}\big(\hat{\CR}(\PI_1)-\CR(\PI_1),\ldots,\hat{\CR}(\PI_m)-\CR(\PI_m)\big)\overset{d}{\to} N(0,\Sigma)$ where $\Sigma$ is the covariance matrix with $\Sigma_{j_1,j_2}=\mathrm{Cov}_{\pi}(I_{Y\in \PI_{j_1}(X)},I_{Y\in \PI_{j_2}(X)})$. Approximating $\Sigma$ with the sample covariance $\hat\Sigma$ from Step 1 of Algorithm \ref{calibration:normalized} and applying the continuous mapping theorem, we have $\sqrt{n_v}\max_{j}(\hat{\CR}(\PI_j)-\CR(\PI_j))/\hat{\sigma}_j\overset{d}{\to} \max_j Z_j/\hat{\sigma}_j$ where $(Z_j)_{j=1,\ldots,m}$ follows $N(0,\Sigma)$. Therefore, using the $1-\beta$ quantile of $\max_j Z_j/\hat{\sigma}_j$ in the margin leads to a uniform control of the statistical errors in $\hat{\CR}(\PI_j)$'s with probability approximately $1-\beta$. Theorem \ref{feasibility:normalized validator_simple} states this approximation concretely.

The polynomial term in \eqref{finite sample error:normalized_simple} corresponds to the error of the joint central limit convergence, and the exponential error term quantifies the probability of the undesirable event that none of the candidate PIs satisfies the penalized constraint in Step 3. In practice, one usually targets at relatively high coverage rates, say at least $50\%$, and would train the candidate PIs in such a way that the true coverage rates of some of the PIs sufficiently exceed the highest target level, e.g., by heavily penalizing the coverage error. In that case, $\alpha_{\min} = \tilde{\alpha}$, and $\underline{\alpha}<\alpha_{\min}$ with a sufficient gap, therefore using a sample size $n_v$ of order $\Omega(\log^7(m)/\alpha_{\min})$ is enough for ensuring $\epsilon>0$ and the probability \eqref{finite sample error:normalized_simple} close to $1-\beta$ so that correct coverage rates are guaranteed with high confidence. This logarithmic dependence on $m$ allows us to advantageously use lots of candidate models in the calibration step.


Another notable feature of the finite-sample error is its independence of $K$, the number of target rates. This independence arises from the choice of the margin based on the Gaussian supremum that leads to a uniform control of the statistical errors in the empirical coverage rates. This provides the flexibility of constructing PIs for arbitrarily many target levels simultaneously.

Besides coverage attainment, our calibration procedure also possesses guaranteed performance regarding the other side of the feasibility-optimality tradeoff, provided that only the calibration data are used to assess the width in Step 3 of Algorithm \ref{calibration:normalized}. This is detailed in the following result:
\begin{theorem}\label{optimality guarantee: normalized}
Assume all the candidate PIs in Algorithm \ref{calibration:normalized} are selected from a hypothesis class $\mathcal H$ whose envelope $H(x):=\sup_{h\in \mathcal H}\lvert h(x) - \mathbb E_{\pi_X}[h(X)] \rvert$ has a finite sub-Gaussian norm $\Vert H \Vert_{\psi_2}<\infty$. If in Step 3 of Algorithm \ref{calibration:normalized} each $j^*_{1-\alpha_k}$ is selected according to
\begin{equation*}
j^*_{1-\alpha_k} = \argmin{1\leq j\leq m}\Big\{\frac{1}{n_v}\sum_{i=1}^{n_v}\lvert \PI_j(X'_i)\rvert:\hat{\CR}(\PI_j)\geq 1-\alpha_k + \frac{ q_{1-\beta}\hat{\sigma}_j}{\sqrt{n_v}}\Big\}  
\end{equation*}
and all other steps are kept the same, then for every $\epsilon>0$ we have
\begin{eqnarray*}
&&\mathbb P_{\mathcal D_v}\Big(\mathbb E_{\pi_X}[U_{j^*_{1-\alpha_k}}(X) - L_{j^*_{1-\alpha_k}}(X)]\leq \min_{j:\mathrm{CR(\mathrm{PI}_j)\geq 1-\alpha_k+\epsilon}}\mathbb E_{\pi_X}[U_{j}(X) - L_{j}(X)]+2C\epsilon\Vert H\Vert_{\psi_2}\text{ for all }k=1,\ldots,K\Big)\\
&\geq&1-8m\exp \Big( -\frac{1}{4}\max\big\{\epsilon - C\sqrt{\frac{\log (m/\beta)}{n_v}}, 0\big\}^2n_v\Big)
\end{eqnarray*}
for some universal constant $C$.
\end{theorem}

\section{Alternate Calibration Scheme}\label{sec:alternate}
We present an alternate coverage calibration scheme than Algorithm \ref{calibration:normalized} that switches the Gaussian vector used in the margin from ``normalized'' to ``unnormalized''. To explain, the margin $q_{1-\beta}\hat\sigma_j/\sqrt{n_v}$ used in Algorithm \ref{calibration:normalized} is set proportional to the standard deviation of the empirical coverage rate for each individual PI. An alternative is to set $q'_{1-\beta}$, the $1-\beta$ quantile of $\max\{Z_j:1\leq j\leq m\}$, as a uniform margin for all the PIs, which also captures the uniform error in coverage rates due to the convergence $\sqrt{n_v}\max_{j}(\hat{\CR}(\PI_j)-\CR(\PI_j))\overset{d}{\to} \max_j Z_j$. This alternative scheme is depicted in Algorithm \ref{calibration:unnormalized}.
\begin{algorithm}
\caption{Unnormalized PI Calibration}
\label{calibration:unnormalized}
\DontPrintSemicolon
\SetKwInOut{Input}{Input}\SetKwInOut{Output}{Output}

\Input{~Same as in Algorithm \ref{calibration:normalized}.}

\BlankLine

\textbf{Procedure:}\;

\textbf{1.}~Same as in Algorithm \ref{calibration:normalized}.\;

\textbf{2.}~Compute $q'_{1-\beta}$, the $(1-\beta)$-quantile of $\max\{Z_j:1\leq j\leq m\}$ where $(Z_1,\ldots,Z_m)$ is a multivariate Gaussian with mean zero and covariance $\hat\Sigma$.\;

\textbf{3.}~For each coverage rate $k=1,\ldots,K$ compute
\begin{equation*}
    j^*_{1-\alpha_k} = \argmin{1\leq j\leq m}\Big\{\frac{1}{n+n_v}\Big(\sum_{i=1}^n\lvert \PI_j(X_i)\rvert+\sum_{i=1}^{n_v}\lvert\PI_j(X'_i)\vert\Big):\hat{\CR}(\PI_j)\geq 1-\alpha_k + \frac{ q'_{1-\beta}}{\sqrt{n_v}}\Big\}
\end{equation*}
where $\{X_i\}_{i=1}^n$ is the training data set.\;
\BlankLine
\Output{$\PI_{j^*_{1-\alpha_k}}$ for $k=1,\ldots,K$.}
\end{algorithm}


Algorithm \ref{calibration:unnormalized} enjoys a similar finite-sample performance guarantee:
\begin{theorem}\label{feasibility:unnormalized validator}
Under the same setting of Theorem \ref{feasibility:normalized validator_simple}, the finite sample error \eqref{finite sample error:normalized_simple} continues to hold for the PIs output by Algorithm \ref{calibration:unnormalized}, but with $\epsilon=\max\{\alpha_{\min}-\underline{\alpha}-C_1\sqrt{\log(m/\beta)/n_v},0\}$.
\end{theorem}

We compare Algorithms \ref{calibration:normalized} and \ref{calibration:unnormalized} in terms of statistical efficiency. Like for Algorithm \ref{calibration:normalized}, if the target coverage rates are above $50\%$ and the maximal achieved coverage rate of the candidate PIs sufficiently exceeds the highest target level, then $\alpha_{\min} = \tilde{\alpha}$, and $\underline{\alpha}<\alpha_{\min}$ with a sufficient gap. The new expression for $\epsilon$ in Theorem \ref{feasibility:unnormalized validator} now implies a sample size $n_v$ of order $\Omega\big(\frac{\log m}{\alpha_{\min}^2}+\frac{\log^7m}{\alpha_{\min}}\big)$ for Algorithm \ref{calibration:unnormalized} to guarantee correct coverage rates with high confidence. Note that the dependence on $\alpha_{\min}$ grows from a linear one in Algorithm \ref{calibration:normalized} to quadratic, suggesting that Algorithm \ref{calibration:normalized} is more powerful in the case of high target coverage levels. However, when the target coverage levels are moderate (around $50\%$), Algorithm \ref{calibration:unnormalized} is more efficient instead, due to a smaller margin than the one in Algorithm \ref{calibration:normalized} for PIs with moderate coverage rates. To explain, denote by $\PI_{\bar{j}}$ the PI with the maximal sample standard deviation (coverage closest to $50\%$), i.e., $\hat\sigma_{\bar{j}}=\max_{1\leq j\leq m}\hat\sigma_j$, then the margin used for $\PI_{\bar{j}}$ in Algorithm \ref{calibration:normalized} satisfies
\begin{eqnarray*}
q_{1-\beta}\hat\sigma_{\bar{j}}&=&\hat\sigma_{\bar{j}}\cdot 1-\beta\text{ quantile of }\max_{1\leq j\leq m}Z_j/\hat\sigma_j\\
&=&1-\beta\text{ quantile of }\max_{1\leq j\leq m}\hat\sigma_{\bar{j}}Z_j/\hat\sigma_j\\
&>&1-\beta\text{ quantile of }\max_{1\leq j\leq m}Z_j\text{\ \ \ \  if not all $\hat\sigma_j$'s are equal}\\
&=&q'_{1-\beta}
\end{eqnarray*}
which is the margin used by Algorithm \ref{calibration:unnormalized}.


\section{A Lagrangian Formulation  for Training Neural-Network-Based Prediction Intervals}\label{sec:calibrate NN}
We discuss a Lagrangian formulation of \eqref{OP2} for training neural networks to construct PIs. This formulation has the dual multiplier set as the tunable parameter to balance the tradeoff between the objective and the constraint in \eqref{OP2}.
Specifically, we use
$$L(\delta;\lambda)=\mathbb{E}_{\hat\pi_X}[U(X)-L(X)]+ \lambda (1-\alpha+t- \mathbb{P}_{\hat\pi}(Y\in[L(X),U(X)]))$$
or
$$L(\delta;\lambda)=\frac{1}{n}\sum_{i=1}^{n}(U(x_i)-L(x_i))+  \frac{\lambda}{n}\sum_{i=1}^{n}  I_{y_i\notin[L(x_i),U(x_i)]}+ \text{constant}$$
where $\lambda$ is the multiplier. In practice, we use a ``soft" version of the Lagrangian function for gradient descent. The ``soft" loss we adopt is introduced in Section \ref{sec:experiments}. 

We can build multiple PI models by using different parameters $\lambda>0$. Then, these models are calibrated using Algorithm \ref{calibration:normalized} or \ref{calibration:unnormalized} so that the coverage constraint in \eqref{OP1} is satisfied. Intuitively, if $\lambda$ is large, $\sum_{i=1}^{m}(U(X_i)-L(X_i))$ contributes less to the overall loss function, and hence the resulting interval tends to be wide but have a high coverage rate. On the contrary, a small $\lambda$ entails a short interval with a low coverage rate. Hence, a neural network is a reasonable approach to solve  \eqref{OP2} since a neural network with the above loss can directly address the tradeoff between the interval width and the coverage rate.

\section{Empirical Process Background}\label{sec:ep}
For a class $\mathcal{G}$ of measurable functions from $\mathcal{X}$ to $\R$ such that $\mathbb{E}_{\pi_X}[\lvert g(X)\rvert]<\infty$ for every $g\in\mathcal{G}$, we say it is weak (resp. strong) $\pi_X$-Glivenko–Cantelli (GC) if $\sup_{g\in\mathcal{G}}\big\lvert \frac{1}{n}\sum_{i=1}^ng(X_i)-\mathbb{E}_{\pi_X}[g(X)] \big\rvert \to 0$ in probability (resp. almost surely) as $n\to\infty$. For a class $\mathcal{S}$ of measurable subsets of $\mathcal{X}$, i.e., $S\subset\mathcal{X}$ for every $S\in\mathcal{S}$, we say it's weak (resp. strong) $\pi_X$-GC if the corresponding indicator class $\{I_{\cdot \in S}:S\in \mathcal{S}\}$ is weak (resp. strong) $\pi_X$-GC. When no ambiguity arises, we sometimes suppress the underlying distribution $\pi_X$.

A collection of $k$ points $\{x_1,\ldots,x_k\}\subset \mathcal{X}$ is said to be shattered by a class $\mathcal{S}$ of subsets of $\mathcal{X}$ if $\card{\{\{x_1,\ldots,x_k\}\cap S:S\subset\mathcal{S}\}}=2^k$, where $\card{\cdot}$ denotes the cardinality of a set. The VC dimension of the class $\mathcal{S}$ is defined as $\vc{\mathcal{S}}:=\max\{k:\exists\; \{x_1,\ldots,x_k\}\subset\mathcal{X}\text{ shattered by }\mathcal{S}\}$. It is called a VC class if $\vc{\mathcal{S}}<\infty$. A class $\mathcal{G}$ of functions from $\mathcal{X}$ to $\R$ is called VC-subgraph with VC dimension $d$ if the set of subgraphs $\mathcal{S}_{\mathcal{G}}:=\{\{(x,z)\in\mathcal{X}\times \R:z<g(x) \}:g\in\mathcal{G}\}$ is a VC class on the product space $\mathcal{X}\times \R$ with $\vc{\mathcal{S}_{\mathcal{G}}}=d$. Without ambiguity we use the same notation $\vc{\mathcal{G}}$ to denote the VC dimension of a VC-subgraph class $\mathcal{G}$. Note that VC or VC-subgraph classes are combinatorial in nature and distribution-independent, whereas GC classes here are with respect to a specific distribution $\pi_X$.

Given a class $\mathcal{G}$ of functions from $\mathcal{X}$ to $\R$, and a probability measure $Q$ on $\mathcal{X}$, the $\epsilon$-covering number $N(\epsilon,\mathcal{G},L_2(Q))$ is the minimum number of $L_2(Q)$-balls of size $\epsilon$ needed to cover the whole class $\mathcal{G}$. A pair of functions $l,u:\mathcal{X}\to\R$ is called a bracket of size $\epsilon$ with respect to $L_2(Q)$ if $l\leq u$ almost surely and $(\mathbb{E}_Q[(u-l)^2])^{1/2}\leq \epsilon$, and every function $g$ such that $l\leq g\leq u$ is said to be contained in the bracket. The $\epsilon$-bracketing number $N_{[]}(\epsilon,\mathcal{G},L_2(Q))$ is the minimum number of brackets of size $\epsilon$ needed to cover the whole class $\mathcal{G}$.

The above terminologies extend to the product space $\mathcal{X}\times \mathcal{Y}$ with the joint distribution $\pi$ in a straightforward manner, i.e., by replacing each occurrence of $\mathcal{X}$ and $\pi_X$ with $\mathcal{X}\times \mathcal{Y}$ and $\pi$ respectively.

\section{Experimental Details and Additional Experiments}\label{sec:numerics additional}

We illustrate additional experiments and experimental details, which are divided into two subsections. Appendix \ref{sec:low} presents and visualizes different PI construction approaches on one-dimensional examples. Appendix \ref{sec:pareto} illustrates the Pareto curves for the results in Section \ref{sec:experiments}. Appendix \ref{sec:high} provides details of our experimental implementations.

\subsection{Illustration in One-Dimensional Examples}\label{sec:low}

We conduct experiments and visualize PI construction on three univariate examples. Table \ref{single-level result} shows the generative distributions for the three univariate synthetic datasets. Implementation details can be found in Section \ref{sec:high}.

\begin{table}[h] 
  \centering
  \begin{tabular}{ccccc}
    \hline
    Index & Tested Function & Function Space & Variable Space & Noise Type($\epsilon$)\\
    \hline
    \vspace{0.5mm}
      1 & $f(x) = \sin(x) + x\epsilon$  & $\mathbb{R}\rightarrow \mathbb{R}$ & $x\sim Unif[-3,3]$ & $\epsilon\sim Unif[-2,2]$ \\
    \vspace{0.5mm}
      2 & $f(x) = \frac{x^{2}}{2} + \cos{x} + x\epsilon$  & $\mathbb{R}\rightarrow \mathbb{R}$ & $x\sim Unif[-3,3]$ & $\epsilon \sim Unif[-2,2]$ \\
      
      3 & $f(x) = x^2 + \frac{\sin(x)}{8} + x\epsilon$  & $\mathbb{R}\rightarrow \mathbb{R}$ & $x\sim Unif[-3,3]$ & $\epsilon \sim Unif[-1,2]$ \\
    \hline
  \end{tabular}
  \caption{Tested Functions}
 \label{single-level result}  
\end{table}


Figure \ref{1D_examples} illustrates the PIs. We test PIs on a testing dataset and evaluate their performances using the metrics of the coverage rate ($CR$) and the interval width ($IW$). All baselines are targeted to attain the prediction level $95\%$. The titles of all plots are named as Synthetic 1d-$\{$index of synthetic dataset$\}$. Each row shows the performances of the same approach but on different datasets, and each column shows the performances of different approaches on the same dataset. The upper and lower bounds of the PIs that \textbf{attain} the 95\% target level are shown in solid and green lines, otherwise in dashed and red lines. The covered areas are shaded with their corresponding colors. Data points are the black dots in the plot. The corresponding $CR$ and $IW$ are shown in the label at the upper center of each plot.

We observe that on all the datasets, our approaches NNGN and NNGU outperform other methods in terms of always attaining the $95\%$ target prediction level and having narrow intervals at the same time. QRF, CV+, SVMQR and NNVA do not attain the $95\%$ prediction levels, which is consistent with the observation that no finite-sample coverage guarantees are known for these approaches. SCQR and SCL attain the $95\%$ prediction level but their intervals appear much wider than NNGN and NNGU. Also, since NNGU is designed to be more conservative than NNGN, the intervals calibrated by NNGN are generally shorter than the ones calibrated by NNGU. This observation is consistent with Table \ref{multiple-level} in Section \ref{sec:experiments}.

\begin{figure}[H]
 \centering
 \includegraphics[height = 18cm, width = 17cm]{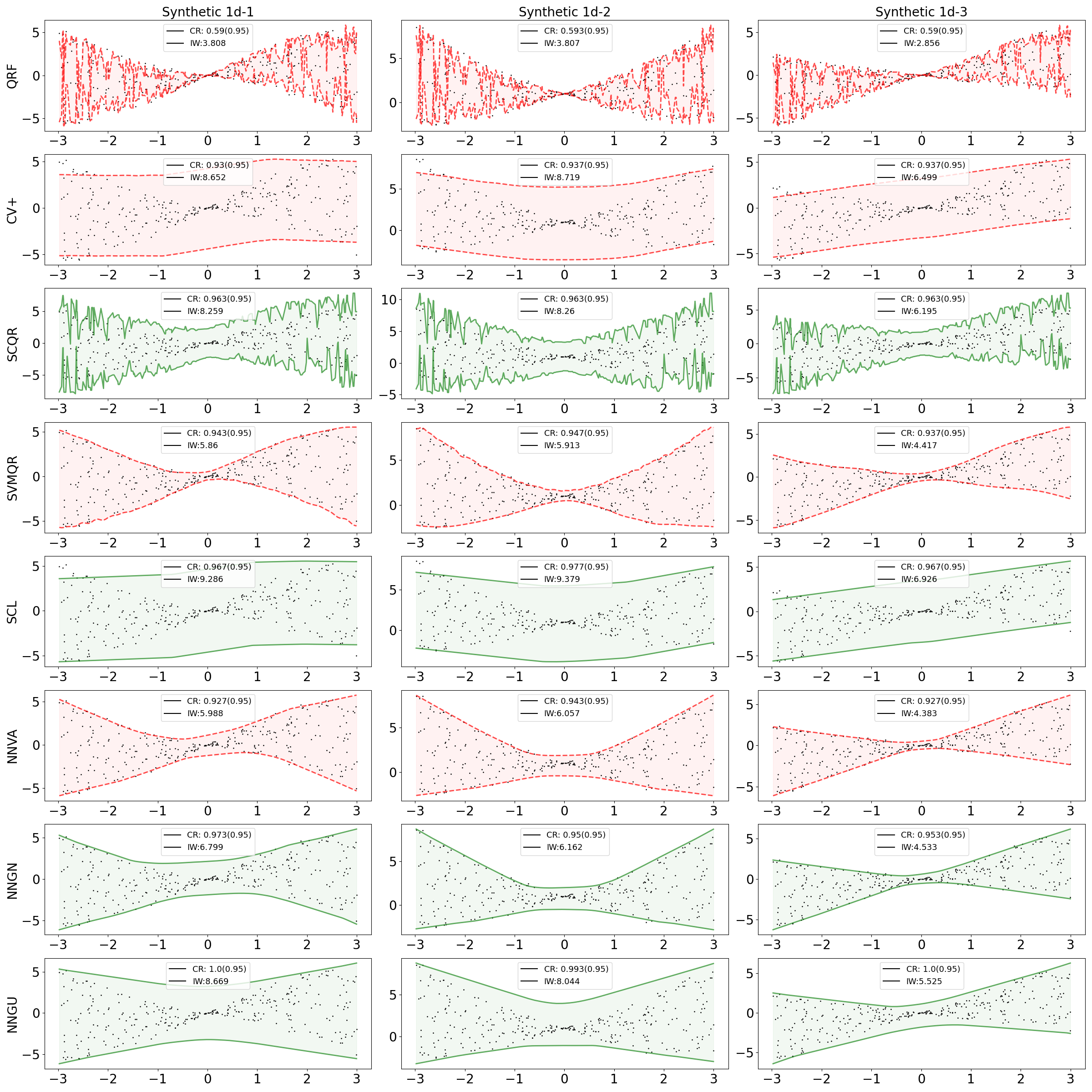}
 \caption{Comparison of single PI constructions. PIs that attain $95\%$ target level are shown in solid and green lines, otherwise in dashed and red lines.} 
 \label{1D_examples}
\end{figure}

\subsection{Pareto Curves}\label{sec:pareto}
We illustrate the Pareto curves for the simultaneous PIs results with 19 target prediction levels in Section \ref{sec:experiments}, in terms of the coverage rate ($CR$) (X-axis) and the interval width ($IW$) (Y-axis), to offer a more intuitive comparison. The titles of all plots are named as the datasets in Section \ref{sec:experiments}. The arrangement of the plots are the same as the ones in Section \ref{sec:low}. Specifically, each subplot contains two curves, one constructed with (input coverage, width) and another constructed with (achieved coverage, width). The curve representing the input coverage and width is shown in dashed line while the one representing the achieved coverage and width is shown in solid line along with a $90\%$ confidence interval as the shaded area. Each point in the line denotes the average obtained from 
$N = 50$ repetitions of trials.

Ideally, a method performs well if in its Pareto curves, 1) the dashed line is on the left of the solid line, and 2) at the same time there is no level intersection between the solid line and the shaded area, meaning that within the same input coverage level, there is a sufficiently high probability in achieving the coverage level. Other than that, a smaller average $IW$ at each input coverage level is better since it refers to a less conservative predictive ability. From Figures \ref{pareto_1} and \ref{pareto_2}, we notice that CV+, NNGN, and NNGU have a lot more cases where there is no intersection between the dashed line and shaded area while SCL, QRF, SVMQR, SCQR and NNVA do not. Also the former three methods tend to generate much shorter PIs than the others. Specifically, there are $5$ datasets where there is no intersection between the dashed line and the shaded area for CV+ and NNGN, and $4$ for NNGU. However, NNGN and NNGU can achieve a much shorter $MIW$ (average of $IW$ over all input coverage levels) than the rest of the methods.

One might notice that the two Pareto curves are a bit farther away from each other for NNVA, NNGN, and NNGU when the input coverage is small. This issue can be solved by choosing the calibration parameter more appropriately by, for example, training a more continuous spectrum of candidate models.

\begin{figure}[H]
 \centering
 \includegraphics[height = 18cm, width = 17cm]{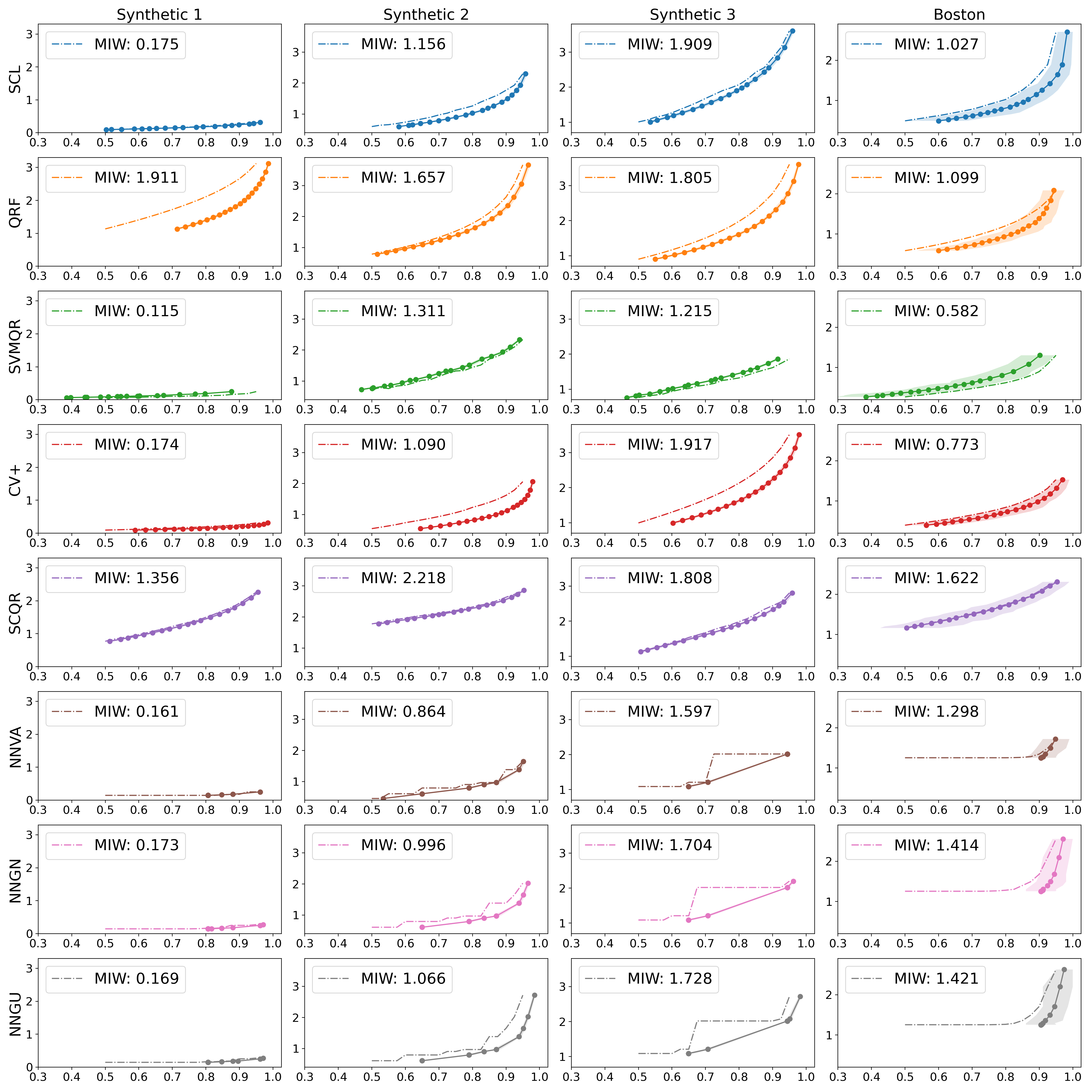}
 \caption{Comparison of simultaneous PIs constructions in synthetic datasets 1 to 3, and Boston dataset. Dashed lines are constructed with (input coverage, width), and solid lines are constructed with (achieved coverage, width). Each point in the line denotes the average obtained from $N = 50$ repetitions of trials. Shaded area is the $90\%$ confidence interval.} 
 \label{pareto_1}
\end{figure}

\begin{figure}[H]
 \centering
 \includegraphics[height = 18cm, width = 17cm]{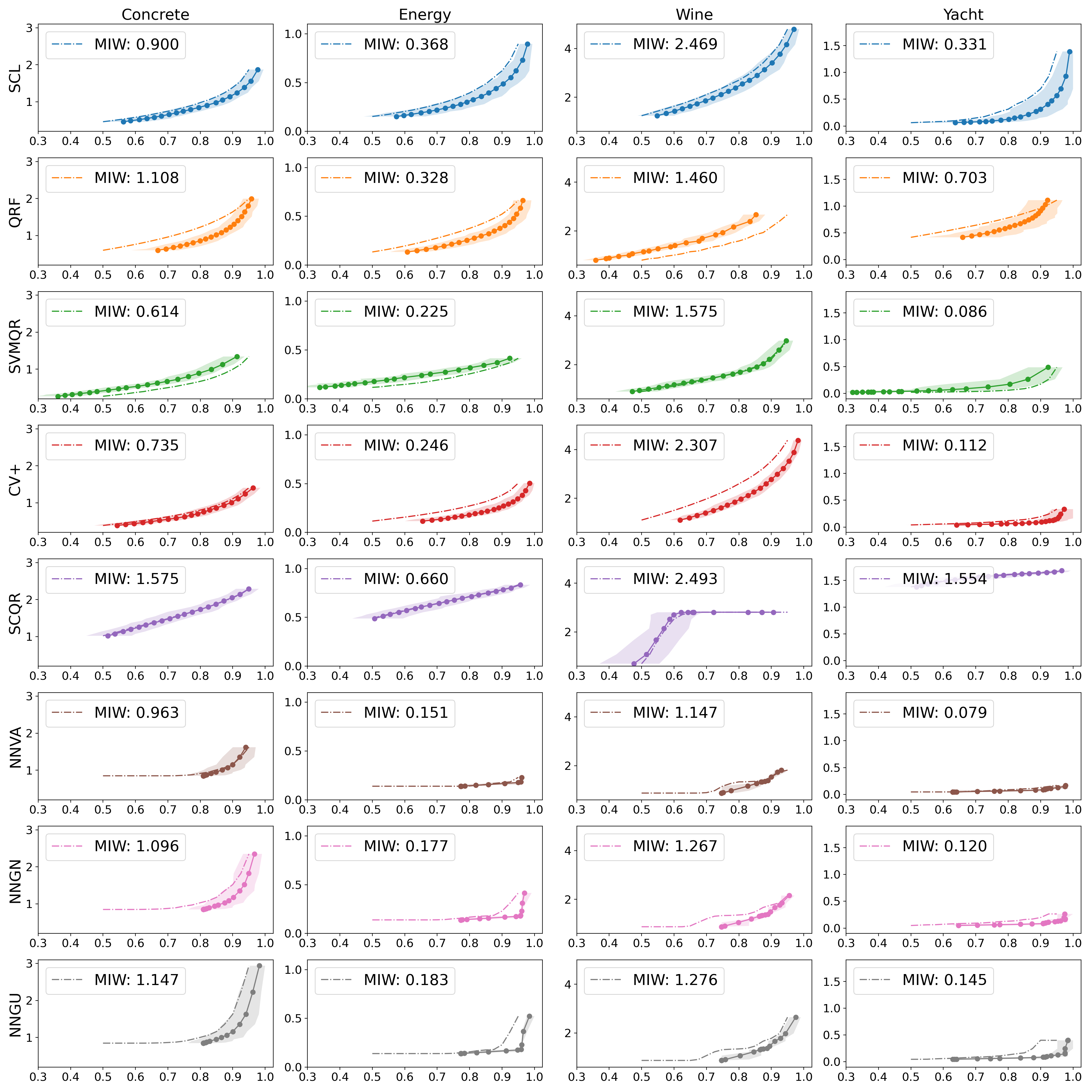}
 \caption{Comparison of simultaneous PIs constructions in Concrete, Energy, Wine and Yacht dataset. Dashed lines are constructed with (input coverage, width), and solid lines are constructed with (achieved coverage, width). Each point in the line denotes the average obtained from $N = 50$ repetitions of trials. Shaded area is the $90\%$ confidence interval.} 
 \label{pareto_2}
\end{figure}

\subsection{Implementation Details}\label{sec:high}

We elaborate more details about our experimental implementations in Sections \ref{sec:experiments} and \ref{sec:low}.  

\textbf{Datasets.} Three synthetic datasets and five real-world benchmark datasets have been shown in Section \ref{sec:experiments}. The real-world datasets are the open-access datasets “Boston”, “Concrete”, “Energy”, “Wine” and “Yacht” that have been widely used in previous studies \citep{hernandez2015probabilistic,gal2016dropout,lakshminarayanan2017simple} for regression tasks. Table \ref{datasets} shows their details.

\begin{table}[h] 
  \centering
  \begin{tabular}{cccc}
    \hline
    Dataset & N & d & Open-access Link\\
    \hline
    Boston: Boston Housing & 506 & 13 & kaggle.com/c/boston-housing\\
    Concrete: Concrete Strength & 1030 & 8 & kaggle.com/aakashphadtare/concrete-data\\
    Energy: Energy Efficiency & 768  & 8 & kaggle.com/elikplim/eergy-efficiency-dataset\\
    Wine: Red Wine Quality & 1599 & 11 & kaggle.com/uciml/red-wine-quality-cortez-et-al-2009  \\
    Yacht: Yacht Hydrodynamics & 308 & 6 & archive.ics.uci.edu/ml/datasets/yacht+hydrodynamics\\
    \hline
  \end{tabular}
  \caption{Full names and details of benchmarking regression datasets. $N$ is the number of samples in the dataset and $d$ is the dimension of the feature vector.}
 \label{datasets}  
\end{table}

The data are split into training and testing sets as follows. For the methods where validation data are needed for calibration (NNVA, NNGN, NNGU), we use a proportion of training data as the validation data. In the single PI case, all of the real-world datasets have 80\%/20\% training/testing split. For 
NN methods, ``Boston'' has $24\%$ of data for validation; ``Concrete'' has $16\%$ of data for validation; ``Energy'' has $10\%$ of data for validation; ``Wine'' has $12\%$ of data for validation; ``Yacht'' has $15\%$ of data for validation. The three multivariate synthetic datasets in Section \ref{sec:experiments} have 1600/3000 training/testing split. For NN methods, 350 data points are used for validation. In the simultaneous PIs case, all the real-world datasets have 80\%/20\% training/testing split. For NN methods, ``Boston'' has $24\%$ of data for validation; ``Concrete'' has $16\%$ of data for validation; ``Energy'' has $24\%$ of data for validation; ``Wine'' has $24\%$ of data for validation; ``Yacht'' has $24\%$ of data for validation. The three multivariate synthetic datasets in Section \ref{sec:experiments} have 1600/3000 training/testing split. For NN methods, 350 data points are used for validation. In addition, for the experiments in Section \ref{sec:low}, the three univariate synthetic datasets have 1200/300 training/testing split. For NN methods, $60$ data points are used for validation.

\textbf{Implementations.}
We provide details of our algorithms and baseline approaches in Section \ref{sec:experiments}. They appear in the same order as in Tables \ref{single-level} and \ref{multiple-level}. 

(1) QRF: quantile regression forests, as proposed in \citet{meinshausen2006quantile}. Our code is based on \textit{RandomForestQuantileRegressor} from the package \textit{scikit-garden} in Python.

(2) CV+: CV+ prediction interval, as proposed in Section 3 in \citet{barber2019predictive}. In addition, the base regression algorithm is a neural network using mean square loss.

(3) SCQR: split conformalized quantile regression, as proposed in Algorithm 1 in \citet{romano2019conformalized}. The base quantile regression algorithm is exactly the QRF in (1).

(4) SVMQR: quantile regression via support vector machine, the code of which is available in \citet{steinwart2017liquidsvm}.

(5) SCL: split conformal learning with correction. The original split conformal learning is described in Algorithm 2 in \citet{lei2018distribution}. The base regression algorithm is a neural network using mean square loss. Moreover, we apply a ``correction" method to stipulate the coverage constraint with high confidence, which has been proposed in Equation 7 in Proposition 2b in \cite{vovk2012conditional} to enhance the split/inductive conformal learning. Specifically, we change the prediction level $1-\alpha$ to $1-\alpha'$ by letting  
$$\beta \ge bin_{n_v,\alpha} (\lfloor \alpha'(n_v + 1) - 1\rfloor),$$
where $1-\beta$ is the prefixed confidence level (90\% throughout our experiments), $n_v$ is the size of the calibration set and $bin_{n_v,\alpha}$ is the cumulative binomial
distribution function with $n_v$ trials and probability of success $\alpha$.

(6) NNVA: neural networks using the loss shown in Section \ref{sec:experiments} with a vanilla scheme. In the vanilla scheme, we build multiple PI models by choosing different parameters $\lambda >0$ in the loss, and then select PIs with the smallest interval width among those whose empirical coverage rates on the validation dataset is larger than the target prediction levels, i.e., without the Gaussian margin in Algorithm \ref{calibration:normalized}.

(7) Ours-NNGN: neural networks using the loss shown in Section \ref{sec:experiments} with the normalized Gaussian PI calibration in Algorithm \ref{calibration:normalized}. We build multiple PI models by choosing different parameters $\lambda >0$ in the loss.

(8) Ours-NNGU: neural networks using the loss shown in Section \ref{sec:experiments} with the unnormalized Gaussian PI calibration in Algorithm \ref{calibration:unnormalized}. We build multiple PI models by choosing different parameters $\lambda >0$ in the loss.

For (6)(7)(8), in order to improve the training of the neural networks, we use the ensemble technique in \citeAPX{pearce2018high}. That is, instead of directly taking the outputs of one network, we train networks several times by using different initializations and then define the final prediction $U$ and $L$ as
\begin{align*}
    &U = \overline{U} + 1.96\sigma_{U}, \text{ where } \overline{U} = \frac{1}{e}\sum_{j = 1}^{e}\hat{U}_{j}, \ \sigma_{U} = \frac{1}{e-1}\sum_{j = 1}^{e}(\hat{U}_{j}-\overline{U})^2 \\
    &{L} = \overline{L} - 1.96\sigma_{L}, \text{ where } \overline{L} = \frac{1}{e}\sum_{j = 1}^{e}\hat{L}_{j},\
    \sigma_{L} = \frac{1}{e-1}\sum_{j = 1}^{e}(\hat{L}_{j}-\overline{L})^2
\end{align*}
where $e$ denotes the ensemble size. 

\textbf{Architectures and Hyper-parameters.} For (2)(5)(6)(7)(8), we train fully connected neural networks with ReLU activations, using the Adam method for stochastic optimization. Note that the base neural networks in (2)(5) have only one output unit (point prediction) while our neural networks in (6)(7)(8) have two output units (lower and upper bounds of PIs). Nevertheless, the neural networks in (2)(5)(6)(7)(8) have the same architecture of hidden layers on each dataset. On ``Boston'' and ``Concrete'', we have 1 hidden layer and each hidden layer has $50$ neurons. On ``Energy'', ``Wine'' and ``Yacht'', we have 2 hidden layers and each hidden layer has $64$ neurons. On three multivariate synthetic datasets in Section \ref{sec:experiments}, we have 2 hidden layers and each hidden layer has $50$ neurons. On the three univariate synthetic datasets in Section \ref{sec:low}, we have 1 hidden layer and each hidden layer has $50$ neurons.
In addition, $N=50$ repetitions of trials are run on each dataset. The confidence level is $1-\beta=90\%$ in all experiments. 

\section{Technical Proofs}
\label{sec:proofs}

\subsection{Proofs for Results in Section \ref{sec:learning}}

\begin{proof}[Proof of Theorem \ref{sensitivity bound}]
We first present a lemma:
\begin{lemma}\label{prob lower bound}
Let $\xi$ be a $[0,1]$-valued random variable such that $\mathbb E[\xi] \leq 1-\beta$ for some $\beta\in[0,1]$, then we have $\mathbb{P}(\xi \leq 1-\beta')\geq \beta-\beta'$ for every $\beta'\in(0,\beta)$.
\end{lemma}
\begin{proof}
Define
\begin{equation*}
    \xi':=\begin{cases}
    0&\text{if }\xi\leq 1-\beta'\\
    1-\beta'&\text{otherwise}
    \end{cases}.
\end{equation*}
Then $\xi\geq \xi'$ almost surely, hence
\begin{equation*}
    1-\beta \geq \mathbb E[\xi] \geq \mathbb E[\xi'] = (1-\beta')(1- \mathbb{P}(\xi\leq 1-\beta'))
\end{equation*}
which gives
\begin{equation*}
    \mathbb{P}(\xi\leq 1-\beta')\geq \frac{\beta-\beta'}{1-\beta'}\geq \beta-\beta'.
\end{equation*}
\end{proof}
Now we turn to the main proof. For any $\epsilon >0$, let $(L_{\epsilon}, U_{\epsilon})\in\mathcal H\times \mathcal H$ be an $\epsilon$-optimal solution of \eqref{OP1}, i.e., $\mathbb{P}_{\pi}(Y\in[L_{\epsilon}(X),U_{\epsilon}(X)])\geq 1-\alpha$ and $\mathbb{E}_{\pi_X}[U_{\epsilon}(X)-L_{\epsilon}(X)]\leq R^*(\mathcal{H}) + \epsilon$. Consider the enlarged interval $L^c_{\epsilon}:=L_{\epsilon} - c$, $U^c_{\epsilon}:=U_{\epsilon} + c$, where $c\geq 0$ is a constant. Let
\begin{equation*}
P(c):=\mathbb{P}_{\pi}(Y\in[L^c_{\epsilon}(X),U^c_{\epsilon}(X)])    
\end{equation*}
be the coverage rate of the new interval. $P(c)$ satisfies $\lim_{c\to +\infty}P(c)=1$ because of the continuity of measure, hence the smallest $c$ such that the coverage rate is above $1-\alpha + t$ is finite, i.e.,
\begin{equation*}
    c^*:=\inf\big\{c\geq 0:P(c) \geq 1-\alpha + t\big\}<\infty.
\end{equation*}

We want to derive an upper bound for $c^*$. If $c^*=0$, every non-negative number is a valid upper bound, therefore we focus on the non-trivial case $c^*> 0$. In this case, we must have $P(c^*)=1-\alpha + t$ due to continuity of the coverage probability function $P(\cdot)$. To explain the continuity of $P(\cdot)$, by conditioning on $X$ we can rewrite
\begin{eqnarray}
P(c)&=&\mathbb{E}_{\pi_X}[\mathbb{P}_{\pi}(Y\in[L^c_{\epsilon}(X),U^c_{\epsilon}(X)]\vert X)]\label{P(c) conditioning}\\
\notag &=&\mathbb{E}_{\pi_X}\Big[\int_{L_{\epsilon}(X)-c}^{U_{\epsilon}(X)+c}p(y\vert X)dy\Big]
\end{eqnarray}
and note that $\int_{L_{\epsilon}(X)-c}^{U_{\epsilon}(X)+c}p(y\vert X)dy$ is continuous in $c$ and bounded by $1$ almost surely, therefore the continuity follows from bounded convergence theorem. For every $c\in [0,c^*]$, we have $P(c)\leq 1-\alpha+t$ by the definition of $c^*$, hence applying Lemma \ref{prob lower bound} to the conditioned form \eqref{P(c) conditioning} of $P(c)$ gives
\begin{equation*}
    \mathbb{P}_{\pi_X}\Big(\mathbb{P}_{\pi}(Y\in[L^c_{\epsilon}(X),U^c_{\epsilon}(X)]\vert X)\leq 1-\frac{\alpha-t}{3}\Big)\geq \frac{2}{3}(\alpha-t).
\end{equation*}


Now we write
\begin{eqnarray*}
1-\alpha+t&=&P(c^*)\\
&=&\mathbb{P}_{\pi}(Y\in[L_{\epsilon}^{c^*}(X),L_{\epsilon}(X))\cup (U_{\epsilon}(X),U_{\epsilon}^{c^*}(X)])+\mathbb{P}_{\pi}(Y\in[L_{\epsilon}(X),U_{\epsilon}(X)])\\
&\geq&\mathbb{E}_{\pi_X}[\mathbb{P}_{\pi}(Y\in[L_{\epsilon}^{c^*}(X),L_{\epsilon}(X))\cup (U_{\epsilon}(X),U_{\epsilon}^{c^*}(X)]\vert X)]+1-\alpha\\
&&\text{\ \ by conditioning on the $X$ and feasibility of $(L_{\epsilon},U_{\epsilon})$}\\
&=&\mathbb{E}_{\pi_X}\Big[\int_0^{c^*} p(L_{\epsilon}^c(X)\vert X) + p(U_{\epsilon}^c(X)\vert X)dc\Big] + 1-\alpha\\
&\geq&\mathbb{E}_{\pi_X}\Big[\int_0^{c^*}\Gamma\big(X,1 - \mathbb{P}_{\pi}(Y\in[L_{\epsilon}^c(X),U_{\epsilon}^c(X)]\vert X)\big)dc\Big]+ 1-\alpha\\
&&\text{\ \ by the definition of $\Gamma(\cdot,\cdot)$}\\
&=&\int_0^{c^*}\mathbb{E}_{\pi_X}\big[\Gamma\big(X,1 - \mathbb{P}_{\pi}(Y\in[L_{\epsilon}^c(X),U_{\epsilon}^c(X)]\vert X)\big)\big]dc+ 1-\alpha\\
&\geq&\int_0^{c^*}\gamma_{\frac{\alpha-t}{3}}\mathbb{P}_{\pi_X}\big(\mathbb{P}_{\pi}(Y\in[L_{\epsilon}^c(X),U_{\epsilon}^c(X)]\vert X)\leq 1-\frac{\alpha-t}{3}\text{ and }\Gamma(X,\frac{\alpha-t}{3})\geq \gamma_{\frac{\alpha-t}{3}}\big)dc+ 1-\alpha\\
&\geq&\frac{\alpha-t}{3}\gamma_{\frac{\alpha-t}{3}}c^*+1-\alpha.
\end{eqnarray*}
Therefore
\begin{equation*}
    c^*\leq \frac{3t}{(\alpha-t)\gamma_{\frac{\alpha-t}{3}}}.
\end{equation*}
Note that $(L_{\epsilon}^{c^*},U_{\epsilon}^{c^*})$ is feasible for \eqref{OP3}, therefore by optimality we have
\begin{eqnarray*}
    R^*_t(\mathcal{H})&\leq &\mathbb{E}_{\pi_X}[U_{\epsilon}^{c^*}(X)-L_{\epsilon}^{c^*}(X)]\\
    &\leq& R^*(\mathcal{H}) + \epsilon + 2c^*\\
    &\leq& R^*(\mathcal{H}) + \epsilon +  \frac{6t}{(\alpha-t)\gamma_{\frac{\alpha-t}{3}}}.
\end{eqnarray*}
Since $\epsilon$ is arbitrary, sending $\epsilon$ to $0$ completes the proof.
\end{proof}

\begin{proof}[Proof of Theorem \ref{convergence rate}]
Let $\hat{\mathcal{H}}^2_t$ and $\mathcal{H}^2_t$ be the feasible set of \eqref{OP2} and \eqref{OP3} respectively. When the events
\begin{equation*}
    W_{\epsilon}:=\big\{\sup_{h\in \mathcal{H}}\lvert \mathbb{E}_{\hat\pi_X}[h(X)]-\mathbb{E}_{\pi_X}[h(X)] \rvert\leq \epsilon\big\}
\end{equation*}
and
\begin{equation*}
    C_t:=\big\{\sup_{L,U\in \mathcal{H}\;\text{and}\;L\leq U}\lvert \mathbb{P}_{\hat\pi_X}(Y\in [L(X),U(X)])-\mathbb{P}_{\pi_X}(Y\in [L(X),U(X)]) \rvert\leq t\big\}
\end{equation*}
occur, it holds that $\mathcal{H}^2_{2t}\subset \hat{\mathcal{H}}^2_t\subset \mathcal{H}^2_0$, therefore $(\hat L_t^*,\hat U_t^*)\in\hat{\mathcal{H}}^2_t\subset \mathcal{H}^2_0$ is feasible for \eqref{OP1}. We also have
\begin{eqnarray}
&&\notag\mathbb{E}_{\pi_X}[\hat U_t^*(X)-\hat L_t^*(X)]\\
\notag &\leq& \mathbb{E}_{\hat\pi_X}[\hat U_t^*(X)-\hat L_t^*(X)] + 2\epsilon\text{\ \ because of $W_{\epsilon}$}\\
\notag&\leq& \inf_{(L,U)\in\mathcal{H}^2_{2t}\subset \hat{\mathcal{H}}^2_t}\mathbb{E}_{\hat\pi_X}[U(X)-L(X)]+2\epsilon\text{\ \ by optimality of $(\hat L_t^*,\hat U_t^*)$ in $\hat{\mathcal{H}}^2_t$}\\
\notag&\leq& \inf_{(L,U)\in\mathcal{H}^2_{2t}}\mathbb{E}_{\pi_X}[U(X)-L(X)]+4\epsilon\text{\ \ because of $W_{\epsilon}$}\\
\notag&=&\mathcal{R}_{2t}^*(\mathcal{H})+4\epsilon\\
\notag &\leq &\mathcal{R}^*(\mathcal{H})+\frac{12t}{(\alpha-2t)\gamma_{\frac{\alpha-2t}{3}}}+4\epsilon\text{\ \ by Theorem \ref{sensitivity bound}}.
\end{eqnarray}
Note that $\mathbb{P}(W_{\epsilon}\cap C_t)\geq 1-\mathbb{P}(W_{\epsilon}^c)-\mathbb{P}(C_t^c)\geq 1-\phi_1(n,\epsilon,\mathcal{H}) - \phi_2(n,t,\mathcal{H})$, concluding the theorem.
\end{proof}

\begin{proof}[Proof of Theorem \ref{rate:vc-subgraph}]
We will need the following results:
\begin{lemma}[Adapted from Theorem 2.6.7 in \citeAPX{van1996weak}]\label{vc:covering bound}
Let $\Vert g\Vert_{Q,2}$ be the $L_2$-norm of a function $g$ under a probability measure $Q$. For a VC-subgraph class $\mathcal{G}$ of functions from $\mathcal{X}$ to $\R$, and every probability measure $Q$ on $\mathcal{X}$, we have for every $\epsilon \in (0,1)$
\begin{equation*}
    N(\epsilon \Vert G\Vert_{Q,2},\mathcal{G},L_2(Q))\leq C(\vc{\mathcal{G}}+1)(16e)^{\vc{\mathcal{G}}+1}\big(\frac{1}{\epsilon}\big)^{2\vc{\mathcal{G}}}
\end{equation*}
where $G(x):=\sup_{g\in\mathcal{G}}\lvert g(x) \rvert$ is the envelope function of $\mathcal{G}$ and $C$ is a universal constant.
\end{lemma}
\begin{lemma}[Adapted from Theorem 2.14.1 in \citeAPX{van1996weak}]\label{L1 bound:covering}
Using the notations from Lemma \ref{vc:covering bound}, we define
\begin{equation*}
    J(\mathcal{G}):=\sup_{Q}\int_0^1\sqrt{1+\log N(\epsilon \Vert G\Vert_{Q,2},\mathcal{G},L_2(Q))}d\epsilon
\end{equation*}
where the supremum is taken over all discrete probability measures $Q$ with $\Vert G\Vert_{Q,2}<\infty$. Then we have
\begin{equation*}
\Vert\sup_{g\in\mathcal{G}}\lvert \mathbb{E}_{\hat\pi_X}[g(X)]-\mathbb{E}_{\pi_X}[g(X)] \rvert\Vert_1\leq \frac{1}{\sqrt{n}}\cdot CJ(\mathcal{G})\Vert G\Vert_2
\end{equation*}
where the $L_1$ norm $\Vert \cdot \Vert_1$ on the left hand size is with respect to the product measure $\pi_X^n$ (i.e., the data), and $C$ is a universal constant.
\end{lemma}
As a side note, a rigorous statement for the results in Lemma \ref{L1 bound:covering} involves a so-called P-measurability condition for the class $\mathcal{G}$, but we choose not to deal with the measurability requirement here. P-measurability holds for common function classes, e.g., if there exists a countable subclass $\mathcal{G}'\subset\mathcal{G}$ such that for every $g\in \mathcal{G}$ there exists a sequence from $\mathcal{G}'$ that converges to $g$ point-wise. We need one more result:
\begin{lemma}[Adapted from Theorem 2.14.5 in \citeAPX{van1996weak}]\label{sub-Gaussian bound}
Using the notations from Lemma \ref{vc:covering bound}, we have
\begin{eqnarray*}
    &&\Vert\sup_{g\in\mathcal{G}}\lvert \mathbb{E}_{\hat\pi_X}[g(X)]-\mathbb{E}_{\pi_X}[g(X)] \rvert\Vert_{\psi_2}\\
    &\leq& C\big( \Vert\sup_{g\in\mathcal{G}}\lvert \mathbb{E}_{\hat\pi_X}[g(X)]-\mathbb{E}_{\pi_X}[g(X)] \rvert\Vert_1 + \frac{1}{\sqrt{n}}\cdot \Vert G\Vert_{\psi_2}\big)
\end{eqnarray*}
where the sub-Gaussian norm $\Vert \cdot \Vert_{\psi_2}$ on the left hand size is with respect to the product measure $\pi_X^n$ (i.e., the data), and $C$ is a universal constant.
\end{lemma}

We now turn to the main proof. We first deal with $\phi_1$. Consider the centered class $\mathcal{H}_c:=\{h - \mathbb{E}_{\pi_X}[h(X)]:h\in\mathcal{H}\}$, whose envelope function is $H$. Since $\mathcal{H}_c\subset \mathcal{H}_+$, we have $\vc{\mathcal{H}_c}\leq \vc{\mathcal{H}_+}$. We calculate the complexity measure $J(\mathcal{H}_c)$ from Lemma \ref{L1 bound:covering}
\begin{eqnarray*}
J(\mathcal{H}_c)& =&\sup_{Q}\int_0^1\sqrt{1+\log N(\epsilon \Vert H\Vert_{Q,2},\mathcal{H}_c,L_2(Q))}d\epsilon\\
&\leq& \int_0^1\Big(1+2(\vc{\mathcal{H}_c}+1)\log\frac{1}{\epsilon}+16(\vc{\mathcal{H}_c}+1)+\log (\vc{\mathcal{H}_c}+1)+\log C\Big)^{\frac{1}{2}}d\epsilon\\
&&\hspace{2em}\text{by Lemma \ref{vc:covering bound}}\\
&\leq& \sqrt{1+16(\vc{\mathcal{H}_c}+1)+\log (\vc{\mathcal{H}_c}+1)+\log C}+\int_0^1\sqrt{2(\vc{\mathcal{H}_c}+1)\log\frac{1}{\epsilon}}d\epsilon\\
&\leq &C\sqrt{\vc{\mathcal{H}_c}}\text{\ \ \ for another universal constant }C\\
&\leq& C\sqrt{\vc{\mathcal{H}_+}}.
\end{eqnarray*}
Applying the bound in Lemma \ref{L1 bound:covering} to the class $\mathcal{H}_c$ gives
\begin{equation*}
\Vert\sup_{h\in\mathcal{H}_c}\lvert \mathbb{E}_{\hat\pi_X}[h(X)]-\mathbb{E}_{\pi_X}[h(X)] \rvert\Vert_1\leq \sqrt{\frac{\vc{\mathcal{H}_+}}{n}}\cdot C\Vert H\Vert_2.
\end{equation*}
Further applying Lemma \ref{sub-Gaussian bound}, and using the fact that $\Vert \cdot \Vert_2\leq C\Vert \cdot \Vert_{\psi_2}$ for some universal constant $C$ lead to
\begin{equation*}
    \Vert\sup_{h\in\mathcal{H}_c}\lvert \mathbb{E}_{\hat\pi_X}[h(X)]-\mathbb{E}_{\pi_X}[h(X)] \rvert\Vert_{\psi_2}\leq \sqrt{\frac{\vc{\mathcal{H}_+}}{n}}\cdot C\Vert H\Vert_{\psi_2}.
\end{equation*}
Finally, note that $\sup_{h\in\mathcal{H}}\lvert \mathbb{E}_{\hat\pi_X}[h(X)]-\mathbb{E}_{\pi_X}[h(X)] \rvert = \sup_{h\in\mathcal{H}_c}\lvert \mathbb{E}_{\hat\pi_X}[h(X)]-\mathbb{E}_{\pi_X}[h(X)] \rvert$, therefore the same bound holds for $\Vert\sup_{h\in\mathcal{H}}\lvert \mathbb{E}_{\hat\pi_X}[h(X)]-\mathbb{E}_{\pi_X}[h(X)] \rvert\Vert_{\psi_2}$. The sub-Gaussian tail bound then gives the expression for $\phi_1$.

Next we analyze $\phi_2$. First note that $\mathcal{H}\subset\mathcal{H}_+$, therefore $\vc{\mathcal{H}}\leq \vc{\mathcal{H}_+}$. By the definition of VC-subgraph, both its closed subgraph class $\mathcal{S}_{upper}:=\{\{(x,y):y\leq U(x)\}:U\in\mathcal{H}\}$ and open subgraph class $\mathcal{S}'_{lower}:=\{\{(x,y):y< L(x)\}:L\in\mathcal{H}\}$ have a VC dimension $\vc{\mathcal{S}_{upper}}=\vc{\mathcal{S}'_{lower}}=\vc{\mathcal{H}}$ (using $\leq$ or $<$ for defining subgraphs does not affect the resulting VC dimension, see Problem 10 from Section 2.6 in \citeAPX{van1996weak}). To proceed, we need the following preservation result for VC classes:
\begin{lemma}[Adapted from Lemma 9.7 statements (i) and (v) in \citeAPX{kosorok2007introduction}]\label{vc preservation}
Let $\mathcal{S}$ be a VC class of sets in a space $\mathbb{S}$, then
\begin{description}
\item[1.] $\psi:\mathbb{S}\to \mathbb{S}$ be a one-to-one mapping, then the class $\psi(\mathcal{S}):=\{\{\psi(s):s\in S\}:S\in\mathcal{S}\}$ is also a VC class with $\vc{\psi(\mathcal{S})}=\vc{\mathcal{S}}$
\item[2.] The complement class $\mathcal{S}^c:=\{\mathbb{S}\backslash S:S\in\mathcal{S}\}$ is a VC class with $\vc{\mathcal{S}^c}=\vc{\mathcal{S}}$.
\end{description}
\end{lemma}
and a VC dimension bound for unions and intersections of VC classes:
\begin{lemma}[Adapted from Theorem 1.1 in \citeAPX{van2009note}]\label{vc bound for unions and intersections}
Suppose $\mathcal{S}_1,\ldots,\mathcal{S}_K$ are VC classes of sets in a space $\mathbb{S}$. Define $\sqcup_{k=1}^K\mathcal{S}_k:=\{\cup_{k=1}^KS_k:S_k\in\mathcal{S}_k\text{ for }k=1,\ldots,K\}$ and $\sqcap_{k=1}^K\mathcal{S}_k:=\{\cap_{k=1}^KS_k:S_k\in\mathcal{S}_k\text{ for }k=1,\ldots,K\}$. We have
\begin{equation*}
    \vc{\sqcup_{k=1}^K\mathcal{S}_k}\leq C\log(K)\sum_{k=1}^K\vc{\mathcal{S}_k},\;\;\;\vc{\sqcap_{k=1}^K\mathcal{S}_k}\leq C\log(K)\sum_{k=1}^K\vc{\mathcal{S}_k}
\end{equation*}
for some universal constant $C$.
\end{lemma}
Further consider $\mathcal{S}_{lower}:=\{\{(x,y):y\geq L(x)\}:L\in\mathcal{H}\}$, and $\mathcal{S}_{btw}:=\{\{(x,y):L(x)\leq y\leq U(x)\}:L,U\in\mathcal{H}\text{ and }L\leq U\}$. We observe that $\mathcal{S}_{lower} = \mathcal{S}_{lower}^{\prime c}$, and that $\mathcal{S}_{btw}\subset \mathcal{S}_{lower}\sqcap \mathcal{S}_{upper}$. Therefore by Lemma \ref{vc preservation} we have $\vc{\mathcal{S}_{lower}}=\vc{\mathcal{S}'_{lower}}$, and applying the bound for intersection from Lemma \ref{vc bound for unions and intersections} to $\mathcal{S}_{btw}$ gives $\vc{\mathcal{S}_{btw}}\leq \vc{\mathcal{S}_{lower}\sqcap \mathcal{S}_{upper}}\leq C\vc{\mathcal{S}_{upper}}=C\vc{\mathcal{H}}$ for some universal constant $C$. With this VC bound for $\mathcal{S}_{btw}$, we are ready to use standard deviation bounds for VC set classes (see, e.g., equation (3.3) in \citeAPX{vapnik2013nature}) to get
\begin{eqnarray*}
&&\sup_{L,U\in\mathcal{H}\;\text{and}\;L\leq U}\mathbb{P}(\lvert \mathbb{P}_{\hat\pi}(L(X)\leq Y\leq U(X)) -  \mathbb{P}_{\pi}(L(X)\leq Y\leq U(X))\rvert>t)\\
&=&\sup_{S\in\mathcal{S}_{btw}}\mathbb{P}(\lvert \mathbb{P}_{\hat\pi}((X,Y)\in S) -  \mathbb{P}_{\pi}((X,Y)\in S)\rvert>t)\\
&\leq&4\mathrm{Growth}(2n)\exp(- t^2n)\\
\end{eqnarray*}
where $\mathrm{Growth}(2n)$ is the growth function, or the shattering number, for the class $\mathcal{S}_{btw}$. By the Sauer–Shelah lemma we have $\mathrm{Growth}(2n)=2^{2n}$ if $2n<\vc{\mathcal{S}_{btw}}$ and $\leq \big(\frac{2en}{\vc{\mathcal{S}_{btw}}}\big)^{\vc{\mathcal{S}_{btw}}}$ if $2n\geq \vc{\mathcal{S}_{btw}}$. With the upper bound for $\vc{\mathcal{S}_{btw}}$, we can bound
\begin{equation*}
    \mathrm{Growth}(2n)\leq 
    \begin{cases}
    2^{2n}&\text{if }2n< C\vc{\mathcal{H}}\\
    \big(\frac{2en}{C\vc{\mathcal{H}}}\big)^{C\vc{\mathcal{H}}}&\text{if }2n\geq C\vc{\mathcal{H}}
    \end{cases}
\end{equation*}
giving rise to our formula for $\phi_2$.

The feasibility and optimality bound for $\hat\delta_t^*$ can be obtained by solving $\phi_1(n,\epsilon,\mathcal{H})=\frac{\eta}{2}$ and $\phi_2(n,t,\mathcal{H})=\frac{\eta}{2}$ for $\epsilon$ and $t$, and then applying Theorem \ref{convergence rate}.
\end{proof}

\begin{proof}[Proof of Theorem \ref{rate:Lipschitz}]
We first treat $\phi_1$. We need a maximal inequality that is similar to Lemma \ref{L1 bound:covering}, but based on bracketing numbers instead:
\begin{lemma}[Adapted from Theorem 2.14.2 in \citeAPX{van1996weak}]\label{L1 bound:bracketing}
Using the notations from Lemma \ref{vc:covering bound}, for a function class $\mathcal{G}$ we define
\begin{equation*}
    J_{[]}(\mathcal{G}):=\int_0^1\sqrt{1+\log N(\epsilon \Vert G\Vert_{\pi_X,2},\mathcal{G},L_2(\pi_X))}d\epsilon.
\end{equation*}
Then we have
\begin{equation*}
\Vert\sup_{g\in\mathcal{G}}\lvert \mathbb{E}_{\hat\pi_X}[g(X)]-\mathbb{E}_{\pi_X}[g(X)] \rvert\Vert_1\leq \frac{1}{\sqrt{n}}\cdot CJ_{[]}(\mathcal{G})\Vert G\Vert_2
\end{equation*}
where the $L_1$ norm $\Vert \cdot \Vert_1$ on the left hand size is with respect to the product measure $\pi_X^n$ (i.e., the data), and $C$ is a universal constant.
\end{lemma}
We consider the centered class $\mathcal{H}_c:=\{h - \mathbb{E}_{\pi_X}[h(X)]:h\in\mathcal{H}\}$ as in the proof of Theorem \ref{rate:vc-subgraph}. Note that, by Jensen's inequality, the Lipschitzness condition stipulates that $\lvert \mathbb{E}_{\pi_X}[h(X,\theta_1)]-\mathbb{E}_{\pi_X}[h(X,\theta_2)] \rvert\leq \Vert \mathcal L\Vert_1\Vert \theta_1-\theta_2\Vert_2$, therefore the centered class is also Lipschitz in $\theta$, with a slightly larger coefficient
\begin{equation*}
    \lvert h(x,\theta_1) - \mathbb{E}_{\pi_X}[h(X,\theta_1)]-(h(x,\theta_2) - \mathbb{E}_{\pi_X}[h(X,\theta_2)]) \rvert\leq  (\mathcal L(x)+\Vert \mathcal L\Vert_1)\Vert \theta_1-\theta_2\Vert_2.
\end{equation*}
The envelope function of $\mathcal{H}_c$ is $H$. We then calculate the bracketing number of $\mathcal{H}_c$. Using the Lipschitzness condition, the bracketing number of $\mathcal{H}_c$ can be bounded by the covering number of the parameter space $\Theta$ as below
\begin{equation*}
    N_{[]}(4\epsilon\Vert \mathcal L \Vert_2,\mathcal{H}_c,L_2(\pi_X))\leq N(\epsilon,\Theta,\Vert \cdot \Vert_2)
\end{equation*}
where $N(\epsilon,\Theta,\Vert \cdot \Vert_2)$ is the $\epsilon$-covering number of $\Theta$ with respect to the $l_2$ norm, i.e., the minimum number of $l_2$-balls of size $\epsilon$ needed to cover $\Theta$. Since $\Theta$ is bounded, its covering number is upper bounded by that of the $l_2$-ball of radius $\diam{\Theta}$, which is further bounded by $\big(\frac{3\diam{\Theta}}{\epsilon}\big)^l$ (see Problem 6 from Section 2.1 in \citeAPX{van1996weak}). All these lead to
\begin{equation*}
    N_{[]}(\epsilon\Vert H \Vert_2,\mathcal{H}_c,L_2(\pi_X))\leq N(\frac{\epsilon\Vert H\Vert_2}{4\Vert \mathcal L\Vert_2},\Theta,\Vert \cdot \Vert_2)\leq \big(\frac{12\diam{\Theta}\Vert \mathcal L\Vert_2}{\epsilon \Vert H\Vert_2}\big)^l.
\end{equation*}
Also note that when $\epsilon\geq \frac{4\diam{\Theta}\Vert \mathcal L\Vert_2}{\Vert H\Vert_2}$ the bracketing number $N_{[]}(\epsilon\Vert H \Vert_2,\mathcal{H}_c,L_2(\pi_X))=1$ because $N(\diam{\Theta},\Theta,\Vert \cdot \Vert_2)=1$. We can now compute the complexity measure $J_{[]}(\mathcal{H}_c)$ as follows
\begin{eqnarray*}
J_{[]}(\mathcal{H}_c)&\leq &\int_{0}^{\min\{1,\frac{4\diam{\Theta}\Vert \mathcal L\Vert_2}{\Vert H\Vert_2}\}}\sqrt{1+l\log \frac{12\diam{\Theta}\Vert \mathcal L\Vert_2}{\epsilon \Vert H\Vert_2}}d\epsilon+1-\min\{1,\frac{4\diam{\Theta}\Vert \mathcal L\Vert_2}{\Vert H\Vert_2}\}\\
&\leq& 1+\sqrt{l}\int_{0}^{\min\{1,\frac{4\diam{\Theta}\Vert \mathcal L\Vert_2}{\Vert H\Vert_2}\}}\sqrt{\log \frac{12\diam{\Theta}\Vert \mathcal L\Vert_2}{\epsilon \Vert H\Vert_2}}d\epsilon\\
&= &1+\sqrt{l}\cdot\frac{12\diam{\Theta}\Vert \mathcal L\Vert_2}{\Vert H\Vert_2}\int_{0}^{\min\{\frac{1}{3}, \frac{\Vert H\Vert_2}{12\diam{\Theta}\Vert \mathcal L\Vert_2}\}}\sqrt{\log \frac{1}{\epsilon}}d\epsilon\\
&= &1+C\sqrt{l}\cdot\frac{12\diam{\Theta}\Vert \mathcal L\Vert_2}{\Vert H\Vert_2}\sqrt{\log \max\{3,\frac{12\diam{\Theta}\Vert \mathcal L\Vert_2}{\Vert H\Vert_2}\}}\cdot \min\{\frac{1}{3}, \frac{\Vert H\Vert_2}{12\diam{\Theta}\Vert \mathcal L\Vert_2}\}\\
&&\hspace{2em}\text{by Lemma \ref{bound for log integral} below}\\
&\leq &1+C\sqrt{l}\cdot\sqrt{\log \max\{3,\frac{12\diam{\Theta}\Vert \mathcal L\Vert_2}{\Vert H\Vert_2}\}}\cdot \min\{\frac{4\diam{\Theta}\Vert \mathcal L\Vert_2}{\Vert H\Vert_2}, 1\}\\
&\leq &C\sqrt{l}\cdot\sqrt{\max\{\log\frac{\diam{\Theta}\Vert \mathcal L\Vert_2}{\Vert H\Vert_2},1\}}.
\end{eqnarray*}
\begin{lemma}\label{bound for log integral}
For every $c\in (0,\frac{1}{3}]$, we have
\begin{equation*}
    \int_0^c\sqrt{\log\frac{1}{\epsilon}}d\epsilon\leq C\cdot c\sqrt{\log\frac{1}{c}}
\end{equation*}
where $C$ is a universal constant.
\end{lemma}
\begin{proof}[Proof of Lemma \ref{bound for log integral}]
By a change of variable $t = \sqrt{\log\frac{1}{\epsilon}}$, we write
\begin{eqnarray*}
\int_0^c\sqrt{\log\frac{1}{\epsilon}}d\epsilon&=&\int_{\sqrt{\log\frac{1}{c}}}^{\infty}2t^2\exp(-t^2)dt\\
&=&-t\exp(-t^2)\Big\vert_{t=\sqrt{\log\frac{1}{c}}}^{t=\infty} + \int_{\sqrt{\log\frac{1}{c}}}^{\infty}\exp(-t^2)dt\\
&\leq &c\sqrt{\log\frac{1}{c}}+\int_{\sqrt{\log\frac{1}{c}}}^{\infty}\frac{t}{\sqrt{\log\frac{1}{c}}}\exp(-t^2)dt\\
&=&c\sqrt{\log\frac{1}{c}}+\frac{c}{2\sqrt{\log\frac{1}{c}}}\leq \big(1+ \frac{1}{2\log 3}\big)c\sqrt{\log\frac{1}{c}}
\end{eqnarray*}
where in the last line we use $c\leq \frac{1}{3}$.
\end{proof}
Lemma \ref{L1 bound:bracketing} then entails the following maximal inequality for the centered class $\mathcal{H}_c$
\begin{equation*}
    \Vert\sup_{h\in\mathcal{H}_c}\lvert \mathbb{E}_{\hat\pi_X}[h(X)]-\mathbb{E}_{\pi_X}[h(X)] \rvert\Vert_1\leq \sqrt{\frac{l}{n}}\cdot C\sqrt{\max\{\log\frac{\diam{\Theta}\Vert \mathcal L\Vert_2}{\Vert H\Vert_2},1\}}\Vert H\Vert_2.
\end{equation*}
Further applying Lemma \ref{sub-Gaussian bound} gives
\begin{equation*}
    \Vert\sup_{h\in\mathcal{H}_c}\lvert \mathbb{E}_{\hat\pi_X}[h(X)]-\mathbb{E}_{\pi_X}[h(X)] \rvert\Vert_{\psi_2}\leq \sqrt{\frac{l}{n}}\cdot C\sqrt{\max\{\log\frac{\diam{\Theta}\Vert \mathcal L\Vert_2}{\Vert H\Vert_2},1\}}\Vert H\Vert_{\psi_2}.
\end{equation*}
Finally, note that $\sup_{h\in\mathcal{H}}\lvert \mathbb{E}_{\hat\pi_X}[h(X)]-\mathbb{E}_{\pi_X}[h(X)] \rvert=\sup_{h\in\mathcal{H}_c}\lvert \mathbb{E}_{\hat\pi_X}[h(X)]-\mathbb{E}_{\pi_X}[h(X)] \rvert$, hence the same sub-Gaussian norm holds for the original class $\mathcal{H}$ too. The tail bound $\phi_1$ follows from the sub-Gaussian tail bound.

Secondly, we analyze $\phi_2$. We first calculate the bracketing number of the corresponding indicator class
\begin{equation*}
    \mathcal{H}_{ind}:=\{(x,y)\to I_{h(x,\theta_l)\leq y\leq h(x,\theta_u)}:\theta_l,\theta_u\in\Theta,\text{ and }h(\cdot,\theta_l)\leq h(\cdot,\theta_u)\}.
\end{equation*}
For fixed $\theta_l^o,\theta_u^o\in\Theta$ and $\epsilon>0$, consider a bracket enclosed by
\begin{align*}
l^o(x,y)&:=   I_{h(x,\theta_l^o) +\mathcal L(x)\epsilon \leq y\leq h(x,\theta_u^o) -\mathcal L(x)\epsilon}\\
u^o(x,y)&:=I_{h(x,\theta_l^o) -\mathcal L(x)\epsilon\leq y\leq h(x,\theta_u^o) +\mathcal L(x)\epsilon}
\end{align*}
where $\mathcal L(x)$ is the Lipschitz coefficient. It is clear that $l^o\leq u^o$. By Lipschitzness, for all $\theta_l,\theta_u$ such that $\Vert \theta_l - \theta_l^o\Vert_2\leq \epsilon$ and $\Vert \theta_u - \theta_u^o\Vert_2\leq \epsilon$ we have $h(x,\theta_l^o)-\mathcal L(x)\epsilon\leq h(x,\theta_l)\leq h(x,\theta_l^o)+\mathcal L(x)\epsilon$ (similar for $h(x,\theta_u)$). Therefore $l^o(x,y)\leq I_{h(x,\theta_l)\leq y\leq h(x,\theta_u)}\leq u^o(x,y)$, i.e., $I_{h(x,\theta_l)\leq y\leq h(x,\theta_u)}$ belongs to the bracket $[l^o,u^o]$ whenever $\Vert \theta_l - \theta_l^o\Vert_2\leq \epsilon$ and $\Vert \theta_u - \theta_u^o\Vert_2\leq \epsilon$. To calculate the size of the bracket, we write
\begin{eqnarray*}
&&\mathbb{E}_{\pi}[u^o(X,Y)-l^o(X,Y)]\\
&=&\mathbb{P}_{\pi}(h(X,\theta_l^o) - \mathcal L(X)\epsilon\leq Y < h(X,\theta_l^o) + \mathcal L(X)\epsilon)+\mathbb{P}_{\pi}(h(X,\theta_u^o) - \mathcal L(X)\epsilon < Y\leq h(X,\theta_u^o) + \mathcal L(X)\epsilon)\\
&=&\mathbb{E}_{\pi_X}[\mathbb{P}_{\pi}(h(X,\theta_l^o) - \mathcal L(X)\epsilon\leq Y < h(X,\theta_l^o) + \mathcal L(X)\epsilon\vert X)]+\\
&&\hspace{2em}\mathbb{E}_{\pi_X}[\mathbb{P}_{\pi}(h(X,\theta_u^o) - \mathcal L(X)\epsilon < Y\leq h(X,\theta_u^o) + \mathcal L(X)\epsilon\vert X)]\\
&\leq& 2\mathbb{E}_{\pi_X}[2D_{Y\vert X}\mathcal L(X)\epsilon]=4D_{Y\vert X}\Vert \mathcal L\Vert_1\epsilon.
\end{eqnarray*}
Note that the bracket size is independent of $\theta_l$ and $\theta_u$, therefore the bracketing number $N_{[]}(4D_{Y\vert X}\Vert \mathcal L\Vert_1\epsilon,\mathcal{H}_{ind}, L_1(\pi))\leq \big(N(\epsilon,\Theta,\Vert \cdot \Vert_2)\big)^2\leq \big( \frac{3\diam{\Theta}}{\epsilon} \big)^{2l}$, i.e.,
\begin{equation}\label{bound for bracketing number}
    N_{[]}(\epsilon,\mathcal{H}_{ind}, L_1(\pi))\leq \big( \frac{12\diam{\Theta}D_{Y\vert X}\Vert \mathcal L\Vert_1}{\epsilon} \big)^{2l}=\big( \frac{C_{\mathcal{H}}}{\epsilon} \big)^{2l}.
\end{equation}
To derive the deviation bound using the bracketing number \eqref{bound for bracketing number}, we need one more result:
\begin{lemma}[Adapted from Theorem 6.8 in \citeAPX{talagrand1994sharper}]\label{deviation of indicator:bracketing}
Suppose $\mathcal{G}$ is a class of measurable indicator functions from $\mathcal{X}\times \R\to \R$, and that its $\epsilon$-bracketing number $N_{[]}(\epsilon, \mathcal{G}, L_1(\pi))\leq \big(\frac{V}{\epsilon}\big)^{\nu}$ for $V,\nu>0$, then there exists a universal constant $C$ such that for all $t\geq C\sqrt{\frac{\nu\log (V)\log\log (V)}{n}}$ we have
\begin{equation*}
    \mathbb{P}(\sup_{g\in \mathcal{G}}\lvert \mathbb{E}_{\hat\pi}[g(X,Y)] - \mathbb{E}_{\pi}[g(X,Y)] \rvert>t)\leq \frac{C}{t\sqrt{n}}\big( \frac{CVt^2n}{\nu}\big)^{\nu}\exp(-2t^2n).
\end{equation*}
\end{lemma}
Applying Lemma \ref{deviation of indicator:bracketing} to the indicator class $\mathcal{H}_{ind}$, we obtain
\begin{eqnarray}
    \notag&&\mathbb{P}(\sup_{L,U\in\mathcal{H}\;\text{and}\;L\leq U}\lvert \mathbb{P}_{\hat\pi}(Y\in [L(X),U(X)]) - \mathbb{P}_{\pi}(Y\in [L(X),U(X)])\rvert>t)\\
    \notag &\leq&\frac{C}{t\sqrt{n}}\big( \frac{CC_{\mathcal{H}}t^2n}{2l}\big)^{2l}\exp(-2t^2n)\text{ \ \ \ if }t\geq C\sqrt{\frac{2l\log (C_{\mathcal{H}})\log\log (C_{\mathcal{H}})}{n}}\\
 &\leq&\big( \frac{CC_{\mathcal{H}}t^2n}{2l}\big)^{2l}\exp(-2t^2n)\text{ \ \ \ assuming }2l\log (C_{\mathcal{H}})\log\log (C_{\mathcal{H}})\geq 1.\label{tail bound of indicator:bracketing}
\end{eqnarray}
To make the bound \eqref{tail bound of indicator:bracketing} valid for small $t$, we next enlarge the constant $C$ so that the bound \eqref{tail bound of indicator:bracketing} is trivial. That is, we seek for a $\kappa\geq 1$ such that the bound from \eqref{tail bound of indicator:bracketing} satisfies
\begin{equation*}
    \big( \frac{\kappa CC_{\mathcal{H}}t^2n}{2l}\big)^{2l}\exp(-2t^2n)\geq 1\text{ \ \ \ when }t= C\sqrt{\frac{2l\log (C_{\mathcal{H}})\log\log (C_{\mathcal{H}})}{n}}
\end{equation*}
which reduces to
\begin{equation*}
    \big( \kappa C^3C_{\mathcal{H}}\log (C_{\mathcal{H}})\log\log (C_{\mathcal{H}})\big)^{2l}\exp(-4C^2l\log (C_{\mathcal{H}})\log\log (C_{\mathcal{H}}))\geq 1.
\end{equation*}
Solving the above inequality for $\kappa$ gives
\begin{equation*}
    \kappa\geq \frac{C_{\mathcal{H}}^{2C^2\log\log(C_{\mathcal{H}})-1}}{C^3\log(C_{\mathcal{H}})\log\log(C_{\mathcal{H}})}.
\end{equation*}
Using this $\kappa$ in \eqref{tail bound of indicator:bracketing} leads to a trivial bound for $t= C\sqrt{\frac{2l\log (C_{\mathcal{H}})\log\log (C_{\mathcal{H}})}{n}}$, and for smaller $t$ we simply use this trivial bound, therefore we obtain the following unified tail bound
\begin{eqnarray}
    \notag&&\mathbb{P}(\sup_{L,U\in\mathcal{H}\;\text{and}\;L\leq U}\lvert \mathbb{P}_{\hat\pi}(Y\in [L(X),U(X)]) - \mathbb{P}_{\pi}(Y\in [L(X),U(X)])\rvert>t)\\
    \notag &\leq&\Big( \frac{C\cdot C_{\mathcal{H}}^{2C^2\log\log(C_{\mathcal{H}})}}{C^32l\log(C_{\mathcal{H}})\log\log(C_{\mathcal{H}})}\max\{t^2n,2C^2l\log(C_{\mathcal{H}})\log\log(C_{\mathcal{H}})\}\Big)^{2l}\exp(-2t^2n)\\
    \notag&\leq &\Big(C_{\mathcal{H}}^{2C^2\log\log(C_{\mathcal{H}})}\max\{\frac{t^2n}{2C^2l\log(C_{\mathcal{H}})\log\log(C_{\mathcal{H}})},1\}\Big)^{2l}\exp(-2t^2n)\\
    \notag&\leq &\Big(C_{\mathcal{H}}^{2C^2\log\log(C_{\mathcal{H}})}\max\{\frac{t^2n}{2C^2l},1\}\Big)^{2l}\exp(-2t^2n)\text{\ \ \ assuming }\log(C_{\mathcal{H}})\log\log(C_{\mathcal{H}})\geq 1.
\end{eqnarray}
Replacing $C^2$ with $C$ in the above bound gives the expression for $\phi_2$.

Finally, like in Theorem \ref{rate:vc-subgraph}, we solve $\phi_1(n,\epsilon,\mathcal{H})=\frac{\eta}{2}$ and $\phi_2(n,t,\mathcal{H})=\frac{\eta}{2}$ for $\epsilon$ and $t$, and then apply Theorem \ref{convergence rate} to get the feasibility and optimality errors. We briefly explain how $t$ is derived. Consider the case $\frac{t^2n}{2Cl}\geq 1$, so we can derive the following upper bound for $\phi_2$
\begin{eqnarray*}
\phi_2(n,t,\mathcal{H})&\leq& \Big(C_{\mathcal{H}}^{2C\log\log(C_{\mathcal{H}})}\frac{t^2n}{2Cl}\Big)^{2l}\exp(-2t^2n)\\
&=&\Big(\frac{C_{\mathcal{H}}^{2C\log\log(C_{\mathcal{H}})}}{C}\Big)^{2l}\cdot \big(\frac{t^2n}{2l}\exp(-2\frac{t^2n}{2l})\big)^{2l}\\
&< &\Big(\frac{C_{\mathcal{H}}^{2C\log\log(C_{\mathcal{H}})}}{C}\Big)^{2l}\cdot \big(\exp(-\frac{t^2n}{2l})\big)^{2l}\text{\ \ \ using }\frac{t^2n}{2l}< \exp(\frac{t^2n}{2l})\\
&\leq &\Big(C_{\mathcal{H}}^{2C\log\log(C_{\mathcal{H}})}\Big)^{2l}\cdot \exp(-t^2n)\text{\ \ \ \ assuming }C\geq 1.
\end{eqnarray*}
The choice of $t$ presented in the theorem is obtained by equating this upper bound to $\frac{\eta}{2}$.
\end{proof}

\begin{proof}[Proof of Theorem \ref{vc bound:regression tree}]
The regression tree class $\mathcal{H}$ consists of all functions that of form $h(x)=\sum_{s=1}^{S+1}c_sI_{x\in R_s}$, and note that the augmented class $\mathcal{H}_+$ consists of functions $\sum_{s=1}^{S+1}(c_s-c)I_{x\in R_s}$ which are of the same form, therefore $\mathcal{H}=\mathcal{H}_+$. As discussed before, the subgraph of a regression tree takes the form of the union of at most $S+1$ hyper-rectangles in $\R^{d+1}$, where each rectangle is formed by at most $S$ axis-parallel cuts.

We first calculate the VC dimension of the set of all axis-parallel cuts. Four sets of axis-parallel cuts in $\R^{d+1}$ are defined as
\begin{align*}
    \mathcal{C}^{\leq}_{d+1}&:=\{\{(x_1,\ldots,x_{d+1})\in \R^{d+1}:x_j\leq a\}:j\in \{1,2,\ldots,d+1\},a\in [-\infty,+\infty]\}\\
    \mathcal{C}^{<}_{d+1}&:=\{\{(x_1,\ldots,x_{d+1})\in \R^{d+1}:x_j< a\}:j\in \{1,2,\ldots,d+1\},a\in [-\infty,+\infty]\}\\
    \mathcal{C}^{\geq}_{d+1}&:=\{\{(x_1,\ldots,x_{d+1})\in \R^{d+1}:x_j\geq a\}:j\in \{1,2,\ldots,d+1\},a\in [-\infty,+\infty]\}\\
    \mathcal{C}^{>}_{d+1}&:=\{\{(x_1,\ldots,x_{d+1})\in \R^{d+1}:x_j> a\}:j\in \{1,2,\ldots,d+1\},a\in [-\infty,+\infty]\}.
\end{align*}
Proposition 1 in \citeAPX{gey2018vapnik} states that $\vc{\mathcal{C}^{\leq}_{d+1}}\leq C \log d$ for some universal constant $C$. Note that $\mathcal{C}^{>}_{d+1}$ is the complement class of $\mathcal{C}^{\leq}_{d+1}$, therefore by Lemma \ref{vc preservation} statement 2 we have the same bound for $\vc{\mathcal{C}^{>}_{d+1}}$. The class $\mathcal{C}^{\geq}_{d+1}$ (resp. $\mathcal{C}^{<}_{d+1}$) can be mapped from $\mathcal{C}^{\leq}_{d+1}$ (resp. $\mathcal{C}^{>}_{d+1}$) via the mapping $(x_1,\ldots,x_{d+1})\to (-x_1,\ldots,-x_{d+1})$, therefore by statement 1 in Lemma \ref{vc preservation} we have the same VC bound for them too. Using Lemma \ref{vc bound for unions and intersections} we know that the VC bound for the set of all axis-parallel cuts $\mathcal{C}_{d+1}:=\mathcal{C}^{\leq}_{d+1}\sqcup \mathcal{C}^{<}_{d+1}\sqcup \mathcal{C}^{\geq}_{d+1}\sqcup \mathcal{C}^{>}_{d+1}$ has VC dimension $\vc{\mathcal{C}_{d+1}}\leq C\log d$.

Now we can readily obtain $\vc{\mathcal{H}}$ via intersections and unions. Each hyper-rectangle $R_s\times (-\infty,c_s)$ is the intersection of at most $S+1$ axis-parallel cuts in $\R^{d+1}$. Formally, denoting by $\mathcal{R}$ as set of all such hyper-rectangles, then $\mathcal{R}\subset \sqcap_{s=1}^{S+1}\mathcal{C}_{d+1}$, and hence $\vc{\mathcal{R}}\leq CS\log (d)\log(S)$ by Lemma \ref{vc bound for unions and intersections}. Moreover, the set of subgraphs of regression trees is a subset of $\sqcup_{s=1}^{S+1}\mathcal{R}$, therefore the VC dimension of the set of subgraphs is at most $CS^2\log(d)\log^2(S)$ again by Lemma \ref{vc bound for unions and intersections}.

Finally, when $\max_{x}h(x)-\min_{x}h(x)\leq M$, then the envelope $H \leq 2M$ in Theorem \ref{rate:vc-subgraph}, therefore the sub-Gaussian norm $\Vert H\Vert_{\psi_2}\leq C'M$ for some universal constant $C'$.
\end{proof}

\begin{proof}[Proof of Theorem \ref{Lipschitz property for neural network}]
To make the parameterizaton more instructive, we explicitly write $\theta=(W_1,b_1,\ldots,W_S,b_S)$ in terms of the weights and biases. We consider a perturbation $\Delta \theta:=(\Delta W_1,\Delta b_1,\ldots,\Delta W_S,\Delta b_S)$, and wants to bound the difference $\lvert h(x,\theta+\Delta \theta) - h(x,\theta) \rvert$. Denoting $W'_s = W_s+\Delta W_s$ and $b'_s = b_s+\Delta b_s$, we define two mappings
\begin{align*}
    \psi_{s:S}&:x\in \R^{n_s}\to \phi_S(W_S\phi_{S-1}(\cdots \phi_{s+1}(W_{s+1}\phi_s(x)+b_{s+1}) \cdots)+b_S)\in \R\\
    \psi'_{0:s-1}&:x\in \mathcal{X}\to \phi_{s-1}(W'_{s-1}\phi_{s-2}(\cdots \phi_{1}(W'_{1}x+b'_{1}) \cdots)+b'_{s-1})\in \R^{n_{s-1}}.
\end{align*}
Then the difference can be expressed as
\begin{eqnarray}
    \notag&&\lvert h(x;W'_1,b'_1,\ldots,W'_S,b'_S)-h(x;W_1,b_1,\ldots,W_S,b_S)\rvert\\
   \notag &\leq &\sum_{s=1}^S\lvert h(x;W'_1,b'_1,\ldots, W'_s,b'_s, W_{s+1},b_{s+1},\ldots,W_S,b_S) -\\
   \notag &&\hspace{2em}h(x;W'_1,b'_1,\ldots, W'_{s-1},b'_{s-1}, W_s,b_s,\ldots,W_S,b_S) \rvert\\
    \notag&\leq& \sum_{s=1}^S\lvert \psi_{s:S}(W'_s\psi'_{1:s-1}(x)+b'_s)-\psi_{s:S}(W_s\psi'_{1:s-1}(x)+b_s) \rvert\\
    \notag&\leq& \sum_{s=1}^SL_{\psi_{s:S}}\Vert \Delta W_s\psi'_{0:s-1}(x)+\Delta b_s \Vert_2\text{\ \ \ where $L_{\psi_{s:S}}$ is the Lipschitz constant of $\psi_{s:S}$}\\
    \notag&\leq& \sum_{s=1}^SL_{\psi_{s:S}}(\Vert\psi'_{0:s-1}(x) \Vert_2+1)\Vert \big[\Delta W_s,\Delta b_s\big] \Vert_2 \\
    &\leq& \sum_{s=1}^SL_{\psi_{s:S}}(\Vert\psi'_{0:s-1}(x) \Vert_2+1)\Vert \big[\Delta W_s,\Delta b_s\big] \Vert_F \label{bound from telescoping}
\end{eqnarray}
where in the last two lines $\Vert \cdot \Vert_2$ and $\Vert \cdot \Vert_F$ respectively denote the spectral and Frobenius norms of a matrix. Therefore the problem boils down to bounding the Lipschitz constant $L_{\psi_{s:S}}$ and the norm of the intermediate output $\Vert\psi'_{0:s-1}(x) \Vert_2$.

We first calculate $L_{\psi_{s:S}}$. This is relatively straightforward, since $\psi_{s:S}$ is a composition of linear mappings and activation functions. By the chain rule we have
\begin{equation*}
    L_{\psi_{s:S}}\leq M\cdot \prod_{k=s+1}^SM\Vert W_k \Vert_2\leq M^{S-s+1}\prod_{k=s+1}^S\Vert W_k \Vert_F\leq M^{S-s+1}(B\sqrt{W})^{S-s}
\end{equation*}
where the last inequality uses the fact that $\Vert W_k\Vert_F\leq B\sqrt{W}$ because each entry in $W_k$ is bounded within $[-B,B]$ and there are $W$ parameters in total.

Then we bound $\Vert\psi'_{0:s-1}(x) \Vert_2$. Note that $\psi'_{0:s-1}(x)=\phi_{s-1}(W'_{s-1}\psi'_{0:s-2}(x)+b'_{s-1})$, and $\psi'_{0:0}(x)=x$, therefore we have the following recursion
\begin{eqnarray*}
    \Vert\psi'_{0:s-1}(x) \Vert_2&=&\Vert \phi_{s-1}(W'_{s-1}\psi'_{0:s-2}(x)+b'_{s-1})\Vert_2\\
    &\leq& M_0\sqrt{U}+M\Vert \big[W'_{s-1},b'_{s-1}\big] \Vert_2(\Vert \psi'_{0:s-2}(x) \Vert_2+1)\\
    &\leq& M_0\sqrt{U}+M\Vert \big[W'_{s-1},b'_{s-1}\big] \Vert_F(\Vert \psi'_{0:s-2}(x) \Vert_2+1)\\
    &\leq &M_0\sqrt{U}+MB\sqrt{W}(\Vert \psi'_{0:s-2}(x) \Vert_2+1)
\end{eqnarray*}
where $U$ is the total number of neurons. Expanding the above recursion we get
\begin{eqnarray*}
    \Vert\psi'_{0:s-1}(x) \Vert_2&\leq &(MB\sqrt{W})^{s-1}\Vert x\Vert_2+(MB\sqrt{W}+M_0\sqrt{U})\frac{(MB\sqrt{W})^{s-1}-1}{MB\sqrt{W}-1}\\
    &\leq &(MB\sqrt{W})^{s-1}(\Vert x\Vert_2+MB\sqrt{W}+M_0\sqrt{U}).
\end{eqnarray*}

Finally, we substitute the upper bounds in \eqref{bound from telescoping} to obtain the final bound
\begin{eqnarray*}
&&\lvert h(x;W'_1,b'_1,\ldots,W'_S,b'_S)-h(x;W_1,b_1,\ldots,W_S,b_S)\rvert\\
&\leq& \sum_{s=1}^SM(MB\sqrt{W})^{S-s}((MB\sqrt{W})^{s-1}(\Vert x\Vert_2+MB\sqrt{W}+M_0\sqrt{U})+1)\Vert \big[\Delta W_s,\Delta b_s\big] \Vert_F\\
&\leq& \sum_{s=1}^S(MB\sqrt{W})^{S-s}(MB\sqrt{W})^{s}(\Vert x\Vert_2+MB\sqrt{W}+M_0\sqrt{U})\Vert \big[\Delta W_s,\Delta b_s\big] \Vert_F\\
&\leq&(MB\sqrt{W})^{S}(\Vert x\Vert_2+MB\sqrt{W}+M_0\sqrt{U}) \sum_{s=1}^S\Vert \big[\Delta W_s,\Delta b_s\big] \Vert_F\\
&\leq &(MB\sqrt{W})^{S}(\Vert x\Vert_2+MB\sqrt{W}+M_0\sqrt{U})\cdot \sqrt{S}\Vert \Delta \theta\Vert_2
\end{eqnarray*}
giving rise to the Lipschitz constant.
\end{proof}

\subsection{Proofs for Results in Section \ref{sec:calibration}}
We first introduce several Berry-Esseen theorems in high dimensions that serve as the main tools of the proofs. Let $\{X_{i}=(X_i^{(1)},\ldots,X_i^{(m)}): i=1,\ldots,n$\} be an i.i.d. data set from $\R^m$. Let $\bar{X}^{(j)}=(1/n)\sum_{i=1}^nX_{i}^{(j)}$ be the sample mean of the $j$-th component, and $\hat{\Sigma}$ be the sample covariance matrix formed from the data. We denote by $\hat{Z}:=(\hat{Z}^{(1)},\ldots,\hat{Z}^{(m)})$ an $m$-dimensional multivariate Gaussian with mean zero and covariance $\hat{\Sigma}$. Then under several light-tail conditions:
\begin{assumption}\label{condition1:clt}
$\mathrm{Var}[X_1^{(j)}]>0$ for all $j=1,\ldots,m$ and there exists some constant $D\geq 1$ such that
\begin{align*}
&\mathbb E\big[\exp\Big(\frac{\vert X_1^{(j)}-\mathbb{E}[X_1^{(j)}]\rvert^2}{D^2\mathrm{Var}[X_1^{(j)}]}\Big)\big]\leq 2\text{ for all }j=1,\ldots,m\\
&\mathbb E\big[\Big(\frac{\lvert X_1^{(j)}-\mathbb{E}[X_1^{(j)}]\rvert}{\sqrt{\mathrm{Var}[X_1^{(j)}]}}\Big)^{2+k}\big]\leq D^k\text{ for all }j=1,\ldots,m\text{ and }k=1,2.
\end{align*}
\end{assumption}
\begin{assumption}\label{condition2:clt}
Each $X_1^{(j)}$ is $[0,1]$-valued and $\mathrm{Var}[X_1^{(j)}]\geq \eta$ for all $j=1,\ldots,m$ and some constant $\eta>0$.
\end{assumption}
we have the following Berry Esseen theorems (\citeAPX{chernozhukov2017central}):
\begin{lemma}[Unnormalized supremum, adopted from Theorem EC.9 in \citeAPX{lam2019combating}]\label{clt:unnormalized}
Under Assumption \ref{condition1:clt}, for every $0<\beta<1$ we have
\begin{align*}
\lvert \mathbb P(\sqrt n(\bar{X}^{(j)}-\mathbb{E}[X_1^{(j)}])\leq q_{1-\beta}\text{ for all }j=1,\ldots,m)- (1-\beta)\rvert\leq C\Big(\frac{D^2\log^7(mn)}{n}\Big)^{\frac{1}{6}}
\end{align*}
where $q_{1-\beta}$ is the $1-\beta$ quantile of $\max_{1\leq j\leq m}\hat Z^{(j)}$, i.e.
\begin{align*}
\mathbb P(\hat Z^{(j)}\leq q_{1-\beta}\text{ for all }j=1,\ldots,m\vert \{X_i:i=1,\ldots,n\})=1-\beta
\end{align*}
and $C$ is a universal constant.
\end{lemma}
\begin{lemma}[Normalized supremum, adopted from Theorem EC.10 in \citeAPX{lam2019combating}]\label{clt:normalized}
Let $\hat\sigma_j^2=\hat\Sigma_{j,j}$. Under Assumptions \ref{condition1:clt} and \ref{condition2:clt}, for every $0<\beta<1$ we have
\begin{align*}
&\lvert \mathbb P(\sqrt n(\bar{X}_{(j)}-\mathbb E[X_1^{(j)}])\leq \hat\sigma_j q_{1-\beta}\text{ for all }j=1,\ldots,m)-(1-\beta)\rvert\\
\leq& C\Big(\Big(\frac{D^2\log^7(mn)}{n}\Big)^{\frac{1}{6}}+\frac{\log^2 (mn)}{\sqrt{n\eta}}+m\exp\big(-c\eta D^{2/3}n^{2/3}\big)\Big).
\end{align*}
Here $q_{1-\beta}$ is such that
\begin{align*}
\mathbb P(\hat Z^{(j)}\leq \hat\sigma_j q_{1-\beta}\text{ for all }j=1,\ldots,m\vert \{X_i\}_{i=1}^n)=1-\beta
\end{align*}
and $C,c$ are universal constants.
\end{lemma}

We are now ready to prove Theorems \ref{feasibility:unnormalized validator} and \ref{feasibility:normalized validator_simple}:
\begin{proof}[Proof of Theorem \ref{feasibility:unnormalized validator}]
Define events
\begin{align*}
E_1&=\big\{\hat{\CR}(\PI_j)\geq 1-\alpha_{\min} + \frac{ q'_{1-\beta}}{\sqrt{n_2}}\text{ for some }j=1,\ldots,m\big\}\\
E_2&=\big\{\CR(\PI_j)\geq\hat{\CR}(\PI_j)-  \frac{ q'_{1-\beta}}{\sqrt{n_2}}\; \text{for all $j$ such that}\; \CR(\PI_j)\in(\tilde{\alpha}/2,1-\alpha_{\min})\big\}\\
E_3&=\big\{\hat{\CR}(\PI_j)<\tilde{\alpha} + \frac{ q'_{1-\beta}}{\sqrt{n_2}}\; \text{for all $j$ such that}\; \CR(\PI_j)\leq \tilde{\alpha}/2\big\}.
\end{align*}
We claim that if $E_1\cap E_2\cap E_3$ happens then we must have that $\CR(\PI_{j^*_{1-\alpha_k}})\geq 1-\alpha_k\; \text{for all}\; k=1,\ldots,K$. To explain, for each $k$, $E_1$ entails that the optimization problem in Step 3 of Algorithm \ref{calibration:unnormalized} has at least one feasible solution, and $E_2$ and $E_3$ further imply that every interval with a true coverage level strictly less than $1-\alpha_k$ violates the margin constraint, hence the selected interval must have a true coverage level of at least $1-\alpha_k$. Therefore we have
\begin{eqnarray}
\notag \mathbb P_{\mathcal D_v}(\CR(\PI_{j^*_{1-\alpha_k}})\geq 1-\alpha_k\; \text{for all}\; k=1,\ldots,K)&\geq& \mathbb P_{\mathcal D_v}(E_1\cap E_2\cap E_3)\\
\notag&\geq&1-\mathbb P_{\mathcal D_v}(E_1^c)-\mathbb P_{\mathcal D_v}(E_2^c)-\mathbb P_{\mathcal D_v}(E_3^c)\\
&=&\mathbb P_{\mathcal D_v}(E_2)-\mathbb P_{\mathcal D_v}(E_1^c)-\mathbb P_{\mathcal D_v}(E_3^c).\label{three probability bounds}
\end{eqnarray}

We derive bounds for the three probabilities. Let $\tilde{q}_{1-\beta}$ be the $1-\beta$ quantile of $\max\{Z_j:\CR(\PI_j)\in(\tilde{\alpha}/2,1-\alpha_{\min}),1\leq j\leq m\}$ where $(Z_1,\ldots,Z_m)\sim N_m(0,\hat\Sigma)$. By stochastic dominance it is clear that $\tilde{q}_{1-\beta}\leq q'_{1-\beta}$ almost surely, therefore
\begin{eqnarray*}
\mathbb P_{\mathcal D_v}(E_2)&\geq& \mathbb P_{\mathcal D_v}\big(\CR(\PI_j)\geq\hat{\CR}(\PI_j)-  \frac{ \tilde{q}_{1-\beta}}{\sqrt{n_v}}\; \text{for all $j$ such that}\; \CR(\PI_j)\in(\tilde{\alpha}/2,1-\alpha_{\min})\big)\\
&\geq&1-\beta-C\Big(\frac{\log^7(mn_v)}{n_v\tilde{\alpha}}\Big)^{\frac{1}{6}}
\end{eqnarray*}
by applying Lemma \ref{clt:unnormalized} to $\{I_{Y\in \PI_j(X)}:\CR(\PI_j)\in(\tilde{\alpha}/2,1-\alpha_{\min}),1\leq j\leq m\}$ and noticing that Assumption \ref{condition1:clt} is satisfied with $D=\frac{C}{\sqrt{\tilde{\alpha}}}$ for some universal constant $C$.

We then bound the second probability
\begin{eqnarray*}
\mathbb P_{\mathcal D_v}(E_1^c)&=&\mathbb P_{\mathcal D_v}(\hat{\CR}(\PI_j)< 1-\alpha_{\min} + \frac{ q'_{1-\beta}}{\sqrt{n_v}}\text{ for all }j=1,\ldots,m)\\
&\leq&\mathbb P_{\mathcal D_v}(\hat{\CR}(\PI_{\bar j})< 1-\alpha_{\min} + \frac{ q'_{1-\beta}}{\sqrt{n_v}})\text{\ \ where $\bar{j}$ is the index such that }\CR(\PI_{\bar j})=1-\underline{\alpha}\\
&\leq &\mathbb P_{\mathcal D_v}(\hat{\CR}(\PI_{\bar j})< 1-\alpha_{\min} + \frac{C\sqrt{\log(m/\beta)}}{\sqrt{n_v}})\\
&&\text{\ \ because }q'_{1-\beta}\leq C\max_{j}\hat{\sigma}_j\sqrt{\log(m/\beta)}\leq C\sqrt{\log(m/\beta)}\\
&\leq&\exp\big(-\frac{n_v\epsilon^2}{2(\underline{\alpha}(1-\underline{\alpha})+\epsilon/3)}\big)
\end{eqnarray*}
where in the last line we use Bennett's inequality (e.g., equation (2.10) in \citeAPX{boucheron2013concentration}). Note that this is further bounded by $\exp\big(-C_2n_v\min\{\epsilon,\frac{\epsilon^2}{\underline{\alpha}(1-\underline{\alpha})}\}\big)$ with another universal constant $C_2$.

The third probability can be bounded as
\begin{eqnarray*}
\mathbb P_{\mathcal D_v}(E_3^c)&\leq&\mathbb P_{\mathcal D_v}\big(\hat{\CR}(\PI_j)\geq \tilde{\alpha}\; \text{for some $j$ such that}\; \CR(\PI_j)\leq \tilde{\alpha}/2\big)\\
&\leq&\sum_{j:\CR(\PI_j)\leq \tilde{\alpha}/2}\mathbb P_{\mathcal D_v}(\hat{\CR}(\PI_j)\geq \tilde{\alpha})\\
&\leq&\sum_{j:\CR(\PI_j)\leq \tilde{\alpha}/2}\exp\big(-\frac{n_v(\tilde{\alpha}/2)^2}{2(\CR(\PI_j)(1-\CR(\PI_j))+\tilde{\alpha}/6)}\big)\;\text{by Bennett's inequality}\\
&\leq& m\exp\big(-\frac{n_v(\tilde{\alpha}/2)^2}{\tilde{\alpha}(1-\tilde{\alpha}/2)+\tilde{\alpha}/3}\big)\leq m\exp(-C_3n_v\tilde{\alpha})
\end{eqnarray*}
where $C_3$ is a universal constant. Substituting the bounds into \eqref{three probability bounds} leads to the overall probability bound
\begin{equation*}
1-\beta-C\Big(\frac{\log^7(mn_v)}{n_v\tilde{\alpha}}\Big)^{\frac{1}{6}}-\exp\big(-C_2n_v\min\{\epsilon,\frac{\epsilon^2}{\underline{\alpha}(1-\underline{\alpha})}\}\big)-m\exp(-C_3n_v\tilde{\alpha}).
\end{equation*}

It remains to show that $m\exp(-C_3n_v\tilde{\alpha})$ is negligible relative to other error terms. Since $\tilde{\alpha}<1$ it is clear that $\big(\frac{1}{n_v}\big)^{1/6}\leq \big(\frac{\log^7(mn_v)}{n_v\tilde{\alpha}}\big)^{1/6}$, and we argue that $\big(\frac{1}{n_v}\big)^{1/6}\geq m\exp(-C_3n_v\tilde{\alpha})$ can be assumed so that $m\exp(-C_3n_v\tilde{\alpha})\leq \big(\frac{\log^7(mn_v)}{n_v\tilde{\alpha}}\big)^{1/6}$. If $\big(\frac{1}{n_v}\big)^{1/6}< m\exp(-C_3n_v\tilde{\alpha})$, then $m>\exp(C_3n_v\tilde{\alpha})n_v^{-1/6}$, hence $\frac{\log^7(mn_v)}{n_v\tilde{\alpha}}\geq \frac{(C_3n_v\tilde{\alpha})^7}{n_v\tilde{\alpha}}\geq C_3^7(n_v\tilde{\alpha})^6$, which ultimately leads to $n_v\tilde{\alpha}\leq \frac{\log(mn_v)}{C_3}$ and $\frac{\log^7(mn_v)}{n_v\tilde{\alpha}}\geq C_3\log^6(mn_v)$. Note that in this case the first error term already exceeds $1$ (by enlarging the universal constant $C$ if necessary) and the error bound holds true trivially.
\end{proof}

\begin{proof}[Proof of Theorem \ref{feasibility:normalized validator_simple}]
The proof follows the one for Theorem \ref{feasibility:unnormalized validator}, and we focus on the modifications. The events are now defined as
\begin{align*}
E_1&=\big\{\hat{\CR}(\PI_j)\geq 1-\alpha_{\min} + \frac{ q_{1-\beta}\hat\sigma_j}{\sqrt{n_v}}\text{ for some }j=1,\ldots,m\big\}\\
E_2&=\big\{\CR(\PI_j)\geq\hat{\CR}(\PI_j)-  \frac{ q_{1-\beta}\hat\sigma_j}{\sqrt{n_v}}\text{ for all $j$ such that }\CR(\PI_j)\in(\tilde{\alpha}/2,1-\alpha_{\min})\big\}\\
E_3&=\big\{\hat{\CR}(\PI_j)<\tilde{\alpha} + \frac{ q_{1-\beta}\hat\sigma_j}{\sqrt{n_v}}\text{ for all $j$ such that }\CR(\PI_j)\leq \tilde{\alpha}/2\big\}.
\end{align*}
Again we have $\mathbb P_{\mathcal D_v}(\CR(\PI_{j^*_{1-\alpha_k}})\geq 1-\alpha_k\; \text{for all}\; k=1,\ldots,K)\geq \mathbb P_{\mathcal D_v}(E_2)-\mathbb P_{\mathcal D_v}(E_1^c)-\mathbb P_{\mathcal D_v}(E_3^c)$.

The first probability bound becomes
\begin{eqnarray*}
\mathbb P_{\mathcal D_v}(E_2)\geq 1-\beta-C\Big(\Big(\frac{\log^7(mn_v)}{n_v\tilde{\alpha}}\Big)^{\frac{1}{6}}+\frac{\log^2 (mn_v)}{\sqrt{n_v\tilde{\alpha}}}+m\exp\big(-c(n_v\tilde{\alpha})^{2/3}\big)\Big)
\end{eqnarray*}
by applying Lemma \ref{clt:normalized} and noting that Assumption \ref{condition2:clt} holds with $\eta=\tilde{\alpha}/2\cdot (1-\tilde{\alpha}/2)\geq \frac{1}{4}\tilde{\alpha}$ and $D=\frac{C}{\sqrt{\tilde{\alpha}}}$ in Assumption \ref{condition1:clt}. Here $C,c$ are universal constants.

For the second probability we have
\begin{eqnarray}
\notag\mathbb P_{\mathcal D_v}(E_1^c)&\leq&\mathbb P_{\mathcal D_v}(\hat{\CR}(\PI_{\bar j})< 1-\alpha_{\min} + \frac{ q_{1-\beta}\hat\sigma_{\bar{j}}}{\sqrt{n_v}})\text{\ \ where $\bar{j}$ is the index such that }\CR(\PI_{\bar j})=1-\underline{\alpha}\\
\notag&\leq&\mathbb P_{\mathcal D_v}(\hat{\CR}(\PI_{\bar j})< 1-\alpha_{\min} + \frac{ q_{1-\beta}t}{\sqrt{n_v}})+\mathbb P_{\mathcal D_v}(\hat\sigma_{\bar{j}}>t)\\
\notag&&\text{\ \  where }t=\sqrt{\underline{\alpha}(1-\underline{\alpha})}+\sqrt{2\log(n_v\alpha_{\min})/n_v}\\
\notag&\leq &\mathbb P_{\mathcal D_v}(\hat{\CR}(\PI_{\bar j})< 1-\alpha_{\min} + \frac{ q_{1-\beta}t}{\sqrt{n_v}})+\frac{1}{n_v\alpha_{\min}}\\
\notag&&\text{\ where the bound $1/(n_v\alpha_{\min})$ follows from Theorem 10 in \citeAPX{maurer2009empirical}}\\
\notag&\leq &\mathbb P_{\mathcal D_v}(\hat{\CR}(\PI_{\bar j})< 1-\alpha_{\min} + \frac{C\sqrt{(\underline{\alpha}(1-\underline{\alpha})+\log(n_v\alpha_{\min})/n_v)\log(m/\beta)}}{\sqrt{n_v}})+\frac{1}{n_v\alpha_{\min}}\\
&&\text{\ \ because }q_{1-\beta}\leq C\sqrt{\log(m/\beta)}\label{normal quantile bound}\\
\notag&\leq&\exp\big(-\frac{n_v\epsilon^2}{2(\underline{\alpha}(1-\underline{\alpha})+\epsilon/3)}\big)+\frac{1}{n_v\alpha_{\min}}\text{\ \ by Bennett's inequality}\\
\notag&\leq&\exp\big(-C_2n_v\min\{\epsilon,\frac{\epsilon^2}{\underline{\alpha}(1-\underline{\alpha})}\}\big)+\frac{1}{n_v\alpha_{\min}}.
\end{eqnarray}

As for the third probability, by repeating the same analysis in Theorem \ref{feasibility:unnormalized validator} we see that the bound $\mathbb P_{\mathcal D_v}(E_3^c)\leq m\exp(-C_3n_v\tilde{\alpha})$ remains valid.

Finally, using a similar argument in the proof of Theorem \ref{feasibility:unnormalized validator}, we can show that $\frac{1}{n_v\alpha_{\min}}$, $m\exp(-C_3n_v\tilde{\alpha})$, and $m\exp(-c(n_v\tilde{\alpha})^{2/3})$ are all dominated by $\big(\frac{\log^7(mn_v)}{n_v\tilde{\alpha}}\big)^{1/6}$ when $\big(\frac{\log^7(mn_v)}{n_v\tilde{\alpha}}\big)^{1/6}<1$. Moreover, $\frac{\log^2 (mn_v)}{\sqrt{n_v\tilde{\alpha}}}$ can also be neglected, because $\frac{\log^2 (mn_v)}{\sqrt{n_v\tilde{\alpha}}}= \big(\frac{\log^{5/2}(mn_v)}{n_v\tilde{\alpha}}\big)^{\frac{1}{3}} \big(\frac{\log^7(mn_v)}{n_v\tilde{\alpha}}\big)^{\frac{1}{6}}\leq \big(\frac{\log^7(mn_v)}{n_v\tilde{\alpha}}\big)^{\frac{1}{6}}$ when $\big(\frac{\log^7(mn_v)}{n_v\tilde{\alpha}}\big)^{\frac{1}{6}}<1$. Therefore the desired conclusion follows from combining the three probability bounds.
\end{proof}

\begin{proof}[Proof of Theorem \ref{optimality guarantee: normalized}]
The proof consists of deriving concentration bounds for the empirical width and coverage rate. We first deal with the empirical width. Using standard concentration bounds for sub-Gaussian variables, we write for every interval $\mathrm{PI}_j=[L_j,U_j]$ and every $\epsilon>0$
\begin{eqnarray*}
&&\mathbb P_{\mathcal D_v}\big(\lvert \frac{1}{n_v}\sum_{i=1}^{n_v}(U_j(X_i') - L_j(X_i')) - \mathbb E_{\pi_X}[U_j(X) - L_j(X)]\rvert >\epsilon\Vert H\Vert_{\psi_2}\big)\\
&\leq& \mathbb P_{\mathcal D_v}\big(\lvert \frac{1}{n_v}\sum_{i=1}^{n_v}U_j(X_i') - \mathbb E_{\pi_X}[U_j(X)]\rvert >\frac{\epsilon\Vert H\Vert_{\psi_2}}{2}\big) + \mathbb P_{\mathcal D_v}\big(\lvert \frac{1}{n_v}\sum_{i=1}^{n_v}L_j(X_i') - \mathbb E_{\pi_X}[L_j(X)]\rvert >\frac{\epsilon\Vert H\Vert_{\psi_2}}{2}\big)\\
&&\text{ \ \ by the union bound}\\
&\leq&2\exp\Big(-\frac{\epsilon^2\Vert H\Vert_{\psi_2}^2n_v}{4C^2\Vert U_j - \mathbb E_{\pi_X}[U_j] \Vert_{\psi_2}^2}\Big) + 2\exp\Big(-\frac{\epsilon^2\Vert H\Vert_{\psi_2}^2n_v}{4C^2\Vert L_j - \mathbb E_{\pi_X}[L_j] \Vert_{\psi_2}^2}\Big)\\
&&\text{ \ \ for some universal constant }C\\
&\leq&4\exp\Big(-\frac{\epsilon^2\Vert H\Vert_{\psi_2}^2n_v}{4C^2\Vert H \Vert_{\psi_2}^2}\Big) \text{ \ \ since }\lvert L_j - \mathbb E_{\pi_X}[L_j]\rvert, \lvert U_j - \mathbb E_{\pi_X}[U_j]\rvert\leq H\\
&=&4\exp\Big(-\frac{\epsilon^2n_v}{4C^2}\Big).
\end{eqnarray*}
Applying the union bound to all the $m$ candidate PIs, we have
\begin{equation*}
    \mathbb P_{\mathcal D_v}\big(\lvert \frac{1}{n_v}\sum_{i=1}^{n_v}(U_j(X_i') - L_j(X_i')) - \mathbb E_{\pi_X}[U_j(X) - L_j(X)]\rvert >\epsilon\Vert H\Vert_{\psi_2} \text{ for some }j=1,\ldots,m\big)\leq 4m\exp\Big(-\frac{\epsilon^2n_v}{4C^2}\Big)
\end{equation*}
or equivalently
\begin{equation}\label{width concentration}
    \mathbb P_{\mathcal D_v}\big(\lvert \frac{1}{n_v}\sum_{i=1}^{n_v}(U_j(X_i') - L_j(X_i')) - \mathbb E_{\pi_X}[U_j(X) - L_j(X)]\rvert >C\epsilon\Vert H\Vert_{\psi_2} \text{ for some }j=1,\ldots,m\big)\leq 4m\exp\Big(-\frac{\epsilon^2n_v}{4}\Big).
\end{equation}
Next we handle the empirical coverage rate
\begin{eqnarray*}
&&\mathbb P_{\mathcal D_v}\big(\hat{\mathrm{CR}}(\mathrm{PI}_j) - \mathrm{CR}(\mathrm{PI}_j) < -\epsilon + \frac{q_{1-\beta}\hat\sigma_j}{\sqrt{n_v}} \big)\\
&\leq &\mathbb P_{\mathcal D_v}\big(\hat{\mathrm{CR}}(\mathrm{PI}_j) - \mathrm{CR}(\mathrm{PI}_j) < -\epsilon + \frac{C\sqrt{\log(m/\beta)}}{\sqrt{n_v}} \big)\text{ \ \ by \eqref{normal quantile bound} and the fact that $\hat\sigma_j\leq \frac{1}{2}$}\\
&&\text{\ \ where $C$ is another universal constant}\\
&\leq& \exp \Big( -2\max\big\{\epsilon - C\sqrt{\frac{\log (m/\beta)}{n_v}}, 0\big\}^2n_v\Big)\text{ \ \ by Hoeffding's inequality.}
\end{eqnarray*}
Again applying the union bound we get for every $\epsilon>0$
\begin{equation}\label{coverage concentration}
    \mathbb P_{\mathcal D_v}\big(\hat{\mathrm{CR}}(\mathrm{PI}_j) - \mathrm{CR}(\mathrm{PI}_j) < -\epsilon + \frac{q_{1-\beta}\hat\sigma_j}{\sqrt{n_v}} \text{ for some }j=1,\ldots,m \big)\leq m\exp \Big( -2\max\big\{\epsilon - C\sqrt{\frac{\log (m/\beta)}{n_v}}, 0\big\}^2n_v\Big).
\end{equation}
Note that by choosing both universal constants in \eqref{width concentration} and \eqref{coverage concentration} large enough, we can use the same universal constant $C$ in both. To relate to the width performance, we observe that when $\hat{\mathrm{CR}}(\mathrm{PI}_j) - \mathrm{CR}(\mathrm{PI}_j) \geq -\epsilon + \frac{q_{1-\beta}\hat\sigma_j}{\sqrt{n_v}}$ for all $j=1,\ldots,m$, we have that for all $\mathrm{PI}_j$ whose true coverage rate $\mathrm{CR}(\mathrm{PI}_j)\geq 1-\alpha_k+\epsilon$ the inequality $\hat{\mathrm{CR}}(\mathrm{PI}_j) \geq \mathrm{CR}(\mathrm{PI}_j) -\epsilon + \frac{q_{1-\beta}\hat\sigma_j}{\sqrt{n_v}} \geq 1-\alpha_k+ \frac{q_{1-\beta}\hat\sigma_j}{\sqrt{n_v}}$ holds (i.e., the constraint in Step 3 of Algorithm \ref{calibration:normalized} is satisfied), therefore by the optimality of each $\mathrm{PI}_{j^*_{1-\alpha_k}}$ it must hold that
\begin{equation*}
    \frac{1}{n_v}\sum_{i=1}^{n_v}\lvert\mathrm{PI}_{j^*_{1-\alpha_k}}(X_i') \rvert \leq \min_{j:\mathrm{CR(\mathrm{PI}_j)\geq 1-\alpha_k+\epsilon}}\frac{1}{n_v}\sum_{i=1}^{n_v}\lvert\mathrm{PI}_j(X_i')\rvert\text{ \ \ for each $k=1,\ldots, K$}.
\end{equation*}
If we further have $\lvert \frac{1}{n_v}\sum_{i=1}^{n_v}(U_j(X_i') - L_j(X_i')) - \mathbb E_{\pi_X}[U_j(X) - L_j(X)]\rvert \leq C\epsilon\Vert H\Vert_{\psi_2}$ for all $j=1,\ldots,m$, then
\begin{eqnarray*}
\mathbb E_{\pi_X}[U_{j^*_{1-\alpha_k}}(X) - L_{j^*_{1-\alpha_k}}(X)] &\leq &\frac{1}{n_v}\sum_{i=1}^{n_v}\lvert\mathrm{PI}_{j^*_{1-\alpha_k}}(X_i')\rvert + C\epsilon\Vert H\Vert_{\psi_2}\\
&\leq& \min_{j:\mathrm{CR(\mathrm{PI}_j)\geq 1-\alpha_k+\epsilon}}\mathbb E_{\pi_X}[U_{j}(X) - L_{j}(X)] + 2C\epsilon\Vert H\Vert_{\psi_2}
\end{eqnarray*}
for every $k=1,\ldots,K$. Altogether we can conclude that
\begin{eqnarray*}
&&\mathbb P_{\mathcal D_v}\Big(\mathbb E_{\pi_X}[U_{j^*_{1-\alpha_k}}(X) - L_{j^*_{1-\alpha_k}}(X)]\leq \min_{j:\mathrm{CR(\mathrm{PI}_j)\geq 1-\alpha_k+\epsilon}}\mathbb E_{\pi_X}[U_{j}(X) - L_{j}(X)]+2C\epsilon\Vert H\Vert_{\psi_2}\text{ for all }k=1,\ldots,K\Big)\\
&\geq& \mathbb P_{\mathcal D_v}\Big(\lvert \frac{1}{n_v}\sum_{i=1}^{n_v}(U_j(X_i') - L_j(X_i')) - \mathbb E_{\pi_X}[U_j(X) - L_j(X)]\rvert \leq C\epsilon\Vert H\Vert_{\psi_2} \text{ for all } j=1,\ldots,m, \text{ and}\\
&&\hspace{2em} \hat{\mathrm{CR}}(\mathrm{PI}_j) - \mathrm{CR}(\mathrm{PI}_j) \geq -\epsilon + \frac{q_{1-\beta}\hat\sigma_j}{\sqrt{n_v}} \text{ for all } j=1,\ldots,m\Big)\\
&\geq& 1- 4m\exp\Big(-\frac{\epsilon^2n_v}{4}\Big) - m\exp \Big( -2\max\big\{\epsilon - C\sqrt{\frac{\log (m/\beta)}{n_v}}, 0\big\}^2n_v\Big)\text{ \ \ by \eqref{width concentration} and \eqref{coverage concentration}}\\
&\geq&1-8m\exp \Big( -\frac{1}{4}\max\big\{\epsilon - C\sqrt{\frac{\log (m/\beta)}{n_v}}, 0\big\}^2n_v\Big)
\end{eqnarray*}
where the last inequality holds because $\epsilon\geq \max\big\{\epsilon - C\sqrt{\frac{\log (m/\beta)}{n_v}}, 0\big\}$.
\end{proof}

\subsection{Proofs for Appendix \ref{sec:consistency}}
We provide proofs for Theorems \ref{basic consistency} and \ref{consistency:strong GC H}:
\begin{proof}[Proof of Theorem \ref{basic consistency}]
We first construct the sequence $t_n$ as follows. For each integer $k>0$, weak $\pi$-GC implies that $\mathbb{P}\big(\sup_{L,U\in\mathcal{H},L\leq U}\lvert \mathbb{P}_{\hat\pi}(Y\in [L(X),U(X)]) - \mathbb{P}_{\pi}(Y\in [L(X),U(X)] \rvert > \frac{1}{k}\big)\to 0$ as the data size $n\to\infty$. Therefore, there exists a large enough $n_k$ such that $\mathbb{P}\big(\sup_{L,U\in\mathcal{H},L\leq U}\lvert \mathbb{P}_{\hat\pi}(Y\in [L(X),U(X)]) - \mathbb{P}_{\pi}(Y\in [L(X),U(X)] \rvert > \frac{1}{k}\big)<\frac{1}{k}$ whenever $n\geq n_k$. Moreover, the $n_k$ can be chosen such that $n_k < n_{k+1}$ for each $k$. We then let $t_n = \min\{\frac{1}{k}:n\geq n_k\}$. Clearly $t_n\to 0$ as $n\to\infty$, and, by construction, we have $\mathbb{P}\big(\sup_{L,U\in\mathcal{H},L\leq U}\lvert \mathbb{P}_{\hat\pi}(Y\in [L(X),U(X)]) - \mathbb{P}_{\pi}(Y\in [L(X),U(X)] \rvert > t_n\big)<t_n$.

Then we show joint optimality and feasibility for the chosen $t_n$. Denote by $\hat{\mathcal{H}}^2_t$ the feasible set of the optimization \eqref{OP2}, and by $\mathcal{H}^2_t$ the feasible set of \eqref{OP3}. In particular $\mathcal{H}^2_0$ is the feasible set of \eqref{OP1}. By the construction of $t_n$, we have $\mathbb{P}(\mathcal{H}^2_{2t_n}\subset\hat{\mathcal{H}}^2_{t_n}\subset \mathcal{H}^2_0)> 1-t_n$. Therefore $\mathbb{P}((\hat{L}^*_{t_n},\hat{U}^*_{t_n})\in \mathcal{H}^2_0)\geq \mathbb{P}(\hat{\mathcal{H}}^2_{t_n}\in \mathcal{H}^2_0)\geq 1-t_n\to 0$, concluding the asymptotic feasibility of $(\hat{L}^*_{t_n},\hat{U}^*_{t_n})$. To show optimality, when the events
\begin{equation*}
    W_{\epsilon}:=\big\{\sup_{h\in \mathcal{H}}\lvert \mathbb{E}_{\hat\pi_X}[h(X)]-\mathbb{E}_{\pi_X}[h(X)] \rvert\leq \epsilon\big\}
\end{equation*}
and
\begin{equation*}
    C_{t_n}:=\big\{\sup_{L,U\in \mathcal{H}\;\text{and}\;L\leq U}\lvert \mathbb{P}_{\hat\pi_X}(Y\in [L(X),U(X)])-\mathbb{P}_{\pi_X}(Y\in [L(X),U(X)]) \rvert\leq t_n\big\}
\end{equation*}
occur, it holds that $\mathcal{H}^2_{2t_n}\subset \hat{\mathcal{H}}^2_{t_n}\subset \mathcal{H}^2_0$, and $(\hat L_{t_n}^*,\hat U_{t_n}^*)\in\hat{\mathcal{H}}^2_{t_n}$ is feasible for \eqref{OP1}. We also have
\begin{eqnarray}
&&\notag\mathbb{E}_{\pi_X}[\hat U_{t_n}^*(X)-\hat L_{t_n}^*(X)]\\
\notag &\leq& \mathbb{E}_{\hat\pi_X}[\hat U_{t_n}^*(X)-\hat L_{t_n}^*(X)] + 2\epsilon\text{\ \ because of $W_{\epsilon}$}\\
\notag&\leq& \inf_{(L,U)\in\mathcal{H}^2_{2{t_n}}\subset \hat{\mathcal{H}}^2_{t_n}}\mathbb{E}_{\hat\pi_X}[U(X)-L(X)]+2\epsilon\text{\ \ by optimality of $(\hat L_{t_n}^*,\hat U_{t_n}^*)$ in $\hat{\mathcal{H}}^2_{t_n}$}\\
\notag&\leq& \inf_{(L,U)\in\mathcal{H}^2_{2{t_n}}}\mathbb{E}_{\pi_X}[U(X)-L(X)]+4\epsilon\text{\ \ because of $W_{\epsilon}$}\\
\notag&=&\mathcal{R}_{2{t_n}}^*(\mathcal{H})+4\epsilon\\
&\leq &\mathcal{R}^*(\mathcal{H})+\frac{12t_n}{(\alpha-2t_n)\gamma_{\frac{\alpha-2t_n}{3}}}+4\epsilon\text{\ \ by Theorem \ref{sensitivity bound}}. \label{optimality bound:consistency}
\end{eqnarray}
Note that $\mathbb{P}(W_{\epsilon}\cap C_{t_n})\geq 1-\mathbb{P}(W_{\epsilon}^c)-\mathbb{P}(C_{t_n}^c)\geq 1-\mathbb{P}(W_{\epsilon}^c) - t_n$. 
Note that $\mathcal{H}$ being $\pi$-GC implies that $\mathbb{P}(W_{\epsilon}^c)\to 0$ as $n\to \infty$ for any $\epsilon>0$, and that the term involving $t_n$ in \eqref{optimality bound:consistency} goes to zero as $t_n\to 0$.
On the other hand, $\mathbb{E}_{\pi_X}[\hat U_{t_n}^*(X)-\hat L_{t_n}^*(X)]\geq \mathcal{R}^*(\mathcal{H})$ when $(\hat L_{t_n}^*,\hat U_{t_n}^*)\in \mathcal{H}^2_0$  by the definition of $\mathcal{R}^*(\mathcal{H})$. Since $\mathbb{P}(\mathcal{H}_{2t_n}^2\subset\hat{\mathcal{H}}_{t_n}^2\text{ and }(\hat L_{t_n}^*,\hat U_{t_n}^*)\in \mathcal{H}_0^2)\to 1$, the derived lower and upper bounds for $\mathbb{E}_{\pi_X}[\hat U_{t_n}^*(X)-\hat L_{t_n}^*(X)]$ entails that $\mathbb{E}_{\pi_X}[\hat U_{t_n}^*(X)-\hat L_{t_n}^*(X)]\to \mathcal{R}^*(\mathcal{H})$ in probability, concluding optimality.

\end{proof}

\begin{proof}[Proof of Theorem \ref{consistency:strong GC H}]
It suffices to show that the induced set class is strong $\pi$-GC. Consistency then follows from Theorem \ref{basic consistency} and the fact that strong GC implies weak GC. We will need the following preservation results for GC classes:
\begin{lemma}[Adapted from Theorem 9.26 in \citeAPX{kosorok2007introduction}]\label{GC preservation}
Suppose that $\mathcal{G}_1,\ldots,\mathcal{G}_K$ are strong $\pi$-GC classes of functions with $\max_{1\leq k\leq K}\sup_{g\in \mathcal{G}_k}\lvert\mathbb{E}_{\pi}[g(X,Y)]\rvert<\infty$, and that $\psi:\R^K \to \R$ is continuous. Then the class $\psi(\mathcal{G}_1,\ldots,\mathcal{G}_K):=\{\psi(g_1(\cdot), \ldots, g_K(\cdot)):g_k\in\mathcal{G}_k\text{ for }k=1,\ldots,K\}$ is strong $\pi$-GC provided that $\mathbb{E}_{\pi}[\sup_{g\in \psi(\mathcal{G}_1,\ldots,\mathcal{G}_K)}\lvert g(X,Y) \rvert]<\infty$.
\end{lemma}

We first use Lemma \ref{GC preservation} to simplify the problem. Denote by $\mathcal{G}:=\{(x,y)\to I_{L(x)\leq y\leq U(x)}:L,U\in\mathcal{H},L\leq U\}$ the target indicator class for which we want to show strong $\pi$-GC. Then $\mathcal{G}\subset\psi(\mathcal{G}_l,\mathcal{G}_u)$, where $\psi(z_1,z_2):=z_1z_2$, and
\begin{align*}
    \mathcal{G}_l &:= \{(x,y)\to I_{L(x)-y\leq 0}:L\in\mathcal{H}\}\\
    \mathcal{G}_u &:= \{(x,y)\to I_{y-U(x)\leq 0}:U\in\mathcal{H}\}.
\end{align*}
Note that functions in the class $\psi(\mathcal{G}_l,\mathcal{G}_u)$ are all bounded by $1$ and $\psi(\cdot,\cdot)$ is obviously continuous, therefore by Lemma \ref{GC preservation}, $\psi(\mathcal{G}_l,\mathcal{G}_u)$ (hence $\mathcal{G}$) is strong $\pi$-GC if both $\mathcal{G}_l$ and $\mathcal{G}_u$ are strong $\pi$-GC. So it suffices to show strong $\pi$-GC for $\mathcal{G}_l$ and $\mathcal{G}_u$. In the following analysis, we only show $\pi$-GC for $\mathcal{G}_u$, because the $\mathcal{G}_l$ case can be shown by the same argument.

In order to prove strong $\pi$-GC for $\mathcal{G}_u$, we first demonstrate that, without loss of generality, we can assume that the class $\mathcal{H}$ has a upper bound for $\lvert \mathbb{E}_{\pi_X}[h(X)]\rvert$, say $\lvert\mathbb{E}_{\pi_X}[h(X)]\rvert\leq M$ (so that Lemma \ref{GC preservation} can be used to propagate the GC property from $\mathcal{H}$ to $\mathcal{G}_u$). To proceed, we write for any constant $M>0$
\begin{eqnarray}
\notag&&\sup_{U \in\mathcal{H}}\lvert \mathbb{P}_{\hat\pi}(Y-U(X)\leq 0)-\mathbb{P}_{\pi}(Y-U(X)\leq 0)\rvert\\
&\leq&\sup_{U \in\mathcal{H},\lvert \mathbb{E}_{\pi_X}[U(X)]\rvert \leq M}\lvert \mathbb{P}_{\hat\pi}(Y-U(X)\leq 0)-\mathbb{P}_{\pi}(Y-U(X)\leq 0)\rvert+\label{tail term0}\\
&&\sup_{U \in\mathcal{H},\mathbb{E}_{\pi_X}[U(X)]> M}\lvert \mathbb{P}_{\hat\pi}(Y-U(X)\leq 0)-\mathbb{P}_{\pi}(Y-U(X)\leq 0)\rvert+\label{tail term}\\
&&\sup_{U \in\mathcal{H},\mathbb{E}_{\pi_X}[U(X)]<-M}\lvert \mathbb{P}_{\hat\pi}(Y-U(X)\leq 0)-\mathbb{P}_{\pi}(Y-U(X)\leq 0)\rvert.\label{tail term2}
\end{eqnarray}
We show how we control the second and third suprema in \eqref{tail term} and \eqref{tail term2}. Since the class $\mathcal{H}$ is strong $\pi_X$-GC, the centered function class $\{h(\cdot)-\mathbb{E}_{\pi_X}[h(X)]:h\in\mathcal{H}\}$ must have an integrable envelope (see, e.g., Lemma 8.13 in \citeAPX{kosorok2007introduction}), i.e., $F(x):=\sup_{h(x) \in\mathcal{H}}\lvert h-\mathbb{E}_{\pi_X}[h(X)] \rvert$ satisfies $\mathbb{E}_{\pi_X}[F(X)]<\infty$. Therefore when $\mathbb{E}_{\pi_X}[U(X)]> M$ we can bound $Y-U(X)$ almost surely as
\begin{equation*}
    Y-U(X)< Y- U(X)+\mathbb{E}_{\pi_X}[U(X)]-M\leq Y+F(X)-M.
\end{equation*}
Similarly, when $\mathbb{E}_{\pi_X}[U(X)]<-M$ we have
\begin{equation*}
    Y-U(X)> Y- U(X)+\mathbb{E}_{\pi_X}[U(X)]+M\geq Y-F(X)+M.
\end{equation*}
With these bounds for $Y-U(X)$, the supremum from \eqref{tail term} can be further bounded as
\begin{eqnarray}
\notag&&\sup_{U \in\mathcal{H},\mathbb{E}_{\pi_X}[U(X)]> M}\lvert \mathbb{P}_{\hat\pi}(Y-U(X)\leq 0)-\mathbb{P}_{\pi}(Y-U(X)\leq 0)\rvert\\
\notag&\leq& \sup_{U \in\mathcal{H},\mathbb{E}_{\pi_X}[U(X)]> M}\lvert \mathbb{P}_{\hat\pi}(Y-U(X)\leq 0)-1\rvert +\\
\notag&&\hspace{2em}\sup_{U \in\mathcal{H},\mathbb{E}_{\pi_X}[U(X)]> M}\lvert \mathbb{P}_{\pi}(Y-U(X)\leq 0)-1\rvert\text{\ \ by triangle inequality}\\
\notag&\leq& \lvert \mathbb{P}_{\hat\pi}(Y+F(X)\leq M)-1\rvert +\lvert \mathbb{P}_{\pi}(Y+F(X)\leq M)-1\rvert\\
\notag&\leq& \lvert \mathbb{P}_{\hat\pi}(Y+F(X)\leq M)-\mathbb{P}_{\pi}(Y+F(X)\leq M)\rvert+\\
&&\hspace{2em}2\lvert \mathbb{P}_{\pi}(Y+F(X)\leq M)-1\rvert\text{\ \ again by triangle inequality}\label{bound for tail term}
\end{eqnarray}
and \eqref{tail term2} can be similarly bounded as
\begin{eqnarray}
\notag&&\sup_{U \in\mathcal{H},\mathbb{E}_{\pi_X}[U(X)]<-M}\lvert \mathbb{P}_{\hat\pi}(Y-U(X)\leq 0)-\mathbb{P}_{\pi}(Y-U(X)\leq 0)\rvert\\
\notag&\leq&\sup_{U \in\mathcal{H},\mathbb{E}_{\pi_X}[U(X)]<-M}\lvert \mathbb{P}_{\hat\pi}(Y-U(X)\leq 0)\rvert+\sup_{U \in\mathcal{H},\mathbb{E}_{\pi_X}[U(X)]<-M}\lvert \mathbb{P}_{\pi}(Y-U(X)\leq 0)\rvert\\
\notag&\leq& \lvert \mathbb{P}_{\hat\pi}(Y-F(X)\leq -M)\rvert +\lvert \mathbb{P}_{\pi}(Y-F(X)\leq -M)\rvert\\
\notag&\leq& \lvert \mathbb{P}_{\hat\pi}(Y-F(X)\leq -M)-\mathbb{P}_{\pi}(Y-F(X)\leq -M)\rvert+\\
&&\hspace{2em}2\lvert \mathbb{P}_{\pi}(Y-F(X)\leq -M)\rvert\text{\ \ by triangle inequality}\label{bound for tail term2}
\end{eqnarray}
Note that $\mathbb{P}_{\pi}(Y+F(X)\leq M)\to 1$ as $M\to\infty$ therefore the second absolute value in \eqref{bound for tail term} can be made arbitrarily small by choosing $M$ sufficiently large. The first absolute value in \eqref{bound for tail term} vanishes almost surely for every fixed $M$ by the classical strong law of large numbers. As a result, for every $\epsilon>0$, there exists an $M>0$ such that almost surely there exists an $n'$ for which the second supremum in \eqref{tail term} is less than $\epsilon$ for all sample size $n>n'$. A similar conclusion can be drawn for \eqref{bound for tail term2}. Therefore it's enough to show that for every fixed $M$ the supremum in \eqref{tail term0} vanishes almost surely, or in other words, the following constrained version of $\mathcal{G}_u$
\begin{equation*}
    \mathcal{G}_u^M := \{(x,y)\to I_{y-U(x)\leq 0}:U\in\mathcal{H},\lvert\mathbb{E}_{\pi_X}[U(X)]\rvert\leq M\}
\end{equation*}
is strong $\pi$-GC.

It remains to show that $\mathcal{G}_u^M$ is strong $\pi$-GC. We first note that, since $Y$ is assumed integrable, the class $\mathcal{H}_u^M:=\{(x,y)\to y-U(x):U\in\mathcal{H},\lvert\mathbb{E}_{\pi_X}[U(X)]\rvert\leq M\}$ is strong $\pi$-GC, and has an integrable envelope $\lvert Y\rvert + M + F(X)$ where $F$ is the envelope for $\{h(\cdot )-\mathbb{E}_{\pi_X}[h(X)]:h\in\mathcal H\}$. Our plan is to propagate GC from $\mathcal{H}^M_u$ to $\mathcal{G}_u^M$ using Lemma \ref{GC preservation}. To this end, we define a function $\mathrm{sign}(z)=1$ if $z\leq 0$ and $0$ if $z>0$, then $\mathcal{G}_u^M$ can be written as $\mathrm{sign}(\mathcal{H}^M_u)$. We also define two auxiliary functions both parameterized by $\epsilon>0$
\begin{align*}
\mathrm{sign}_{\epsilon}^l(z)&:=
\begin{cases}
1&\text{if }z\leq -\epsilon\\
-\frac{z}{\epsilon}&\text{if }\epsilon< z<0\\
0&\text{if }z\geq 0
\end{cases}\\
\mathrm{sign}_{\epsilon}^u(z)&:=
\begin{cases}
1&\text{if }z\leq 0\\
1-\frac{z}{\epsilon}&\text{if }0< z<\epsilon\\
0&\text{if }z\geq \epsilon
\end{cases}.
\end{align*}
Since $\mathrm{sign}_{\epsilon}^l\leq \mathrm{sign}\leq \mathrm{sign}_{\epsilon}^u$, we can approximate the class $\mathcal{G}_u^M$ (i.e., $\mathrm{sign}(\mathcal{H}^M_u)$) from below by $\mathrm{sign}_{\epsilon}^l(\mathcal{H}^M_u)$ and from above by $\mathrm{sign}_{\epsilon}^u(\mathcal{H}^M_u)$. Moreover, we have the following approximation bound
\begin{eqnarray}
\notag&&\sup_{h\in \mathcal{H}^M_u}\big(\mathbb{E}_{\pi}[\mathrm{sign}_{\epsilon}^u(h(X,Y))] - \mathbb{E}_{\pi}[\mathrm{sign}_{\epsilon}^l(h(X,Y))]\big)\\
\notag&=&\sup_{h\in \mathcal{H}^M_u}\mathbb{E}_{\pi}[(\mathrm{sign}_{\epsilon}^u-\mathrm{sign}_{\epsilon}^l)(h(X,Y))]\\
\notag&\leq&\sup_{U\in \mathcal{H},\lvert\mathbb{E}_{\pi_X}[U(X)]\rvert\leq M}\mathbb{P}_{\pi}(Y-U(X)\in (-\epsilon,\epsilon))\\
\notag&&\hspace{2em}\text{\ \ because }(\mathrm{sign}_{\epsilon}^u-\mathrm{sign}_{\epsilon}^l)(z)\leq I_{z\in (-\epsilon,\epsilon)}\\
\notag&=&\sup_{U\in \mathcal{H},\lvert\mathbb{E}_{\pi_X}[U(X)]\rvert\leq M}\mathbb{E}_{\pi_X}[\mathbb{P}_{\pi}(Y-U(X)\in (-\epsilon,\epsilon)\vert X)]\\
&\leq &2\epsilon\sup_{x,y}p(y\vert x).\label{difference of expected values}
\end{eqnarray}
On the other hand, both $\mathrm{sign}_{\epsilon}^l$ and $\mathrm{sign}_{\epsilon}^u$ are continuous and bounded, therefore by Lemma \ref{GC preservation} both $\mathrm{sign}_{\epsilon}^l(\mathcal{H}^M_u)$ and $\mathrm{sign}_{\epsilon}^l(\mathcal{H}^M_u)$ are strong $\pi$-GC for each fixed $\epsilon>0$. We can then write
\begin{eqnarray*}
&&\sup_{h\in \mathcal{H}^M_u}\lvert \mathbb{E}_{\hat\pi}[\mathrm{sign}(h(X,Y))] - \mathbb{E}_{\pi}[\mathrm{sign}(h(X,Y))] \rvert\\
&\leq &\sup_{h\in \mathcal{H}^M_u}\big(\lvert \mathbb{E}_{\hat\pi}[\mathrm{sign}_{\epsilon}^l(h(X,Y))] - \mathbb{E}_{\pi}[\mathrm{sign}(h(X,Y))] \rvert + \lvert \mathbb{E}_{\hat\pi}[\mathrm{sign}_{\epsilon}^u(h(X,Y))] - \mathbb{E}_{\pi}[\mathrm{sign}(h(X,Y))] \rvert\big)\\
&&\hspace{2em}\text{because }\mathrm{sign}_{\epsilon}^l\leq \mathrm{sign}\leq \mathrm{sign}_{\epsilon}^u\\
&\leq &\sup_{h\in \mathcal{H}^M_u}\big(\lvert \mathbb{E}_{\hat\pi}[\mathrm{sign}_{\epsilon}^l(h(X,Y))] - \mathbb{E}_{\pi}[\mathrm{sign}_{\epsilon}^l(h(X,Y))] \rvert + \lvert \mathbb{E}_{\hat\pi}[\mathrm{sign}_{\epsilon}^u(h(X,Y))] - \mathbb{E}_{\pi}[\mathrm{sign}_{\epsilon}^u(h(X,Y))] \rvert\\
&&\hspace{2em}+\mathbb{E}_{\pi}[\mathrm{sign}_{\epsilon}^u(h(X,Y))] - \mathbb{E}_{\pi}[\mathrm{sign}_{\epsilon}^l(h(X,Y))]\big)\text{\ \ by triangle inequality}\\
&\leq &\sup_{h\in \mathcal{H}^M_u}\lvert \mathbb{E}_{\hat\pi}[\mathrm{sign}_{\epsilon}^l(h(X,Y))] - \mathbb{E}_{\pi}[\mathrm{sign}_{\epsilon}^l(h(X,Y))] \rvert + \\
&&\hspace{2em}\sup_{h\in \mathcal{H}^M_u}\lvert \mathbb{E}_{\hat\pi}[\mathrm{sign}_{\epsilon}^u(h(X,Y))] - \mathbb{E}_{\pi}[\mathrm{sign}_{\epsilon}^u(h(X,Y))] \rvert+2\epsilon\sup_{x,y}p(y\vert x)\text{\ \ by the bound \eqref{difference of expected values}}.
\end{eqnarray*}
Since $\epsilon$ is arbitrary, the strong GC of $\mathcal{G}_u^M$ then follows from the strong GC of $\mathrm{sign}_{\epsilon}^l(\mathcal{H}^M_u)$ and $\mathrm{sign}_{\epsilon}^l(\mathcal{H}^M_u)$, and the finiteness of $\sup_{x,y}p(y\vert x)$. This concludes the proof.
\end{proof}

\subsection{Proofs for Appendix \ref{sec:linear}}
\begin{proof}[Proof of Theorem \ref{deviation bounds for linear hypothesis}]
Since $\sup_{h \in \mathcal{H}}\lvert \mathbb{E}_{\hat\pi_X}[h(X)] - \mathbb{E}_{\pi_X}[h(X)] \rvert\leq B\Vert \frac{1}{n}\sum_{i=1}^nX_i-\mathbb{E}_{\pi_X}[X]\Vert_{\infty}$, we have $\phi_1(n,\epsilon,\mathcal{H})\leq \mathbb{P}(B\Vert \frac{1}{n}\sum_{i=1}^nX_i-\mathbb{E}_{\pi_X}[X]\Vert_{\infty}>\epsilon)$. We bound the sub-Gaussian norm of $\Vert \frac{1}{n}\sum_{i=1}^nX_i-\mathbb{E}_{\pi_X}[X]\Vert_{\infty}$. Let $X^{(j)},j=1,\ldots,d$ be the $j$-th component of the random vector $X\in \R^d$. Since $X^{(j)}_i,i=1,\ldots,n$ are i.i.d., we have $\Vert \frac{1}{n}\sum_{i=1}^nX_i^{(j)} - \mathbb{E}_{\pi_X}[X^{(j)}]\Vert_{\psi_2}\leq \frac{C}{\sqrt{n}}\Vert X^{(j)} - \mathbb{E}_{\pi_X}[X^{(j)}] \Vert_{\psi_2}$ for some universal constant $C$, by general Hoeffding's inequality (e.g., Proposition 2.6.1 in \citeAPX{vershynin2018high}). Now we can use the sub-Gaussian maximal inequality (e.g., Lemma 2.2.2 in \citeAPX{van1996weak}) to get
\begin{eqnarray*}
\Vert \Vert \frac{1}{n}\sum_{i=1}^nX_i-\mathbb{E}_{\pi_X}[X]\Vert_{\infty}\Vert_{\psi_2}&\leq& C'\sqrt{\log d}\max_{1\leq j\leq d}\Vert \frac{1}{n}\sum_{i=1}^nX_i^{(j)} - \mathbb{E}_{\pi_X}[X^{(j)}]\Vert_{\psi_2}\\
&&\hspace{2em}\text{for another universal constant }C'\\
&\leq& C'C\sqrt{\frac{\log d}{n}}\max_{1\leq j\leq d}\Vert X^{(j)} - \mathbb{E}_{\pi_X}[X^{(j)}] \Vert_{\psi_2}\\
&\leq& C'C\sqrt{\frac{\log d}{n}}\Vert\Vert X - \mathbb{E}_{\pi_X}[X] \Vert_{\infty}\Vert_{\psi_2}
\end{eqnarray*}
where the last inequality holds because $\lvert X^{(j)} - \mathbb{E}_{\pi_X}[X^{(j)}]\rvert\leq \Vert X - \mathbb{E}_{\pi_X}[X] \Vert_{\infty}$ implies $\Vert X^{(j)} - \mathbb{E}_{\pi_X}[X^{(j)}]\Vert_{\psi_2}\leq \Vert  \Vert X - \mathbb{E}_{\pi_X}[X] \Vert_{\infty}\Vert_{\psi_2}$ for all $j=1,\ldots,d$. The expression for $\phi_1$ then comes from the sub-Gaussian tail bound.
\end{proof}


{ \bibliographystyleAPX{abbrvnat} \bibliographyAPX{reference} }

\end{document}